\PassOptionsToPackage{pdftex}{hyperref}
 \documentclass[year=26]{fmcad}

\makeatletter
\@ifundefined{spnewtheorem}{%
	\newcommand{\spnewtheorem}{\@ifstar\spnewtheorem@star\spnewtheorem@nostar}
	\newcommand{\spnewtheorem@star}[4]{\newtheorem{#1}{#2}}
	\newcommand{\spnewtheorem@nostar}[4]{\newtheorem{#1}{#2}}
}{}
\makeatother

\providecommand{\authorrunning}[1]{}
\providecommand{\titlerunning}[1]{}
\providecommand{\inst}[1]{}
\makeatletter
\@ifundefined{c@theorem}{}{}
\@ifundefined{c@lemma}{\newtheorem{lemma}{Lemma}}{}
\@ifundefined{c@proposition}{}{}
\@ifundefined{c@corollary}{}{}
\@ifundefined{c@definition}{}{}
\@ifundefined{c@remark}{\newtheorem{remark}{Remark}}{}
\@ifundefined{c@example}{\newtheorem{example}{Example}}{}
\@ifundefined{proof}{\newenvironment{proof}[1][Proof]{\par\noindent\textit{#1. }}{\hfill$\square$\par}}{}

\@ifundefined{theHALG@line}{}{\renewcommand\theHALG@line{\thealgorithm.\arabic{ALG@line}}}
\makeatother

\usepackage[T1]{fontenc}
\usepackage[english]{babel}
\usepackage[utf8]{inputenc}

\usepackage{amsmath,amssymb,amsfonts,xcolor,graphicx,wrapfig,paralist}
\usepackage{upgreek}
\usepackage{algorithm}
\usepackage[noEnd,indLines=false,italicComments=true,commentColor=darkgray]{algpseudocodex}
\usepackage{mathtools}
\usepackage{booktabs}
\usepackage{adjustbox}
\usepackage{xspace}
\usepackage{enumitem}
\usepackage{standalone}
\usetikzlibrary{arrows}
\usetikzlibrary{arrows.meta}
\usepackage{pgfplots}
\usepackage{csquotes} %
\usepackage{array}
\usepackage{csvsimple}
\usepackage{multirow}
\usepackage{textgreek}
\usepackage[pdftex]{hyperref}
\usepackage{xcolor}
\usepackage{siunitx}
\usepackage{etoolbox} %
\usepackage{fontawesome5}
\usepackage{fancyvrb}

\definecolor[named]{ACMBlue}{cmyk}{1,0.1,0,0.1}
\definecolor[named]{ACMYellow}{cmyk}{0,0.16,1,0}
\definecolor[named]{ACMOrange}{cmyk}{0,0.42,1,0.01}
\definecolor[named]{ACMRed}{cmyk}{0,0.90,0.86,0}
\definecolor[named]{ACMLightBlue}{cmyk}{0.49,0.01,0,0}
\definecolor[named]{ACMGreen}{cmyk}{0.20,0,1,0.19}
\definecolor[named]{ACMPurple}{cmyk}{0.55,1,0,0.15}
\definecolor[named]{ACMDarkBlue}{cmyk}{1,0.58,0,0.21}

\hypersetup{
	colorlinks,
	linkcolor=ACMPurple,
	citecolor=ACMPurple,
	urlcolor=ACMDarkBlue,
	filecolor=ACMDarkBlue}

\usepackage[nolist]{acronym}
\usepackage{breakcites} %

\DeclareFixedFont{\ttb}{T1}{txtt}{bx}{n}{8} %
\DeclareFixedFont{\ttm}{T1}{txtt}{m}{n}{8}  %
\DeclareFixedFont{\tti}{T1}{txtt}{it}{n}{8}

\usepackage{color}
\definecolor{deepblue}{rgb}{0,0,0.5}
\definecolor{deepred}{rgb}{0.6,0,0}
\definecolor{deepgreen}{rgb}{0,0.5,0}

\usepackage{listings}

\newcommand\pythonstyle{\lstset{
		language=Python,
		basicstyle=\ttm,
		morekeywords={self},              %
		keywordstyle=\ttb\color{deepblue},
		emph={MyClass,__init__},          %
		emphstyle=\ttb\color{deepred},    %
		stringstyle=\color{deepgreen},
		commentstyle=\tti\color{purple!60!black},
		frame=tb,                         %
		showstringspaces=false
}}

\lstnewenvironment{python}[1][]
{
	\pythonstyle
	\lstset{#1}
}
{}

\newcommand\pythoninline[1]{{\pythonstyle\lstinline!#1!}}

\lstdefinestyle{customJ}{
  belowcaptionskip=1\baselineskip,
  breaklines=true,
  frame=L,
  xleftmargin=\parindent,
  language=Java,
  showstringspaces=false,
  basicstyle=\footnotesize\ttfamily,
  keywordstyle=\bfseries\color{green!40!black},
  commentstyle=\itshape\color{purple!40!black},
  identifierstyle=\color{blue},
  stringstyle=\color{orange},
}
\usepackage[capitalise]{cleveref}

\usepackage{thm-restate}

\usepackage{placeins} %

\usepackage{booktabs, multirow} %

\usepackage[most]{tcolorbox}

\tcbset{every box/.style={
	left=1mm,right=1mm,top=1pt,bottom=1pt,boxrule=1pt,colbacktitle=white!0,colback=white!0,coltitle=black
}}

\colorlet{linecolor}{gray}
\newtcolorbox{linebox}[1]{
	empty,
	left skip=1mm,
	attach boxed title to top left,
	minipage boxed title,
	title=#1,
	boxed title style={empty,size=minimal,toprule=0pt,top=1mm,left=2mm,bottom=1mm,overlay={}},
	coltitle=black,fonttitle=\bfseries\scshape,
	before=\par\medskip\noindent,parbox=false,boxsep=0pt,left=2mm,right=0mm,top=2pt,breakable,pad at break=0mm,
	before upper=\csname @totalleftmargin\endcsname0pt,
	overlay unbroken={\draw[linecolor,line width=2pt] ([xshift=-0pt]title.north west) -- ([xshift=-0pt]frame.south west); },
	overlay first={\draw[linecolor,line width=2pt] ([xshift=-0pt]title.north west) -- ([xshift=-0pt]frame.south west); },
	overlay middle={\draw[linecolor,line width=2pt] ([xshift=-0pt]frame.north west) -- ([xshift=-0pt]frame.south west); },
	overlay last={\draw[linecolor,line width=2pt] ([xshift=-0pt]frame.north west) -- ([xshift=-0pt]frame.south west); },%
}

\input{format/macros}
\usepackage{xcolor}

\newtoggle{showplots}
 \toggletrue{showplots}

\usepackage[Tol]{colorblind}
\colorlet{color1}{T-Q-MC2}
\colorlet{color2}{T-Q-MC3}
\colorlet{color3}{T-Q-MC4}
\colorlet{color4}{T-Q-MC5}
\colorlet{color5}{T-Q-MC6}
\colorlet{color6}{T-Q-MC7}
\colorlet{color7}{T-Q-MC0}

\newcommand{\quantileplotxlabel}{solved benchmarks}
\newcommand{\quantileplotylabel}{treesize}
\newlength{\quantileplotwidth}
\newlength{\quantileplotheight}
\newcommand{\quantileplotlegendcols}{1}

\newcommand{\standardquantileplotlinestyle}{ultra thick}
\newcommand{\quantileplotlinestyle}{\standardquantileplotlinestyle}

\newcommand{\quantileplot}[9]{%
	\begin{tikzpicture}
	\begin{axis}[
	width=\quantileplotwidth,
	height=\quantileplotheight,
	xmin=#4,
	xmax=#5,
	ymin=#6,
	ymax=#7,
	ymajorgrids,
	ymode=log,
	axis x line=bottom,
	axis y line=left,
	unbounded coords=discard,filter discard warning=false, %
	xlabel=\quantileplotxlabel,
	xlabel style={font=\scriptsize,yshift=4pt},%
	ylabel=\quantileplotylabel,
	ylabel style={font=\scriptsize,yshift=-4pt},%
	yticklabel style={font=\scriptsize},
	scaled y ticks=false,
	xticklabel style={font=\scriptsize},
	legend columns=\quantileplotlegendcols,
	legend pos={#8},
	legend style={nodes={scale=0.75, transform shape},inner sep=1pt,#9},
/pgfplots/legend image code/.code={\draw[mark repeat=2,mark phase=2,##1] plot coordinates {(0cm,0cm) (0.3cm,0cm)};}, %
	every axis plot/.append style={\quantileplotlinestyle},
	legend cell align={left}
	]
	\iftoggle{showplots}{
	\foreach \tool\color in {#2}{%
		\edef\loopbody{
			\noexpand\addplot[\color] table [x=n,y=\tool, col sep=comma] {#1};
		}
		\loopbody
	}
	\legend{#3}}{\node[anchor=south west, align=center, red] {\huge NOT\\ COMPILED};}
	\end{axis}
	\end{tikzpicture}
}

\newlength\scatterplotsize
\newcommand{\scatterplottime}[6]{%
	\begin{tikzpicture}
		\begin{axis}[
			width=\scatterplotsize,
			height=\scatterplotsize,
			axis equal image,
			xmin=1,
			ymin=1,
			ymax=2048,
			xmax=2048,
			xmode=log,
			ymode=log,
			axis x line=bottom,
			axis y line=left,
			xtick={1,2,4,8,16,32,64,128,256},
			xticklabels={1,2,4,8,16,,64,,\!\!256},
			extra x ticks = {512,1024,2048},
			extra x tick labels = {,TO, err},
			extra x tick style = {grid = major},
			ytick={1,2,4,8,16,32,64,128,256},
			yticklabels={1,2,4,8,16,32,64,128,256},
			extra y ticks = {512, 1024, 2048},
			extra y tick labels = {$\ge$512,TO, err},
			extra y tick style = {grid = major},
			xlabel={#3},
			xlabel style={font=\scriptsize,yshift=5pt},%
			ylabel={#5},
			ylabel style={font=\scriptsize,yshift=-9pt},%
			yticklabel style={font=\scriptsize},
			xticklabel style={font=\scriptsize},%
			legend pos=north east,
			legend columns=3,
			legend style={nodes={scale=0.75, transform shape},inner sep=1.5pt, xshift=1mm, yshift=7mm},
			set layers,
			mark layer=axis background
			]
			
			\iftoggle{showplots}{\addplot[
				scatter,
				only marks,
				scatter/classes={
					a={mark=pentagon*,color3,mark size=1.75}%
				},
				scatter src=explicit symbolic
				]%
				table [col sep=comma,x=#2,y=#4
				,meta=Class
				] {#1};
			}{\node[anchor=south west, align=center, red] {\huge NOT\\ COMPILED};}
			\addplot[no marks] coordinates {(0.01,0.01) (512,512) };
			\addplot[no marks, densely dotted] coordinates {(0.01,0.02) (256,512)};
			\addplot[no marks, densely dotted] coordinates {(0.02,0.01) (512,256)};
		\end{axis}
	\end{tikzpicture}
}

\newlength\scatterplot
\newcommand{\scatterplottrees}[6]{%
	\begin{tikzpicture}
		\begin{axis}[
			width=\scatterplotsize,
			height=\scatterplotsize,
			axis equal image,
			xmin=1,
			ymin=1,
			ymax=2048,
			xmax=2048,
			xmode=log,
			ymode=log,
			axis x line=bottom,
			axis y line=left,
			xtick={1,2,4,8,16,32,64,128,256},
			xticklabels={1,2,4,8,16,,64,,\!\!256},
			extra x ticks = {512,1024, 2048},
			extra x tick labels = {,TO, err},
			extra x tick style = {grid = major},
			ytick={1,2,4,8,16,32,64,128,256},
			yticklabels={1,2,4,8,16,32,64,128,256},
			extra y ticks = {512, 1024, 2048},
			extra y tick labels = {$\ge$512,TO, err},
			extra y tick style = {grid = major},
			xlabel={#3},
			xlabel style={font=\scriptsize,yshift=5pt},%
			ylabel={#5},
			ylabel style={font=\scriptsize,yshift=-9pt},%
			yticklabel style={font=\scriptsize},
			xticklabel style={font=\scriptsize},%
			legend pos=north east,
			legend columns=3,
			legend style={nodes={scale=0.75, transform shape},inner sep=1.5pt, xshift=1mm, yshift=7mm},
			set layers,
			mark layer=axis background
			]
			
			\iftoggle{showplots}{\addplot[
				scatter,
				only marks,
				scatter/classes={
					a={mark=pentagon*,color3,mark size=1.75}%
				},
				scatter src=explicit symbolic
				]%
				table [col sep=comma,x=#2,y=#4
				,meta=Class
				] {#1};
			}{\node[anchor=south west, align=center, red] {\huge NOT\\ COMPILED};}
			\addplot[no marks] coordinates {(0.01,0.01) (512,512) };
			\addplot[no marks, densely dotted] coordinates {(0.01,0.02) (256,512)};
			\addplot[no marks, densely dotted] coordinates {(0.02,0.01) (512,256)};
			\addplot[no marks, densely dotted] coordinates {(0.01,0.1) (51.2,512)};
			\addplot[no marks, densely dotted] coordinates {(0.1,0.01) (512,51.2)};
		\end{axis}
	\end{tikzpicture}
}

\newcommand{\scatterplottimeToModel}[6]{%
	\begin{tikzpicture}
		\begin{axis}[
			width=\scatterplotsize,
			height=\scatterplotsize,
			xmin=20000,
			ymin=1,
			ymax=700,
			xmax=600000,
			xmode=log,
			ymode=linear,
			axis x line=bottom,
			axis y line=left,
			xtick={20000,100000,500000},
			xticklabels={\num{20000},\num{100000},\num{500000}},
			extra x ticks = {},
			extra x tick labels = {},
			extra x tick style = {grid = major},
			ytick={1,100,200,300,400,500},
			yticklabels={1,100,200,300,400,500},
			extra y ticks = {600},
			extra y tick labels = {TO},
			extra y tick style = {grid = major},
			xlabel={#3},
			xlabel style={font=\scriptsize,yshift=5pt},%
			ylabel={#5},
			ylabel style={font=\scriptsize,yshift=-9pt},%
			yticklabel style={font=\scriptsize},
			xticklabel style={font=\scriptsize},%
			legend pos=north east,
			legend columns=3,
			legend style={nodes={scale=0.75, transform shape},inner sep=1.5pt, xshift=1mm, yshift=10mm},
			set layers,
			mark layer=axis background
			]
			
			\iftoggle{showplots}{\addplot[
				scatter,
				only marks,
				scatter/classes={
					ours-controller-simulation-simple={mark=diamond*,color1,mark size=1.75},
					ours-combined-agency-simple={mark=pentagon*,color3,mark size=1.75},
					ours-permissive-cost-of-error-simple={mark=triangle*,color2,mark size=1.75},
					cav15={mark=*,color5,mark size=1.75},
					vanilla={mark=otimes,color6,mark size=1.75},
					virtualBest={mark=empty,color3,mark size=1.75},
					ours-combined-simulation-simple={mark=empty,color3,mark size=1.75},
					ours-permissive-simulation-simple={mark=empty,color3,mark size=1.75},
					ours-permissive-agency-simple={mark=empty,color3,mark size=1.75},
					ours-controller-agency-simple={mark=empty,color3,mark size=1.75},
					ours-controller-cost-of-error-simple={mark=empty,color3,mark size=1.75},
					ours-combined-cost-of-error-simple={mark=empty,color3,mark size=1.75},
					ours-permissive-uniform-simple={mark=empty,color3,mark size=1.75},
					ours-controller-uniform-simple={mark=empty,color3,mark size=1.75},
					ours-combined-uniform-simple={mark=empty,color3,mark size=1.75}
				},
				scatter src=explicit symbolic
				]%
				table [col sep=comma,x=#2,y=#4
				,meta=Class
				] {#1};
			}{\node[anchor=south west, align=center, red] {\huge NOT\\ COMPILED};}
				\ifthenelse{\NOT\equal{#6}{false}}{\legend{\wsimulation, \wagency, \wmessup, CAV15, \dtcontrol}}{}
		\end{axis}
	\end{tikzpicture}
}

\newcommand{\scatterplotSizeToModel}[6]{%
	\begin{tikzpicture}
		\begin{axis}[
			width=\scatterplotsize,
			height=\scatterplotsize,
			unbounded coords=discard,
			filter discard warning=false,
			xmin=20000,
			ymin=1,
			ymax=15000,
			xmax=600000,
			xmode=log,
			ymode=log,
			axis x line=bottom,
			axis y line=left,
			xtick={20000,100000,500000},
			xticklabels={\num{20000},\num{100000},\num{500000}},
			extra x ticks = {},
			extra x tick labels = {},
			extra x tick style = {grid = major},
			ytick={1,10, 100, 1000,10000},
			yticklabels={1,10, 100, 1000,10000},
			extra y ticks = {4000},
			extra y tick labels = {},
			extra y tick style = {},
			xlabel={#3},
			xlabel style={font=\scriptsize,yshift=5pt},%
			ylabel={#5},
			ylabel style={font=\scriptsize,yshift=-9pt},%
			yticklabel style={font=\scriptsize},
			xticklabel style={font=\scriptsize},%
			legend pos=north east,
			legend columns=3,
			legend style={nodes={scale=0.75, transform shape},inner sep=1.5pt, xshift=1mm, yshift=10mm},
			set layers,
			mark layer=axis background
			]
			
			\iftoggle{showplots}{\addplot[
				scatter,
				only marks,
				scatter/classes={
					ours-controller-simulation-simple={mark=diamond*,color1,mark size=1.75},
					ours-combined-agency-simple={mark=pentagon*,color3,mark size=1.75},
					ours-permissive-cost-of-error-simple={mark=triangle*,color2,mark size=1.75},
					cav15={mark=*,color5,mark size=1.75},
					vanilla={mark=otimes,color6,mark size=1.75},
					virtualBest={mark=empty,color3,mark size=1.75},
					ours-combined-simulation-simple={mark=empty,color3,mark size=1.75},
					ours-permissive-simulation-simple={mark=empty,color3,mark size=1.75},
					ours-permissive-agency-simple={mark=empty,color3,mark size=1.75},
					ours-controller-agency-simple={mark=empty,color3,mark size=1.75},
					ours-controller-cost-of-error-simple={mark=empty,color3,mark size=1.75},
					ours-combined-cost-of-error-simple={mark=empty,color3,mark size=1.75},
					ours-permissive-uniform-simple={mark=empty,color3,mark size=1.75},
					ours-controller-uniform-simple={mark=empty,color3,mark size=1.75},
					ours-combined-uniform-simple={mark=empty,color3,mark size=1.75}
				},
				scatter src=explicit symbolic
				]%
				table [col sep=comma,x=#2,y=#4
				,meta=Class
				] {#1};
			}{\node[anchor=south west, align=center, red] {\huge NOT\\ COMPILED};}
			\ifthenelse{\NOT\equal{#6}{false}}{\legend{\wsimulation, \wagency, \wmessup, CAV15, \dtcontrol}}{}
		\end{axis}
	\end{tikzpicture}
}

\pgfplotsset{compat=1.17}
\usetikzlibrary{backgrounds,scopes,calc,positioning,arrows,shapes}

\hypersetup{hypertexnames=false}

\makeatletter

\renewcommand*{\theHALG@line}{\arabic{page}.\arabic{ALG@line}}
\makeatother

\newlist{todolist}{itemize}{2}
\setlist[todolist]{label=$\square$}
\usepackage{pifont}

\title{
	\dte: Trading Optimality for Explainability in 
	MDPs via Decision Trees
}

\author{
\IEEEauthorblockN{%
	Tereza Kinsk\'{a}\textsuperscript{1}\orcid{0009-0006-2968-4006},\;
	Jan K\v{r}et\'{i}nsk\'{y}\textsuperscript{1,2}\orcid{0000-0002-8122-2881},\;
	Tobias Meggendorfer\textsuperscript{3}\orcid{0000-0002-1712-2165},\;
	Sabine Rieder\textsuperscript{1,2}\orcid{0009-0006-6397-3100},\;
	Maximilian Weininger\textsuperscript{4}\orcid{0000-0002-0163-2152}%
}
\IEEEauthorblockA{\textsuperscript{1}Masaryk University, Brno, Czech Republic\\Email: \{xkinska, jan.kretinsky, xrieder\}@fi.muni.cz}
\IEEEauthorblockA{\textsuperscript{2}Technical University of Munich, Munich, Germany \\Email: \{jan.kretinsky, sabine.rieder\}@tum.de}
\IEEEauthorblockA{\textsuperscript{3}Lancaster University Leipzig, Leipzig, Germany \\ Email: t.meggendorfer@lancaster.ac.uk}
\IEEEauthorblockA{\textsuperscript{4}Ruhr-University Bochum, Bochum, Germany \\ Email: maximilian.weininger@rub.de}
}

\authorrunning{T. Kinsk\'{a} et al.}

\begin{document}

\maketitle

\begin{abstract}
	Over the past decade, decision trees have been %
	used to represent controllers (a.k.a.\ policies) in an explainable way, with \dtcontrol as a current state-of-the-art tool.
	However, for systems that are large or have many corner cases, even such representations tend to be too complex and not human-comprehensible. %
	Unfortunately, reducing the size of the decision tree is not straightforward, as missing just a single crucial case might result in an incorrect controller.
	We tackle this issue in the setting of Markov decision processes, extending \dtcontrol by \enquote{$\varepsilon$} functionality:
	Given an allowed imprecision $\varepsilon\geq 0$, we construct a smaller decision tree, distilling the essence of the controller, while still guaranteeing its $\varepsilon$-optimality.
	This enables us to provide tunably simpler explanations, omitting a controllable amount of detail.
	Our tool constructs decision trees that are orders of magnitude smaller than the state of the art.
\end{abstract} 

\section{Introduction}
\acresetall

\newcommand{\para}[1]{\smallskip\noindent\textbf{#1}}

\para{Need for Compact and Explainable Controllers.} %
\emph{Controller synthesis}~\cite{DBLP:journals/pieee/RamadgeW89,DBLP:conf/popl/PnueliR89} automatically constructs a controller for a non-deterministic system so that it satisfies a given specification.
While being a tempting alternative to error-prone human-designed controllers, several hurdles prevent widespread adoption:
Automatically synthesised controllers are typically represented explicitly as \emph{unstructured, large tables}, assigning to each state of the system an action to take. %
In contrast, numerous practical purposes require controllers to be \emph{concise} and \emph{explainable}.
Regarding size, it may be challenging to implement tables with millions of rows, %
especially on embedded devices with limited hardware.
Regarding explainability, such a table's sheer size and lack of structure prohibits even domain experts from understanding the controller.
In addition to engineers' natural reluctance to use controllers they cannot understand (and thus sensibly maintain and certify), explaining the controller contributes significantly to \emph{validating} the whole modelling process. Finally, explainability per se may even become a \emph{legal requirement} \cite{EUX,USX} whenever AI techniques are aiding the process. %
A plethora of surveys on explainability underline its relevance~\cite{DBLP:journals/scirobotics/GunningSCMSY19,DBLP:journals/inffus/ArrietaRSBTBGGM20,DBLP:journals/widm/AngelovSJAA21,DBLP:journals/tnn/TjoaG21,DBLP:conf/hicss/GerlingsSC21,meske2022explainable,DBLP:journals/ai/MillerHAH22}.

\para{Decision trees (DT)} are among the most explainable formalisms, being both transparent and interpretable, e.g.~\cite{DBLP:journals/widm/AngelovSJAA21, DBLP:journals/eswa/PiltaverLGM16, DBLP:journals/air/CostaP23}.
For a decade, they have been used for representing controllers, e.g.~\cite{liu2010compact,CAV15,DBLP:conf/qest/AshokKLCTW19,dtcontrol2,DBLP:conf/aaai/TopinMFV21,DBLP:conf/ijcai/VosV23,DBLP:conf/xai/DubslaffKP24}
and as such found various usages, namely for detecting bugs~\cite{CAV15,dtcontrol2}, as part of modelling~\cite{KushJonis}, or for explaining complex strategies~\cite{DBLP:conf/vecos/BuddeDH24}.
Mature tool support is available~\cite{dtcontrol2,dtpaynt}.
However, in an interesting contrast to (generally imprecise) DT learning, the safety-criticality motivated these tools to focus on \emph{precise representation} of controllers, i.e.\ for \emph{every} state providing an optimal (winning/safe) action.

\para{Our focus} is on representing controllers arising in %
Markov decision processes (MDPs)~\cite{puterman}.
There, representing the full controller completely is often \emph{not} necessary, as many states are only reached with small probability and behaving sub-optimally in them does not significantly change the outcome, cf.~\cite{KM20-cores,DBLP:conf/atva/Meggendorfer22}.
Moreover, most scalable methods for controller synthesis do not provide an optimal controller, but only an \emph{$\varepsilon$-optimal} one~\cite{qcomp-BHKKPQTZ20,HJQW23}, i.e.\ whose performance is $\varepsilon$-close to optimal for a given $\varepsilon\geq0$.
Exploiting this fact, we allow the DT to ignore some decisions, provided the resulting performance is still $\varepsilon$-optimal.
DTs~\cite{DBLP:journals/eswa/PiltaverLGM16,DBLP:journals/natmi/Rudin19}, though comprehensibility also depends on predicate complexity, not size alone~\cite{DBLP:journals/dss/HuysmansDMVB11}.
A related idea was presented in~\cite{CAV15} for reachability objectives, where simulation was used to heuristically learn only the decision in frequently visited states, ignoring others.
However, this heuristic was \enquote{aggressive}, potentially compromising $\varepsilon$-optimality.
Moreover, the prototypical implementation of~\cite{CAV15} is unavailable.
In contrast, we provide a self-contained, richly parametrisable tool that guarantees $\varepsilon$-optimality of the resulting controller across the full standard spectrum of objectives.

\para{Our contribution} is the tool \dte: It extends 
\dtcontrol~\cite{dtcontrol2}, most importantly by adding the functionality---dubbed \enquote{$\varepsilon$}---to construct \emph{smaller and more explainable DTs}, \emph{exploiting the allowed imprecision $\varepsilon$}.

The workflow description (\cref{fig:tool-overview}, \cref{sec:3-workflow}) outlines all the technical novelties.
The key enabler is our tight coupling of the tool with the probabilistic model checker \storm~\cite{storm}, based on which we provide numerous ways for DT size reduction while guaranteeing $\varepsilon$-optimality:
(i)~The initial controller is permissive, containing multiple optimal actions per state.
Thus, we can profit from \dtcontrol's powerful determinisation heuristics.
(ii)~We reduce the size of the controller by identifying which decisions are of key importance, and which have little or no impact on the resulting behaviour. %
While \dtcontrol captures decisions in all states, \dte builds smaller DTs containing only relevant decisions. %
Unlike~\cite{CAV15}, most of these heuristics are \enquote{safe}, i.e.\ they do not compromise $\varepsilon$-optimality.
(iii)~The high-level description of the model allows us to suggest more efficient predicates to be used in DT construction, improving on the semi-automatic approach of~\cite{dtcontrol2}.
(iv)~Model checking a DT before returning it provides its current precision, %
which allows for trying aggressive heuristics such as the one of~\cite{CAV15} or extensive pruning while still guaranteeing $\varepsilon$-optimality.
Practically, we encountered several (unexpected) technical challenges, such as needing to extend \stormpy, e.g. to handle randomising and non-deterministic controllers, required for efficient DT validation.

Finally, our approach is applicable to all the standard objectives such as reachability/safety, reach-avoid, or reward-based objectives, covering the standard benchmark collections~\cite{qcomp-benchmarkset,hartmanns2025revised}.
Notably, our tool is the only one supporting linear temporal logic (LTL) properties, see \cref{app:ltl}.

\para{Improvements over other tools.}
Our evaluation compares to \dtcontrol{}~\cite{dtcontrol2}, \dtpaynt~\cite{dtpaynt}, and \dtnest~\cite{DTNest}.
Already for precise representation ($\varepsilon=0$), \dte achieves order-of-magnitude improvements over \dtcontrol{} due to the improved initial controller, safe heuristics, and automatic predicate suggestions.
\dtpaynt mostly times out, and on over half the remaining benchmarks produces larger DTs.
\dtnest{} is the strongest competitor; still, \dte produces DTs at most half the size in 23 of 38 benchmarks and an order of magnitude smaller in 10.
Moreover, \dte is best at exploiting allowed imprecision: when allowing $\varepsilon = 10^{-2}$, we obtain single-node DTs for nearly half the benchmarks, exposing that focusing on a single action is $\varepsilon$-optimal — a previously unobserved fact.
In contrast, \dtnest mostly fails due to internal errors and, surprisingly, may even produce \emph{larger} DTs when $\varepsilon>0$.

\para{Related work.}
\dtcontrol{}~\cite{dtcontrol2}, the direct predecessor of \dte, comes from a stream of work on \emph{precise} controller representation~\cite{CAV15,DBLP:conf/qest/AshokKLCTW19,DBLP:conf/hybrid/AshokJJKWZ20,dtcontrol2,DBLP:journals/sttt/JungermannKW23}.

The closest work to our $\varepsilon$-approach is~\cite{CAV15}: it prunes the DT and weights learning by visit frequency from simulation.
Both ideas are \emph{aggressive}, i.e., they can compromise $\varepsilon$-optimality, and ~\cite{CAV15} handles this by restarting from scratch if the result is unsound.
In contrast, \dte cleanly separates safe from aggressive heuristics, guaranteeing $\varepsilon$-optimality without restarting, and builds on \dtcontrol's tailored DT algorithms rather than a general-purpose ML package~\cite[Sec.\,6.1]{CAV15}.
Verification of the resulting DT is automated in \dte.
In contrast, \cite{CAV15} required manual binary search over hyperparameters.
Finally,~\cite{CAV15} covers only reachability and uses a state$\times$action ($\states\times\actions\to\{0,1\}$) representation, while \dte supports all standard objectives and the $\states\to\actions$ representation of the \dtcontrol tradition.

\dtpaynt~\cite{dtpaynt} finds policies representable as minimal-depth DTs via abstraction-refinement and SMT.
However, its scalability is limited, timing out on most of our benchmarks.
\dtnest~\cite{DTNest} builds on both \dtcontrol and \dtpaynt: it constructs a DT with \dtcontrol and iteratively shrinks subtrees with \dtpaynt, feeding the result back to both tools for local search.
This is orthogonal to our approach and can be seen as an elaborate pruning step; a tight coupling of \dte with \dtnest could combine both strengths.
OMDT~\cite{DBLP:conf/ijcai/VosV23} finds optimally performing trees of a prescribed size, but provides no error control and does not scale beyond depth~3; it is significantly outperformed by \dtpaynt~\cite[Sec.\,5]{dtpaynt}.
Standard DT learning~\cite{liu2010compact} can represent controllers but is not tailored to the task of finding $\varepsilon$-optimal controllers that can be represented by small DTs.

Furthermore, representing RL policies as DTs has received considerable attention, e.g.\ via direct RL frameworks~\cite{DBLP:conf/aaai/TopinMFV21,NEURIPS2018_e6d8545d}, policy gradients~\cite{Das_Gupta_Talvitie_Bowling_2015}, distillation~\cite{kohler2024interpretableeditableprogrammatictree}, and dynamic programming~\cite{JMLR:v23:20-520}.
DTs have also been used for shields~\cite{brorholt2025uppaalcoshyautomaticsynthesis}.
The RL community has also studied approximating value functions via linear programming~\cite{GoodReprSufficientRL}, domain-specific languages~\cite{verma2019programmaticallyinterpretablereinforcementlearning,batz2023programmaticstrategysynthesisresolving}, and shallow neural networks~\cite{shallow_nn,article_shallow_nn}.
We do not compare to these, as they do not provide any control of the (im)precision.

Binary decision diagrams (BDDs) are an alternative to DTs but are hindered in this context by bit-blasting of non-Boolean variables, and fixed variable ordering resulting in larger representations~\cite{DBLP:conf/hybrid/AshokJJKWZ20,dtcontrol2}.
Recent work~\cite{JanHSCC25} partially closes this gap.
Neural networks are fundamentally less explainable than DTs~\cite{DBLP:journals/widm/AngelovSJAA21}.

\section{Technical Preliminaries} \label{sec:prelimns}
\subsection{Markov Decision Processes}
Let $\Distributions(X)$ denote the set of all probability distributions on a finite set $X$.
A (finite-state) \emph{Markov decision process} (MDP) \cite{puterman,BK08,kallenberg2011markov} is a tuple $(\states,\actions,\initstate,\trans)$, where $\states$ and $\actions$ are finite set of states and actions, $\initstate\in\states$ is an initial state, and $\trans \colon \states \times\actions\rightharpoonup \Distributions(\states)$ is a (partial) transition function such that the set of \emph{available actions} $\actions(s)\eqdef \{a \mid (s,a) \in \text{domain}(\trans)\}$ is non-empty for all $s\in\states$.
Intuitively, in every state $s\in\states$, there is a non-deterministic choice between the available actions $\actions(s)$; choosing an action samples a successor according to $\trans(s,a)$, where the process is repeated.
For readability, we may write $\trans(s, a, s')$ instead of $\trans(s, a)(s')$.

A (memoryless) \emph{controller} (a.k.a.\ scheduler, policy, strategy) is a function $\controller \colon \states \to \Distributions(\actions)$ that for every state resolves the non-deterministic choice by picking a distribution over (available) actions (i.e.\ $\controller(s)(a) > 0$ implies $a \in \actions(s)$).
We write $\policies$ to denote the set of all controllers.

The \emph{semantics} of MDPs are defined as usual (e.g.~\cite[Chp.~10]{BK08}):
A path is an infinite sequence of states $s_1 s_2 \ldots$, and $\paths$ is the set of all such paths.
Together with a controller $\controller$ and state $s$, we obtain a unique probability measure $\prob_s^\controller$ for (the standard $\sigma$-algebra over) $\paths$.
We denote the expectation under this probability measure by $\mathbb{E}_s^\controller$.

An \emph{objective} $\Phi$ assigns values to paths.
	\emph{Reachability} is the objective we focus on in the main body.
	Given a set of goal states $\gStates$, we are interested in $\lozenge \gStates = \{\rho \in \paths \mid \exists i.\ \rho_i \in \gStates\}$, i.e.\ the set of all paths that reach any state of $\gStates$.
	These paths are assigned 1, all others 0, i.e.\ $\Phi(\rho) = 1$ if $\rho \in \lozenge \gStates$ and $0$ otherwise.
	\emph{Safety} is the dual of reachability. 
	We want to \emph{avoid} reaching the given set of states, i.e.\ the set of paths with assigned 1 are $\prob_\initstate^{\opt}[\overline{\lozenge \gStates}] = 1 - \prob_\initstate^{\opt}[{\lozenge \gStates}]$.
	\emph{Constrained reachability}, also known as a \emph{reach-while-avoid} property, is their combination: Given unsafe and goal states $X$ and $\gStates$, we are interested in the set of paths $\{\rho \mid \exists i.\ \rho_i \in \gStates \land \forall j < i.\ \rho_j \notin X\}$.
	Note that in general this is slightly different from $\overline{\lozenge X} \cap \lozenge \gStates$, which prohibits reaching $X$ even after goal $\gStates$ is reached.

	\emph{Reward-based objectives} make use of a reward function $r : \states \to \mathbb{R}_0^+$, mapping states to non-negative numbers.
	There are also variants which additionally assign rewards to individual actions or transitions.
		\emph{Total reward} is defined as $\mathsf{TR}(\rho) = \sum_{i=1}^\infty r(\rho_i)$, i.e.\ the sum of all rewards along the path.
		Since the rewards are non-negative, the sum is well-defined.
		\emph{Reachability cost} is then defined similarly.
		Additionally given a goal set $\gStates$, we set $\mathsf{RC}(\rho) = \infty$ if $\rho \notin \lozenge \gStates$, i.e.\ the goal is not reached, and $\sum_{i=1}^x r(\rho_i)$ where $x = \min \{i \mid \rho_i \in \gStates\} - 1$, i.e.\ the sum of all rewards (or, in this case, rather costs) until the goal set is reached.
	Further details and proof that (globally) optimal controllers exist can be found in \cite{SGreward}.
	
	\emph{Linear temporal logic (LTL)} can be used to express more complex goals (e.g., "pick up and deliver a package while ensuring safety").
	We refer to~\cite{BK08} for a formal definition, as a precise treatment of LTL objectives for MDPs goes beyond the scope of this work.
	We highlight that in contrast to previously mentioned objectives, LTL requires controllers with memory, which \dte{} supports (see \cref{app:ltl}).

The \emph{value} of a state $s$ under a given controller $\controller$ is the expectation of $\Phi$, the random variable assigning numbers to paths, under the probability measure $\prob_s^\controller$, i.e.\
$\val^\controller(s)\eqdef \mathbb{E}_s^\controller[\Phi]$.
For reachability, safety, and constrained reachability, this coincides with the probability of the set of paths we are interested in, e.g.\ for reachability we have $\val^\controller(s) = \mathbb{E}_s^\controller[\Phi] = \prob_s^\controller[\lozenge \gStates]$.

The \emph{optimal value} of a state is the optimal value over all controllers, i.e.\ $\prob_s^{\opt}[\lozenge\gStates] \eqdef \opt_{\controller\in\policies} \prob_s^\controller[\lozenge\gStates]$ (for $\opt \in \{\min, \max\}$).
When the goal states and optimization are clear from the context, we write $\val^\controller(s) \eqdef \prob_s^\controller[\lozenge\gStates]$ for the value of a controller $\controller$ in $s$, $\val(s) \eqdef \prob_s^{\opt}[\lozenge\gStates]$ for the optimal value in $s$, and define the \emph{value of an MDP} as the value of its initial state $\val(\initstate)$.
A controller is \emph{optimal} (in $\initstate$) if $\val^\controller(\initstate) = \val(\initstate)$, and \emph{globally optimal} if $\val^\controller(s) = \val(s)$ for all $s \in \states$.
Globally optimal controllers exist for the objectives we consider~\cite{BK08,SGreward}.
Finally, for $\varepsilon \in [0,1]$, $\controller$ is $\varepsilon$-optimal (in $\initstate$) if $\abs{\val^\controller(\initstate) - \val(\initstate)} \leq \varepsilon$.
We write $\val(s, a) \eqdef \sum_{s'\in\states}\trans(s,a)(s') \cdot \val(s')$ for the value obtained by following $a$ in state $s$ and then playing optimally.
Note that this may be different from the value obtained by always choosing action $a$ in $s$, e.g.\ if $a$ loops back to $s$.

\subsection{Decision Trees}

\def\ltrue{\textsf{true}}
\def\lfalse{\textsf{false}}

A \emph{decision tree} (DT) $\decisiontree$~\cite{mitchell97} is a binary tree where leafs are labelled by actions and inner nodes are labelled by \emph{predicates} $\pred \in \predicates$, $\pred \colon \states \to \{\ltrue, \lfalse\}$.
States typically comprise a tuple of variables, e.g.\ $x$- and $y$-coordinates, allowing for compact symbolic predicates, e.g.\ $x \leq 2$, see \cref{fig:example} (right).

A DT yields an action for a given state $s$ as follows:
Start from the root node.
In an inner node, evaluate its predicate $\pred(s)$ and go to the left or right child if it evaluates to \ltrue{} or \lfalse{}, respectively.
In a leaf node labelled $a$, pick $a$ if $a\in \actions(s)$, and otherwise %
uniformly at random choose an available action.
We denote the last case by the special label $\uniforma$.
We write $\decisiontree(s)$ for the action (or uniform choice) yielded by the DT $\decisiontree$ in state $s$ and identify the DT with the corresponding controller.

\begin{remark}
	Note that the above subtly assumes that the system can at runtime \enquote{pick any action (at random)}.
	Since even \enquote{do nothing} is an action to be modelled, we believe this assumption is natural.
	However, if this is not the case, one can simply disable related heuristics (\eqval{} and \allSafe{}, defined below); the rest of our approach does not rely on them.
\end{remark}

\begin{algorithm}[!t]
    \caption{\dtconstruct}\label{alg:DT-construct}
    \begin{algorithmic}[1]
        \Require Dataset $\mathcal{D} \colon \states' \to \actions$ where $\states'\subseteq \states$
        \Ensure Decision Tree representing $\mathcal{D}$
        \If{exists $a\in\actions$ with $\mathcal{D}(s)=a$ for all $s\in\states'$} %
            \State \Return Leaf with label $a$
        \EndIf
        \State $\pred^* \gets \argmin_{\pred\in\predicates} \impurity(\mathcal{D}, \pred)$
        \State \Return Node with label $\pred^*$ and children
        \Statex \phantom{\Return} $\dtconstruct(\pred^*(\mathcal{D}))$ and
				\Statex \phantom{\Return} $\dtconstruct(\lnot \pred^*(\mathcal{D}))$
    \end{algorithmic}
\end{algorithm}%

Constructing a minimal DT is NP-complete~\cite{DBLP:journals/ipl/HyafilR76}, so we base our approach on the standard heuristic-based algorithm outlined in \cref{alg:DT-construct}.
Here, the input dataset is a controller, mapping states to actions.
If all given states use the same action, the algorithm constructs a leaf node with that common action.
Otherwise, we attempt to \emph{split} the dataset using a predicate $\pred^*$, chosen from a set $\predicates$ using an \emph{impurity measure} $\impurity$, 
which intuitively assesses how \enquote{heterogeneous} $\pred(\mathcal{D})$ and $\lnot \pred(\mathcal{D})$ are (where $\pred(\mathcal{D}) = \{(s, a) \in \mathcal{D} \mid \pred(s)\}$ and $\lnot \pred(\mathcal{D})$ analogously).
This heuristically indicates how many further splits are required to represent the dataset.
Thus, we insert the predicate minimizing impurity as a tree node, split the dataset accordingly and recurse on the two smaller dataset which only contain choices for a subset $\states'\subseteq\states$.

\begin{remark}
	We point out two differences to the definition in~\cite{dtcontrol2}:
	Firstly, we label leaves with a single action, whereas they use sets of actions.
	While our theory is also applicable to using sets, there are exponentially many sets of actions, and \emph{determinising}, i.e.\ committing to one safe choice, reduces the final DT's size and increases readability.
	(We, however, do employ non-determinism in intermediate steps, see below.)
	Secondly, \cite{dtcontrol2} ensures that $\controller(s)\in\actions(s)$ for every $s$, i.e.\ the DT only recommends available actions.
	We lift this requirement to allow cases where the choice in $s$ is irrelevant and thus does not need to be represented explicitly.
\end{remark}

\section{Tool Overview} \label{sec:overview_example}

\newcommand{\obstacleicon}{\textcolor{red}{\faSkull}}%

This section sketches the main ideas and overall structure of \dtcontroleps{}.
The tool aims to efficiently create small, explainable decision tree controllers for probabilistic systems.
The key novelty and difference to \dtc{} is employing the \emph{information about the model} and the \emph{allowed imprecision~$\varepsilon$}.

Our goal then can be stated as follows:

\begin{linebox}{Problem Statement}
	Given an MDP, objective, and precision $\varepsilon \geq 0$, obtain a (small) decision tree $\decisiontree$ such that $|\val^{\decisiontree}(\initstate) - \val(\initstate)| \leq \varepsilon$.
\end{linebox}
\noindent
(Note that $\varepsilon = 0$ means that $\decisiontree$ should represent an optimal controller.)
The key challenge is to exploit the allowed imprecision while maintaining all: scalability (in contrast to, e.g., \cite{DTNest}), correctness (in contrast to, e.g.,~\cite{CAV15}, which relies on discovering all important states during simulation, which sometimes yields a DT that is not epsilon-optimal),
and completeness (if a decision tree exists, our tool will always find one by omitting aggressive heuristics when necessary).

\subsection{Illustrative Example}

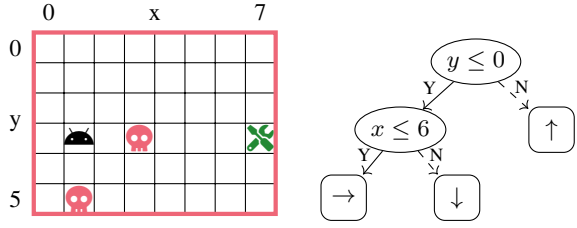
\begin{figure} %
	\centering
	\begin{tikzpicture}[auto,xscale=.4,yscale=.4]
		\draw[step=1.0] (-4,-3) grid (4,3);
		\draw[red,ultra thick] (-4,-3) rectangle (4,3);
		\node[anchor=center] at (3.5,-0.5) {\color{green}\faTools};
		\node[anchor=center] at (-2.5,-0.5) {\faAndroid};
		\node[anchor=center] at (-2.5,-2.5) {\obstacleicon};
		\node[anchor=center] at (-0.5,-0.5) {\obstacleicon};
		\node[anchor=south] at (-3.5,3.1) [] {\small 0};
		\node[anchor=south] at (3.5,3.1) [] {\small 7};
		\node[anchor=south] at (0,3.1) [] {\small x};
		
		\node[anchor=east] at (-4.1,2.5) [] {\small 0};
		\node[anchor=east] at (-4.1,-2.5) [] {\small 5};
		\node[anchor=east] at (-4.1,0) [] {\small y};
	\end{tikzpicture}\hspace{.5cm}%
	\begin{tikzpicture}[
		node/.style={draw,ellipse,inner sep=1pt,minimum height=.6cm,font=\small},
		leaf/.style={draw,rectangle,rounded corners,minimum height=.6cm,minimum width=.6cm,rectangle,font=\small},
		tedge/.style={draw},
		fedge/.style={draw,dashed},
		auto,yscale=0.9
		]
		\node[node] at (0,0) (n1) {$y \leq 0$};
		\node[node] at (-1,-1) (n2) {$x \leq 6$};
		\node[leaf] at (-1.75,-2) (l1) {$\rightarrow$};
		\node[leaf] at (-0.25,-2) (l2) {$\downarrow$};
		\node[leaf] at (1,-1) (l3) {$\uparrow$};
		
		\path[->]
		(n1) edge[tedge,swap] node[inner sep=.5pt,font=\scriptsize] {Y} (n2)
		(n1) edge[fedge] node[inner sep=.5pt,font=\scriptsize] {N} (l3)
		(n2) edge[tedge,swap] node[inner sep=.5pt,font=\scriptsize] {Y} (l1)
		(n2) edge[fedge] node[inner sep=.5pt,font=\scriptsize] {N} (l2)
		;
		
	\end{tikzpicture}
	\caption{
		Visualisation of a grid-world robot navigation task (left) and a DT computed by our tool \dtcontroleps{} (right).
		In the example, the robot (\faAndroid) tries to reach the location in need of repair ({\color{green}\faTools}) while avoiding the~obstacles~(\obstacleicon).
	}
	\label{fig:example}
\end{figure}

We illustrate the core ideas of our tool: 
Consider a robot \raisebox{-1pt}{\faAndroid{}} in a grid world environment on a hillside.
The robot can move %
in the cardinal directions $\rightarrow,\leftarrow,\uparrow,\downarrow$,
but always has a $10\%$ chance of sliding south instead.
The goal is to reach the repair location \textcolor{green}{\scriptsize\faTools} while avoiding the obstacles {\small\obstacleicon{}}. %
\cref{fig:example} illustrates the environment and a near-optimal controller ($\varepsilon=0.001$) obtained by our tool.

Modelling this scenario as an MDP, we can use a probabilistic model checker (like \storm~\cite{storm}) to compute a controller achieving the optimal reachability probability.
Despite the simplicity of the example, this optimal controller consists of dozens of state-action mappings. %
The DT computed by \dtcontrol{}, accurately representing this controller, uses 33 nodes (see \cref{fig:dt_example_vanilla_reduced} (left) in \cref{app:dt_example_figures}).
In contrast, \dte{} uses \emph{information about the model} to apply several reductions,
e.g.\ that states with $x=0$ are not reached under the optimal controller and there is no need to learn choices for them.
This and similar observations already allow to reduce the DT size to 21 nodes (see \cref{fig:dt_example_vanilla_reduced} (right)). %
This reduced DT still achieves the optimal value.
By allowing for an imprecision $\varepsilon$, we can afford to make sub-optimal choices in some states
and prune parts of the DT that are only relevant with very small probability.
Concretely, allowing a minuscule deviation of $\varepsilon=0.001$, we can reduce the DT to \emph{five} nodes and only three decision paths (see \cref{fig:example} (right)).
This DT captures the \enquote{essence} of the behaviour:
Go up to the top border, then go right to the border, then go down.
Note that this controller is indeed not optimal, as it risks repeatedly sliding down into an obstacle by going up rather then moving to the right.

\subsection{Functionality and Usage}

Our tool fully automatically builds a DT representing a controller satisfying the problem stated above.
For example, to construct the DT in \cref{fig:example} (right) from the model in \cref{fig:example} (left), the call is:
\begin{Verbatim}[fontsize=\footnotesize]
 $ python main.py --model tinyrobot.prism \
 --property 'Pmax=?[(!"dead") U "repair"]' \ 
 build --precision 1e-3 --export-dot tinyrobot.dot
\end{Verbatim}
The output DT is given in JSON or GraphViz format. %
The former can be used to apply a DT to different versions of a model, in the spirit of~\cite{DTstrat}.

\subsection{Tool Structure}\label{sec:3-workflow}

\begin{figure*}[t!]
 \centering
 \tikzset{
 partbox/.style={font=\scriptsize,minimum width=1.4cm,rectangle,rounded corners},
 toolstep/.style={partbox,minimum height=.6cm,draw,thick},
 iobox/.style={partbox,fill=gray!20,minimum height=.5cm},
 externalbox/.style={draw=gray,rectangle,rounded corners,text=darkgray!80,thick},
 stormbox/.style={externalbox,fill=T-Q-MC3!20,inner sep=3pt,thick,font=\scriptsize\stormicon\,},
 stormarrow/.style={<->,>={Stealth[scale=.6]},thick,draw=darkgray!60},
 stepedge/.style={->,>=stealth,thick},
 ioedge/.style={->,>=stealth,thick},
 }%
 \newcommand{\safeopticon}{\textcolor{green}\faCheck}%
 \newcommand{\safeoptitem}{\safeopticon}%
 \newcommand{\boundedsafeopticon}{\textcolor{green}\faCheckCircle}%
 \newcommand{\boundedsafeoptitem}{\boundedsafeopticon}%
 \newcommand{\maybopticon}{\textcolor{yellow}{\faQuestion}}%
 \newcommand{\mayboptitem}{\hspace{.5pt}\maybopticon{}\hspace{.5pt}}%
 \newcommand{\aggropticon}{\textcolor{red}{\faFire*}}%
 \newcommand{\aggroptitem}{\hspace{.5pt}\aggropticon\hspace{.5pt}}%
 \newcommand{\stormicon}{\textcolor{yellow}{\faCloudShowersHeavy}}%
 \centering
 \resizebox{0.8\textwidth}{!}{
 \begin{tikzpicture}[node distance=0.5cm,remember picture]
 \node[iobox,text width=1.75cm,align=center] at (0,0) (input) {Input\tiny\begin{itemize}[topsep=2pt,labelsep=2pt,leftmargin=6pt,label=$\bullet$]\item Model\item Objective\item Precision $\varepsilon{\geq}0$\end{itemize}};
 
 \node[toolstep,right=1cm of input] (dtbuild) {1: Solve};
 \node[anchor=north,align=left,font=\tiny,text width=2cm] at (dtbuild.south) {\begin{itemize}[topsep=0pt,labelsep=2pt,label=$\bullet$]\item Model\item opt.\ value\item opt.\ $\controller$\end{itemize}};
 
 \node[toolstep,right=of dtbuild] (dtreduce) {2: Dataset};
 \node[anchor=north,align=center,font=\tiny] at (dtreduce.south) (dtreduceexp) {\begin{minipage}{1.5cm}Initial:\begin{itemize}[topsep=2pt,labelsep=2pt,leftmargin=7pt]\item[\safeoptitem]\datasetController\item[\safeopticon]\datasetPermissive\item[\safeopticon]\datasetCombined\end{itemize}Reduction:
 \begin{itemize}[topsep=2pt,labelsep=2pt,leftmargin=7pt]
 \item[\safeoptitem]\unreach($\mathcal{D}$)
 \item[\safeoptitem]\allSafe
 \item[\boundedsafeoptitem]\agency
 \item[\safeoptitem]\dominance
 \end{itemize}\vfill\end{minipage}};
 
 \node[toolstep,right=of dtreduce] (dtlearn) {3: Learn DT};
 \node[anchor=north,align=center,externalbox,font=\tiny,outer sep=1pt] at ([yshift=-2pt]dtlearn.south) (dtlearndtcontrol) {\dtcontrol{} \cite{dtcontrol2}};
 \node[anchor=north,align=center,font=\tiny,outer sep=1pt,inner sep=0pt] at (dtlearndtcontrol.south) (dtlearnplus) {\faPlus};
 \node[anchor=north,align=center,font=\tiny,outer ysep=-1.5pt,inner sep=0pt,text width=1.9cm] at (dtlearnplus.south) (dtlearnexp) {\begin{itemize}[topsep=0pt,labelsep=2pt,leftmargin=7pt]\item[\safeoptitem]Better \predicates\item[\safeoptitem]Better \impurity\tikzmark{impuritytext}\begin{itemize}[topsep=0pt,labelsep=2pt,leftmargin=7pt]\item[\safeoptitem]Agency\item[\safeoptitem]Cost of Error\item[\safeopticon]Simulation\item[\aggroptitem]CAV15 \cite{CAV15}\end{itemize}\item[\aggroptitem]\earlystop\item[\aggroptitem]Pruning\end{itemize}};

 \node[stormbox,above=.4 cm of dtbuild] (stormbuild) {Build+Solve \cite{storm}};

 \node[stormbox,above=.4 cm of dtlearn] (stormcheck) {Check \cite{storm}};

 \begin{scope}[on background layer]
 \node[fit=(dtbuild) (dtreduceexp) (dtlearnexp),inner sep=.2cm,draw=T-Q-MC4,rectangle,rounded corners=.2cm,ultra thick,fill=T-Q-MC4!5] (dtbox) {};
 \end{scope}
 
 \node[draw=T-Q-MC4,rectangle,rounded corners=4pt,thick,fill=white,anchor=west,inner sep=4pt] at ([xshift=.3cm]dtbox.south west){\scriptsize\dtcontroleps{}};
 
 \node[iobox,right=1cm of dtlearn,text width=1.9cm,align=center] (output) {DT repr.\\of $\varepsilon$-opt.\ $\controller$\tiny\begin{itemize}[topsep=2pt,labelsep=2pt,leftmargin=6pt,label=$\bullet$]\item GraphViz (dot)\item JSON\item \dots\end{itemize}};
 
 \path[stormarrow]
 (dtbuild) edge (stormbuild)
 (dtlearn) edge (stormcheck)
 ;
 \path[stepedge]
 (dtbuild) edge (dtreduce)
 (dtreduce) edge (dtlearn)
 ;
 \path[ioedge]
 (input) edge (dtbuild)
 (dtlearn) edge (output)
 ;
 \end{tikzpicture}}
 \caption{
 Structural overview of \dtcontroleps.
 The tool builds a DT %
 representing an $\varepsilon$-optimal controller for an MDP. %
 Given the model description, objective, and precision, our tool calls \storm (\stormicon)~\cite{storm} to build the model and obtain the optimal value and controller.
 This controller is transformed into a reduced dataset of relevant decisions using our, in the DT-learning context novel, \enquote{safe} methods preserving the value of the controller (\safeopticon), and those that can result in bounded imprecision (\boundedsafeopticon).
 Then, we learn a small DT consistent with this dataset, potentially employing \enquote{aggressive} (\aggropticon{}) heuristics, potentially resulting in unbounded imprecision.
 Using \storm, we check the $\varepsilon$-optimality of the constructed DT and, in case of violation, 
 reduce aggressive optimisations.
 The new tool components are initial dataset construction and reduction, predicate extraction, improved impurity measures except \cite{CAV15}, early stopping, pruning, and, enabling all of these, the overall pipeline and connection to \storm.
 }
 \label{fig:tool-overview}
\end{figure*}
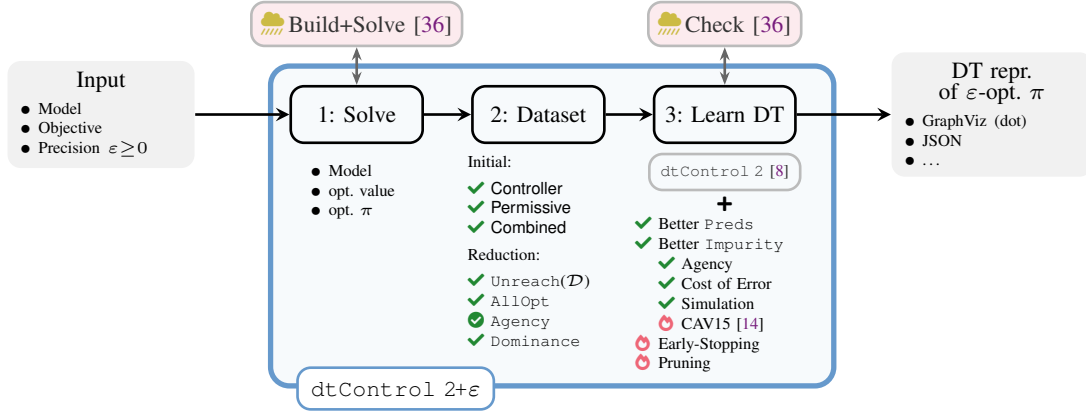

\noindent
We summarise the major steps of the workflow of \dte{} (see \cref{fig:tool-overview}).
Technical insights are provided in the following section.

\para{Step~1: \stepOne.}
We use \stormpy to construct the provided model (supporting both PRISM and JANI format), parse the objective, and compute the optimal value and an optimal controller for the given objective.

\para{Step~2: \stepTwo.}
Using the optimal value and controller, we construct
an initial dataset $\mathcal{D}$, where any controller \emph{consistent} with it (i.e., $\decisiontree(s) \in \mathcal{D}(s)$ for all $s$ where $\mathcal{D}$ is defined) is guaranteed to be ($\varepsilon$-)optimal; especially, these can be \emph{permissive} 
(non-deterministic),
enabling determinisation heuristics.
Moreover, we identify and remove \emph{irrelevant} states, i.e.\ where the controller's choice does not affect the obtained value, using model- and objective-specific information.
In contrast, for \dtcontrol, the user has to manually obtain a (non-permissive) controller from \storm for which \dtc then learns all decisions, including for irrelevant states.
We discuss the details in \cref{sec:step2}.

\para{Step~3: \stepThree.}
We employ our adaptation of \dtcontrol to learn a DT based on the reduced dataset.
Among other improvements, we extract model-specific predicates as potential decision nodes in the tree and employ specialised impurity calculations and pruning techniques.
We provide details in \cref{sec:step3}.

\begin{remark}
Most of our improvements in Steps~2 and 3 rely on the tight integration with the model checking pipeline;
e.g.,\ by automatically deriving algebraic predicates from the model description, which previously required expert guidance~\cite{dtcontrol2}.
Moreover, we can employ \enquote{aggressive} heuristics, namely~\cite{CAV15}, early-stopping or pruning, which in some cases might yield DTs that are not $\varepsilon$-optimal ($|\val^{\decisiontree}(\initstate) - \val(\initstate)| > \varepsilon$).
We check this using model checking and, if needed, relax the aggressive heuristics and repeat Step 3.
\end{remark}

\begin{remark}
	\dte{} employs the same pipeline for all objectives, with adjustments within each step.
	While most objectives, like reachability, safety, and total reward, can be achieved with a memoryless controller, this is not the case for LTL, where the controller is extended with a finite set of memory states $M$.
	Therefore, the DT maps from a state-memory pair, instead just the state, to an action, i.e.\ the action chosen in a state depends on the current memory state (see \cref{app:ltl} for details).
	Additionally, we output the underlying memory structure in form of an automaton.
Please note that while some works consider automata as an explainable structure, e.g. \cite{neider2016automaton, carr2021verifiable}, for instance the three state automata for Until formulae, this is disputable for more complex properties. There, semantic structure of automata states, e.g.~\cite{cav23}, could provide information for future approaches. %
\end{remark}

\section{Details of Step 2: \stepTwo}
\label{sec:step2}

Following the procedure of our tool, we first provide insights into the creation of the initial dataset, which is then reduced using our heuristics.

\subsection{Creation of the Initial Dataset/Controller}
When creating the initial dataset, we can either focus on a controller that is itself small and only contains one action per state, or we can try to start with a permissive controller.
For DT learning, a \emph{permissive} dataset, i.e.\ one that maps states to a \emph{set} of allowed actions, is often advantageous~\cite{DBLP:conf/hybrid/AshokJJKWZ20,dtcontrol2}.
When learning a DT \emph{consistent} with a permissive dataset ($\decisiontree(s) \in \mathcal{D}(s)$ for all $s$ where $\mathcal{D}$ is defined), it is often possible to terminate the DT construction earlier by using that actions are shared between states and determinizing by choosing one of them.
For example, consider the optimal controller $\controller$ for the MDP in \cref{fig:permissive} with $\controller(s)=l$ and $\controller(t)=r$.
DT construction would yield a 3-node tree when representing the controller exactly.
By additionally allowing action $r$ in $s$, we can obtain a simpler 1-node tree that always picks $r$.
Previously, the initial dataset was not permissive because tools like \storm export only deterministic optimal controllers.
While, as discussed, a more permissive dataset generally yields good results, it has the disadvantage that more states are potentially reachable and need to be considered while learning.
In particular, if $\uniforma \in \mathcal{D}(s)$ then all potential successors of $s$ need to be considered.
Additionally, adding actions to the permissive dataset requires thought to not break optimality.%
\begin{example}\label{ex:progress}
	One might be tempted to simply allow all \enquote{optimal} actions, i.e.\ those yielding optimal value, by $\mathcal{D}(s) = \{a \in \actions(s) \mid \val(s, a) = \val(s)\}$.
	Indeed, this is correct for, e.g., safety.
	However, for other objectives such as reachability and total reward, it is not:
	Observe that all states and actions in \cref{fig:permissive} have a value of 1.
	If both actions are allowed for $s$ and $t$, then the controller picking $r$ in $s$ and $l$ in $t$ is dataset-consistent %
	but never reaches the target, yielding value 0.
	Intuitively, this is because it makes no \enquote{progress}.
	Moreover, there is no unique inclusion-maximal permissive dataset; both $\{s \to \{l\}, t \to \{l, r\} \}$ and $\{s \to \{l, r\}, t \to \{r\}\}$ are optimal.
\end{example}
Thus, we introduce two methods for obtaining \emph{some} optimal permissive dataset.

\permissive:
Building on \cite{DBLP:conf/tacas/ChatterjeeQSWWZ25}, 
we allow all actions with optimal value that reduce the distance to the goal with non-zero probability.
In the above example, this yields $\mathcal{D} = \{s \to \{l\}, t \to \{r\}\}$.
While this cannot break symmetry when states have the same distance to the goal, %
it often adds actions.
For objectives not requiring progress, e.g.\ safety, \permissive{} adds all actions with optimal value.

\eqval:
We can further exploit states where all actions have the same value.
There, only making progress with some probability matters.
If all actions in a state have the same value, we add the uniform-choice label $\uniforma$. 
This allows the DT to \enquote{unify} states which otherwise have disjoint sets of allowed actions, as follows:
In the example, $\uniforma$ is added for both states, i.e.\ $\mathcal{D} = \{s \to \{l, \uniforma\}, t \to \{r, \uniforma\}\}$, and thus there is a consistent one-node DT, recommending $\uniforma$.

Note that explicitly adding $\uniforma$ is different from setting $\mathcal{D}(s) = \actions(s)$ or removing $s$ from $\mathcal{D}$:
In the latter two cases, the constructed DT $\decisiontree$ may produce any action, even one which potentially compromises optimality.
($\decisiontree(s)$ \emph{may} lead to uniform choice, but is not required to do so.)
In contrast, when adding $\uniforma$, the strategy resulting from the DT chooses an action uniformly at random, ensuring progress. %
For objectives not requiring progress, \eqval{} is unnecessary:
If all actions are optimal, \permissive{} adds all actions, and then we remove the state entirely by \allSafe\ (defined below).

Based on the heuristics named above, we propose three variants of obtaining an initial dataset, or, in other words, the initial controller.
\begin{description}
	\item[\datasetController] Firstly, we directly use the dataset given by the optimal controller, i.e.\ $\mathcal{D}(s) = \{\controller(a)\}$, without any permissive choices.
	\item[\datasetPermissive] Secondly, as a fully permissive approach, we use the dataset computed by \permissive{} and \eqval, giving the best chances for determinization, but potentially increasing the size of the reachable state space.
	\item[\datasetCombined] Lastly, %
	we restrict the \datasetPermissive\ dataset to the set of states $\states_R$ reached by the optimal controller from the initial states
	by removing all actions leaving $\states_R$, i.e.\ those with $\trans(s, a)(s') > 0$ for some $s' \notin \states_R$.%
\end{description}

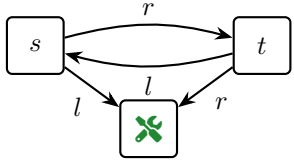
\begin{figure}[t]
	\centering
	\begin{tikzpicture}[state/.append style={draw,rectangle,rounded corners=2pt,inner sep=1pt,minimum size=.75cm,thick}]
		\node[state] at (0,0) (s) {$s$};
		\node[state] at (3,0) (t) {$t$};
		\node[state] at (1.5,-1.1) (g) {\color{green}\faTools};
		\path[-{Stealth},thick]
		(s) edge[bend left=15] node[above] {$r$} (t)
		(t) edge[bend left=15] node[below] {$l$} (s)
		(s) edge node[anchor=north east] {$l$} (g)
		(t) edge node[anchor=north west] {$r$} (g)
		;
	\end{tikzpicture}
	\caption{
		MDP illustrating permissive strategies.
	} \label{fig:permissive}
\end{figure}

\begin{remark}
	Obtaining a dataset this way is fundamentally different from obtaining an $\varepsilon$-optimal controller and using that for DT-learning.
	Such a controller would assign one specific action to every single state, even if that choice is entirely inconsequential, leaving us with a much more restricted dataset.
	Moreover, we lose all \enquote{imprecision budget} for simplification -- as the starting controller may already be sub-optimal, any additional imprecision might compromise the overall guarantees.
\end{remark}

\subsection{Dataset Reduction}
After the computation of the initial dataset, we can identify states that can be removed entirely from the dataset, as choosing any actions in these states ensures optimality.
We list the methods we use to identify these states.
\smallskip

\allSafe:
If a dataset allows all actions in a state $s$, i.e.\ $\mathcal{D}(s) = \actions(s)$, then it can be removed.
Recall that if the DT suggests an action not available in $s$, then one of the available actions is chosen uniformly at random.
This in particular removes states with a single available action, i.e.\ $|\actions(s)| = 1$.

\agency~\cite{DBLP:conf/ijcai/CordobaJABSSPK23}: 
The agency of a state $s$ is the maximum error that can be incurred by playing adversarially from $s$, formally $\prob_{s}^{\max}[\lozenge \gStates] - \prob_{s}^{\min}[\lozenge \gStates]$. 
States with agency 0 can be removed without affecting the value of the controller.
Moreover, we can exploit the allowed imprecision: Removing states with agency less than $\varepsilon$ only reduces the controller's value by at most $\varepsilon$.

\dominance~\cite{DBLP:conf/papm/DArgenioJJL02}:
A state $s$ is \emph{dominated} by a state $s' \neq s$ if $s'$ is almost surely reached from $s$, %
i.e. $\prob_s^{\min}[\lozenge s'] = 1$.
Thus, $s$ can be removed, since $s'$ is reached regardless of the choice in $s$.
Slight adaptations are necessary for reward objectives.
The intuitive idea of \dominance{} is that on the way from $s$ to $s'$, choices are irrelevant, as no matter what we do, we always eventually end up in the same spot, namely $s'$.
In other words, we want all ways from $s$ to $s'$ to be the \enquote{same} to the objective.
Hence, for reward-based objectives, we say that $s'$ dominates $s$ if $s'$ is a.s.\ reached from $s$ and under all controllers, the same reward is obtained in between.
Formally, if $\states'$ are all the states that are dominated by $s'$ and we have two controllers $\controller$ and $\controller'$ which only differ on $\states'$, then we have $\val^{\controller}(s) = \val^{\controller}(s') + c = \val^{\controller'}(s') + c = \val^{\controller}(s)$.
To compute this set of states, we augment the classical algorithm for finding dominating states from \cite{DBLP:conf/papm/DArgenioJJL02} by additionally tracking whether a state can be reached by two different rewards.

\unreach: 
Since we only care about the value in the initial state $\initstate$, we let \unreach($\mathcal{D}$) denote all states from $\mathcal{D}$ unreachable under any controller consistent with $\mathcal{D}$ when starting from the initial state.
Removing these does not affect the achieved value.

\smallskip
We highlight that these four conditions \allSafe, \agency, \dominance, and \unreach\ are \emph{incomparable}, i.e.\ none subsumes another, see \cref{app:4-2-example-incomparable}, and are always safe:

\begin{restatable}{theorem}{reduceOK}\label{thm:reduction-safe}
	For every MDP and precision $\varepsilon$, every dataset $\mathcal{D}$ obtained using our initial dataset approaches and reduction procedure is $\varepsilon$-optimal, i.e.\ for every controller $\controller$ consistent with $\mathcal{D}$ we have $|\val(\initstate) - \val^{\controller}(\initstate)| \leq \varepsilon$.
\end{restatable}
We provide an outline of the proof here and, the full proof can be found in \cref{app:4-2-reduction-safe}.

\begin{proof}[Proof Sketch]
	We first argue that any initial dataset only contains optimal controllers and then that all reductions only introduce an $\varepsilon$-error at most.
	For the initial datasets, \datasetController{} is correct by definition, \datasetPermissive{} follows from \cite{DBLP:conf/tacas/ChatterjeeQSWWZ25} (both \permissive{} and \eqval{}), in particular their Section~3.
	Finally, for \datasetCombined{}, observe that we never remove all actions, as it contains \datasetController{}.
	Moreover, it is a subset of \datasetPermissive{}.
	As any controller compatible with the latter is optimal, so is any compatible with the former.

	For the reductions, \unreach{} clearly does not change the value.
	For \dominance, correctness similarly follows directly by definition: No matter what we do, we always end up in the dominating state.
	Similarly, \agency{} says that no matter what we do, our overall value can only decrease by at most $\varepsilon$.
	Here, the only subtlety is to observe that this error does not compound.
	This however follows from the \enquote{global} definition of \agency.
	Finally, for \allSafe{}, note that as long as the DT outputs any action in the available actions, it is optimal by definition.
	If not, we uniformly choose among the available actions, which for all objectives yields the average value of the actions we randomize over, preserving optimality.
	Here, the only subtlety lies in objectives requiring progress.
	However, we make progress with non-zero probability, preserving the value.
\end{proof}

\section{Details of Step~3: \stepThree}\label{sec:step3}
Our tool provides three novelties regarding safe DT learning: 
Firstly, we automatically extract relevant predicates from the high-level description of the model in \texttt{PRISM} language.
Secondly, we improve the impurity calculation using a model-specific measure of importance.
Thirdly, we apply aggressive early stopping and pruning heuristics, which was previously impossible to do safely, as $\varepsilon$-optimality could not be guaranteed without re-validation.

\subsection{\predicates}
We offer axis-aligned and linear splits as already implemented in \dtcontrol.
This split divides the input space by choosing one variable and comparing it to a constant.
An example is $y \leq 0$ from the example in \cref{fig:example}.
Additionally, we propose a new set of predicates:
When the model is given in the \texttt{PRISM} modelling language~\cite{KNP11-prism}, we can access the expressions defining its dynamics.
Using the transition guards as predicates gives us a fully automated way of exploiting domain knowledge and obtaining predicates that are likely relevant for the behaviour of the MDP.
This improves on the semi-automatic approach in~\cite{dtcontrol2}, where templates for predicates were manually provided using domain knowledge.

\subsection{\impurity}\label{sec:step3_impurity}
As previously mentioned, the detailed analysis in~\cite{dtcontrol2} suggested that the best choice for \impurity measure is Shannon's entropy~\cite{DBLP:journals/bstj/Shannon48a}, which is defined as $H(D) = \prob(\pred(D))\log\prob(\pred(D)) + \prob(\lnot\pred(D))\log\prob(\lnot\pred(D))$.
Usually, all states are weighed uniformly when computing the probability; we denote impurity with uniform weights by \wuniform.
We complement the idea by further weighting the impact of states with a model-specific measure of importance. %
We provide four heuristics:

\smallskip

\wcav{} and \wsimulation:
\wcav{} denotes the idea of~\cite{CAV15}: simulate the system under the optimal controller several times and weigh each state with the number of times it is sampled.
As states may get weight 0, this importance heuristic implicitly reduces the dataset, potentially compromising $\varepsilon$-optimality.
Thus, our variant \wsimulation~assigns a positive weight of $10^{-6}$ to unsampled states to ensure $\varepsilon$-optimality.

\wagency:
We use the values computed by the reduction method \agency{} (maximum vs.\ minimum probability to reach the goal) as weights.
Removing states satisfying \agency{} is equivalent to assigning weight 0 in the impurity computation, so this importance heuristic generalises the \agency{} reduction. 
For other objectives, the weights are accordingly adjusted, e.g., the difference in maximum and minimum reward.

\wmessup{} (Cost of Error):
While \agency{} captures the error incurred if we would choose the worst action in \emph{all} states, \wmessup\ considers the impact of only selecting the worst action in \emph{one} state.
A precise computation can in the worst case require a model checking query for every state (since states can depend on themselves).
Thus, we approximate this by comparing the value of the worst and the best action of a state.
Observe that this heuristic generalises \allSafe, as the cost of error in states with equal value on all actions is 0.
To ensure $\varepsilon$-optimality, similar to \wsimulation, we add a small positive weight to all states.

\subsection{\earlystop}
The idea of stopping early is standard in machine learning, but has not been employed in the context of controller representation, since so far, optimality had to be preserved.
Exploiting the allowed imprecisions, we adopt this idea as follows.
When the current impurity is below a given threshold $\theta \geq 0$, we return a leaf node. 
This is an extension of the stopping condition of \dtcontrol (Line~1 of \cref{alg:DT-construct}) which requires all states to agree on an action. 
As the choice of label is no longer straightforward, we label the leaf with an action determined through weighted voting, described below.
We start with a threshold of $\theta=0.5$ and decrease it by 0.1 if the resulting controller is not $\varepsilon$-optimal; in the worst case (when the heuristic always increases the imprecision to more than $\varepsilon$), this results in a precise DT after 6 iterations.

\paragraph{Weighted Voting}
Our \earlystop and Pruning heuristics require to decide on an action for a node when not all states agree.
We employ the following weighted voting approach:
Each state $s$ has weight $w(s)$ according to the \impurity{} method we used for building the DT (see \cref{sec:step3_impurity}).
Then, denoting by $\mathcal{D}$ the given dataset and by $X$ the set of states in the training dataset that reach the node, we select ${\argmax}_{a\in A} {\sum}_{s\in X, a \in \mathcal{D}(s)} w(s)$, i.e.\ the action with the highest weight when summing up the weights of states for which this action is allowed by the dataset/controller. %

\subsection{Pruning}

\begin{algorithm}[t]
	\caption{Greedy Pruning for $\max$}\label{alg:pruning}
	\begin{algorithmic}[1]
		\Require Optimal value $v$, precision $\varepsilon$, $\varepsilon$-precise DT $\decisiontree$
		\Ensure Reduced $\varepsilon$-precise decision tree $\decisiontree'$
		\State $\decisiontree' \gets \decisiontree$
		\If{$v \leq \varepsilon$}
		\Return trivial tree
		\EndIf
		\While{\ltrue}
		\State $b \gets 0$
		\For{\text{Node} $n$ in $\decisiontree'$}
		\If{both children of $n$ are leaves}
		\State $\decisiontree_n \gets \Call{Collapse}{\decisiontree', n}$
		\State $v_n \gets \Call{CheckValue}{\decisiontree_n, \initstate}$
		\If{$b < v_n$} \Comment{Found a better pruned tree}
		\State $b \gets v_n$, $\decisiontree_b \gets \decisiontree_n$
		\EndIf
		\EndIf
		\EndFor
		\If{$v - b > \varepsilon$} \Comment{No pruned tree is $\varepsilon$-optimal}
		\State \Return $\decisiontree'$
		\EndIf
		\State $\decisiontree' \gets \decisiontree_b$ \Comment{$\decisiontree_b$ is the best performing prune}
		\EndWhile
	\end{algorithmic}
\end{algorithm}
As part of building a reduced DT, we also apply a greedy pruning approach (see \cref{alg:pruning}), meaning that we try to delete nodes that have little influence on the value of the controller.
In each iteration, we go over all nodes which have leaves as both their children.
For each of these nodes, we obtain a new DT by replacing that inner node with a leaf whose label is determined by weighted voting as for \earlystop.
Then we model-check each of these candidates.
If any of these DTs still satisfy the precision constraints, we choose the one that decreases the value the least and repeat the process with the new DT.
As this requires several model checking queries, the process is quite expensive.
Nonetheless, since our goal is to obtain DTs that are as small as possible, it can be worth investing the time.
We outline the algorithm for $\opt = \max$ in \cref{alg:pruning}; the variant for $\min$ is completely analogous.

Note that even when early stopping is enabled, pruning can still reduce the size of the tree.
Early stopping focuses on the quality of the prediction with respect to the given dataset, while pruning focuses on the learnt controller.
This means that some splits, while necessary to represent the dataset well, are in the end not really important to the controller.

We comment on two further related mechanisms:
Firstly, in \cite{DBLP:conf/qest/AshokKLCTW19}, a heuristic called \enquote{safe pruning} was presented.
However, this heuristic was exploiting the non-determinism present in the given controller, and could only remove nodes if the resulting controller was still exactly represented.
In contrast, we allow for the resulting controller to be less defined, potentially worsening its performance; then, we rely on the model checking to ensure $\varepsilon$-optimality.
Secondly, the core idea of \cite{DTNest} is to take the DT constructed by \dtcontrol\ and then prune it by looking for subtrees that can be reduced using the exhaustive search of~\cite{andriushchenko2022inductivesynthesisfinitestatecontrollers}.
In principle, this post-processing can also be applied to our output; however, the prototypical implementation of \cite{DTNest} does not allow for this.

\section{Evaluation}
\label{sec:ev}

Our evaluation focuses on the following questions:
\begin{description}
	\item[RQ1:] How do DT sizes compare between \dte and its competitors?
	\item[RQ2:] What is the impact of allowing imprecision $\varepsilon$?
	\item[RQ3:] How do dataset reduction, early stopping, and pruning influence DT size?
	\item[RQ4:] Can \dte handle large models in reasonable time?
\end{description}
We additionally explain some obtained DTs in case studies. %
\cref{app:exp-add} contains full results, and, in particular, the impact of different methods for dataset construction and DT learning.

\subsection{Experimental Setup}

\para{Configurations.}
\dte has 12 different configurations, resulting from three initial datasets \datasetController, \datasetPermissive, and \datasetCombined, and four impurity measures \wsimulation, \wagency, \wmessup, \wuniform. %
We consider three variants of these configurations: 
\emph{our best}, i.e. running all 12 configurations and picking the smallest DT;
a \emph{portfolio} of three configurations to improve practical scalability (\datasetController+\wsimulation, \datasetPermissive+\wmessup, and \datasetCombined+\wagency);
and \datasetCombined+\wagency{} as example single configuration.

\para{Competitors}
are \dtcontrol{}~\cite{dtcontrol2},  \dtpaynt~\cite{dtpaynt}, \dtnest~\cite{DTNest}, and \enquote{CAV15}, our reimplementation of~\cite{CAV15} (the original code is unavailable).
\cref{app:exp-competitors} provides the exact versions and describes technical challenges.

\para{Benchmarks.} 
We select models from two sources:
first, all instances from the quantitative verification benchmark set \cite{qcomp-benchmarkset} which are in \texttt{PRISM} language and have between \num{40000} and \num{500000} states; 
for models with more than one supported property, we take the first of each type, e.g., only one for reachability and one for total reward.
Second, all instances used in~\cite{dtcontrol2} which are in \texttt{PRISM} language, have at least \num{100} states, and are not included (with some parametrisation) in our first selection.
Overall, we obtain 38 benchmarks: 34 from the first selection and 4 from the second.
While \dte supports other modelling languages, we focus on \texttt{PRISM} models, as these allow automatic extraction of predicates.
For LTL, there are unfortunately neither any established non-trivial (i.e. beyond a simple Until) MDP benchmarks \cite[Sec.~5]{DBLP:conf/toolympics/AndriushchenkoB23}, nor a competitor able to handle them (see \cref{app:ltl}).

\para{Quality Measure and Plots.}
We collect the runtime of the tools and the size (i.e.\ its number of nodes) of constructed DTs. 
We present our data as quantile plots, which are suitable for comparing several methods at once. 
A point $(x, y)$ for method X indicates that the $x$ best instances of X measured at most $y$ in the considered metric (size or time).

\para{Execution Environment.}
All executions for our benchmarks were run on a server with AMD EPYC 9274F CPUs inside a Docker container restricted to a single core and 16 GB of RAM.
The timeout for \dte{} is 5 minutes for model solving, 10 minutes for dataset creation and DT learning, and an additional 5 minutes for pruning (if applicable).
The competitors have a timeout of 20 minutes.

\subsection{Results}

\begin{figure*}[t]
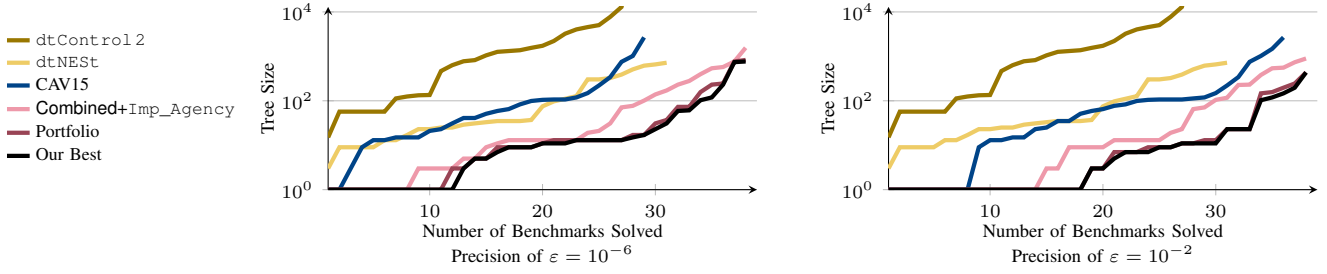
%
	\centering%
	\setlength{\quantileplotheight}{4cm}%
	\setlength{\quantileplotwidth}{0.4\textwidth}%
	\renewcommand{\quantileplotxlabel}{\empty}%
	\renewcommand{\quantileplotylabel}{\empty}%
	\begin{minipage}[b]{0.6\textwidth}\vspace{0cm}%
		\renewcommand{\quantileplotylabel}{Tree Size}%
		\renewcommand{\quantileplotxlabel}{Trees without Pruning}%
		\centering%
		\renewcommand{\quantileplotxlabel}{\shortstack{Number of Benchmarks Solved \\ Precision of $\varepsilon=10^{-6}$}}%
		\quantileplot{results/ablation-results/quantile_final_tree_nodes_with_dtnest.csv}
		{
			vanilla-0.000000/color4,
			dtnest-7-0/color1,
			cav15-0.000001/color6,
			ours-combined-agency-simple-0.000001/color2,
			portfolio-0.000001/color5,
			virtual-best-0.000001/color7
		}
		{
			\dtcontrol,
			\dtnest,
			CAV15,
			\datasetCombined+\wagency,
			Portfolio,
			Our Best
		}
		{1}{39}{1}{13000}{south west}{font=\small,at={(-0.2,0.5)},anchor=east,draw=none}%
	\end{minipage}%
	\begin{minipage}[b]{0.4\textwidth}\vspace{0cm}%
		\centering%
		\renewcommand{\quantileplotylabel}{Tree Size}
		\renewcommand{\quantileplotxlabel}{\shortstack{Number of Benchmarks Solved \\ Precision of $\varepsilon=10^{-2}$}}%
		\quantileplot{results/ablation-results/quantile_final_tree_nodes_with_dtnest.csv}
		{
			vanilla-0.000000/color4,
			dtnest-7-0/color1,
			cav15-0.010000/color6,
			ours-combined-agency-simple-0.010000/color2,
			portfolio-0.010000/color5,
			virtual-best-0.010000/color7
		}
		{}
		{1}{39}{1}{13000}{south west}{font=\small}%
	\end{minipage}%
	\caption{
		Quantile plot comparing the DT size of different configurations of \dte{} with \dtcontrol{}, \dtnest and CAV15 for $\varepsilon = 10^{-6}$ (left) and $\varepsilon = 10^{-2}$ (right).
		\dtcontrol and \dtnest have $\varepsilon = 0$ in both.
		The $y$-axis uses logarithmic scale. %
	}
	\label{fig:treesize}
\end{figure*}

\DeclareRobustCommand\legenddash[2]{\tikz[baseline=-2.75pt] \draw[line width=2pt,#1] (0,0) -- (0.3,0);#2}

\tikzset{
	dash1/.style={color1, dashed},
	dash2/.style={color2, dashed},
	dash3/.style={color3, dashed}
}

\begin{figure*}[t]
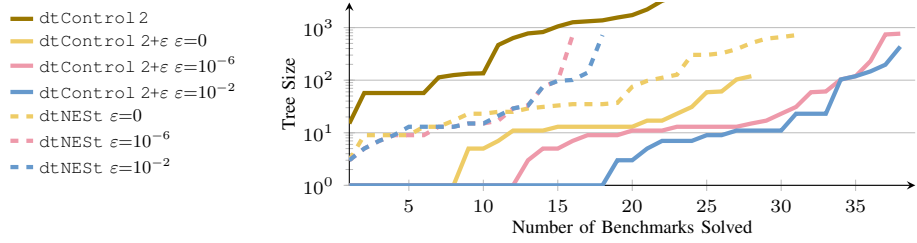

	\centering%
	\setlength{\quantileplotheight}{4cm}%
	\setlength{\quantileplotwidth}{0.5\textwidth}%
	\renewcommand{\quantileplotylabel}{Tree Size}%
	\renewcommand{\quantileplotxlabel}{Number of Benchmarks Solved}%
	\quantileplot{results/ablation-results/quantile_final_tree_nodes_with_dtnest.csv}
	{
		vanilla-0.000000/color4,
		virtual-best-0.000000/color1,
		virtual-best-0.000001/color2,
		virtual-best-0.010000/color3,
		dtnest-7-0/dash1,
		dtnest-7-1e-2/dash2,
		dtnest-7-1e-6/dash3
	}
	{
		\dtcontrol,
		\dte $\varepsilon{=}0$,
		\dte $\varepsilon{=}10^{-6}$,
		\dte $\varepsilon{=}10^{-2}$,
		\dtnest $\varepsilon{=}0$,
		\dtnest $\varepsilon{=}10^{-6}$,
		\dtnest $\varepsilon{=}10^{-2}$
	}
	{1}{39}{1}{3000}{south east}{font=\small,inner sep=.5pt,outer xsep=-2pt,nodes={scale=1,transform shape},at={(-0.2,0.5)},anchor=east,draw=none}%
	\caption{
		Comparison of \dtcontrol and \dtnest to our methods.
		The quantile plot compares the size of the final DT of \dtcontrol and \dtnest to the best DTs of \dte with different precision requirements.
		The $y$-axis uses logarithmic scale. %
	} \label{fig:dtc-evolution}
\end{figure*}

\para{RQ1: Comparison to Competitors.} 
\cref{fig:treesize} compares DT sizes for \dtcontrol, \dtnest, CAV15, and our three variants.
Our tool and CAV15 use an allowed imprecision~$\varepsilon$.
For \dtcontrol and \dtnest, we use $\varepsilon=0$ (i.e.\ same in both plots): The former cannot exploit it, the latter surprisingly performs worse for $\varepsilon>0$, see RQ2 below.
We omit \dtpaynt, which only solves 15 of 38 benchmarks, and produces larger DTs than \dte in 9 cases (see \cref{tab:tree-size} in \cref{app:subsec_tree_sizes}).
Our best DTs are always smaller than those of \dtcontrol, even by more than two orders of magnitude on 12 of 38 benchmarks.
Similarly, CAV15 is always outperformed by our best, and in only two cases beats the portfolio; also, CAV15 is aggressive and indeed sometimes produces DTs whose performance is suboptimal by more than $\varepsilon$.
For \dtnest{}, where both tools report solutions, our best DTs are always at least as small and more than an order of magnitude smaller in 10 cases (see \cref{fig:app_virtualbest-dtnest} in \cref{app:subsec_tree_sizes}).

Overall, \dte is a significant advancement of the state of the art, constructing very small DTs.
We highlight that when allowing an imprecision of $\varepsilon = 10^{-2}$, our tool returns DTs of explainable size (at most 15 nodes) in 30 of 38 benchmarks.
Moreover, we receive single-node DTs on 18 benchmarks (nearly half the benchmark set), identifying the previously unobserved fact that following a single action, and playing uniformly at random, when this action is not available, is $\varepsilon$-optimal in these benchmarks.

\para{RQ2: Effect of Precision.}
\cref{fig:dtc-evolution} shows DT sizes for \dte and \dtnest for several precisions $\varepsilon$ (including \dtcontrol for reference).
The improvement over \dtcontrol for $\varepsilon=0$ is due to the improved initial controller, safe heuristics and automatic predicate suggestions.
Allowing for imprecision always further reduces DT size.
For example, we obtain 15, 28, and 30 DTs of explainable size when $\varepsilon$ is $0$, $10^{-6}$ and $10^{-2}$, respectively.
In contrast, \dtnest can yield larger DTs for increasing $\varepsilon$ and mostly fails due to internal errors when the allowed imprecision is increased.

\begin{table*}
	\caption{
		Aggregated dataset size statistics (geometric mean, min, and max over all 38 benchmarks) for the three initial datasets \datasetPermissive, \datasetCombined, and \datasetController. The table shows the model size, and divides the resulting dataset size after all reduction heuristics of \stepTwo for varying precision~$\varepsilon$, the number of states removed by \agency for varying~$\varepsilon$, and the number of states removed by \allSafe, \dominance, and \unreach (which do not depend on~$\varepsilon$) by the model size.
	}\label{tab:state-statistics-agg}
	\centering
	\footnotesize
	\begin{tabular}{ll rrr rrr rrr}
		\toprule
		& & \multicolumn{3}{c}{Final Dataset Size} & \multicolumn{3}{c}{\agency} & & & \\
		\cmidrule(lr){3-5}\cmidrule(lr){6-8}
		Dataset & Stat & \multicolumn{1}{c}{$\varepsilon{=}0$} & \multicolumn{1}{c}{$10^{-6}$} & \multicolumn{1}{c}{$10^{-2}$} & \multicolumn{1}{c}{$0$} & \multicolumn{1}{c}{$10^{-6}$} & \multicolumn{1}{c}{$10^{-2}$} & \multicolumn{1}{c}{\allSafe} & \multicolumn{1}{c}{\dominance} & \multicolumn{1}{c}{\unreach} \\
		\midrule
		\begin{filecontents*}{results/ablation-results/state_statistics_agg.csv}
dataset,stat,Model Size,final-dataset-size-0.000000,final-dataset-size-0.000001,final-dataset-size-0.010000,agency-reduction-0.000000,agency-reduction-0.000001,agency-reduction-0.010000,all-actions-safe-states,dominated-states-reduction,reachability-reduction
\datasetPermissive,Mean,76782.6,0.45\%,0.78\%,0.75\%,8.08\%,8.65\%,8.65\%,62.84\%,29.88\%,38.39\%
\datasetPermissive,Min,235.0,0\%,0\%,0\%,0\%,0\%,0\%,1.98\%,0.98\%,0\%
\datasetPermissive,Max,530965.0,9.17\%,33.14\%,33.14\%,100.0\%,100.0\%,100.0\%,100.0\%,99.09\%,100.0\%
\datasetCombined,Mean,76782.6,0.13\%,0.14\%,0.13\%,8.08\%,8.65\%,8.65\%,62.84\%,29.88\%,62.94\%
\datasetCombined,Min,235.0,0\%,0\%,0\%,0\%,0\%,0\%,1.98\%,0.98\%,0.33\%
\datasetCombined,Max,530965.0,5.09\%,13.32\%,13.32\%,100.0\%,100.0\%,100.0\%,100.0\%,99.09\%,100.0\%
\datasetController,Mean,76782.6,0.37\%,0.39\%,0.39\%,8.08\%,8.65\%,8.65\%,30.82\%,29.88\%,58.9\%
\datasetController,Min,235.0,0\%,0\%,0\%,0\%,0\%,0\%,0.98\%,0.98\%,0.33\%
\datasetController,Max,530965.0,6.02\%,16.15\%,16.15\%,100.0\%,100.0\%,100.0\%,99.9\%,99.09\%,99.98\%
\end{filecontents*}
\csvreader[
		late after line={\\\ifnumequal{\thecsvrow}{4}{\midrule}{}\ifnumequal{\thecsvrow}{7}{\midrule}{}},
		late after last line=\\,]
		{results/ablation-results/state_statistics_agg.csv}{}
		{\compStateStatisticsAgg}
	\end{tabular}
\end{table*}

\begin{table*}
	\caption{
		Ratio of the final DT size obtained with a heuristic disabled to the DT size of the full pipeline, aggregated (geometric mean, min, max) over all 38 benchmarks and configurations for the three initial datasets and varying precision~$\varepsilon$.
	}\label{tab:ablation-dtsize}
	\centering
	\footnotesize
	\begin{tabular}{ll S[table-format=3.2,round-mode=places,round-precision=2] S[table-format=3.2,round-mode=places,round-precision=2] S[table-format=3.2,round-mode=places,round-precision=2] S[table-format=3.2,round-mode=places,round-precision=2] S[table-format=3.2,round-mode=places,round-precision=2] S[table-format=3.2,round-mode=places,round-precision=2] S[table-format=3.2,round-mode=places,round-precision=2] S[table-format=3.2,round-mode=places,round-precision=2] S[table-format=3.2,round-mode=places,round-precision=2]}
		\toprule
		& & \multicolumn{3}{c}{Disable Dataset Reduction} & \multicolumn{3}{c}{No Early Stopping} & \multicolumn{3}{c}{No Dataset Reduction or Early Stopping} \\
		\cmidrule(lr){3-5}\cmidrule(lr){6-8}\cmidrule(lr){9-11}
		Dataset & Stat & \multicolumn{1}{c}{$\varepsilon{=}0$} & \multicolumn{1}{c}{$10^{-6}$} & \multicolumn{1}{c}{$10^{-2}$} & \multicolumn{1}{c}{$0$} & \multicolumn{1}{c}{$10^{-6}$} & \multicolumn{1}{c}{$10^{-2}$} & \multicolumn{1}{c}{$0$} & \multicolumn{1}{c}{$10^{-6}$} & \multicolumn{1}{c}{$10^{-2}$} \\
		\midrule
		\begin{filecontents*}{results/ablation-analysis/ablation_size_summary.csv}
dataset,stat,nored-0.000000,nored-0.000001,nored-0.010000,nostop-0.000000,nostop-0.000001,nostop-0.010000,bothoff-0.000000,bothoff-0.000001,bothoff-0.010000
\datasetPermissive,Mean,3.35093886493384,2.529845502704834,2.1194800133993104,1.1427191017996157,1.2883912542232943,1.562795632001884,4.999495064328378,4.416429509769797,5.3266115451315805
\datasetPermissive,Min,0.09090909090909091,0.004589495155532891,0.004589495155532891,1.0,0.8947368421052632,0.7522123893805309,1.2857142857142858,1.0561199806482826,0.8592750533049041
\datasetPermissive,Max,221.0,203.0,219.0,16.68421052631579,47.0,49.0,209.0,211.0,221.0
\datasetCombined,Mean,2.6122651429492314,2.2953466530285067,2.17646345692096,1.0918622772901017,1.2510562544784838,1.6761722512411477,3.7419891767344127,3.876913649951073,5.1640348473546105
\datasetCombined,Min,0.6470588235294118,0.0989010989010989,0.0989010989010989,0.9098786828422877,0.922165820642978,0.4813979706877114,1.2701688555347093,1.0144404332129964,0.909433962264151
\datasetCombined,Max,209.0,209.0,209.0,9.0,13.31578947368421,40.142857142857146,269.0,239.0,233.0
\datasetController,Mean,2.7505345188646633,2.3882593784396775,2.187230449487921,1.0842611948137588,1.2577081122040945,1.6514439616245895,4.021637626093312,4.1636760407091735,5.599890319490528
\datasetController,Min,0.6470588235294118,0.6470588235294118,0.14754098360655737,0.8995433789954338,0.9305301645338209,0.5285913528591353,0.978021978021978,0.9079497907949791,0.9079497907949791
\datasetController,Max,725.0,725.0,725.0,9.0,28.142857142857142,54.714285714285715,773.0,773.0,773.0
\end{filecontents*}
\csvreader[
		late after line={\\\ifnumequal{\thecsvrow}{4}{\midrule}{}\ifnumequal{\thecsvrow}{7}{\midrule}{}},
		late after last line=\\,]
		{results/ablation-analysis/ablation_size_summary.csv}{}
		{\compAblationDtSize}
	\end{tabular}
\end{table*}

\begin{table}
	\caption{
		Effect of \stepFour within the full pipeline. The table reports the ratio of the DT size after pruning to the DT size before pruning, aggregated (geometric mean, min, max) over all 38 benchmarks and configurations for the three initial datasets and varying precision~$\varepsilon$. \emph{All} considers every instance; \emph{Attempted} restricts to instances where pruning changed the tree.
	}\label{tab:baseline-pruning}
	\centering
	\footnotesize
	\begin{tabular}{l l l S[table-format=1.2,round-mode=places,round-precision=2] S[table-format=1.2,round-mode=places,round-precision=2] S[table-format=1.2,round-mode=places,round-precision=2]}
		\toprule
		Dataset & Condition & Stat & \multicolumn{1}{c}{$\varepsilon{=}0$} & \multicolumn{1}{c}{$10^{-6}$} & \multicolumn{1}{c}{$10^{-2}$} \\
		\midrule
		\begin{filecontents*}{results/ablation-analysis/baseline_pruning_summary.csv}
dataset,condition,stat,0.000000,0.000001,0.010000
\datasetPermissive,All,Mean,0.6963737276739465,0.6808940119926633,0.4238686507274941
\datasetPermissive,All,Min,0.030303030303030304,0.030303030303030304,0.02040816326530612
\datasetPermissive,All,Max,1.0,1.0,1.0
\datasetPermissive,Attempted,Mean,0.3809895010370824,0.36521887248568485,0.21547746094904283
\datasetPermissive,Attempted,Min,0.030303030303030304,0.030303030303030304,0.02040816326530612
\datasetPermissive,Attempted,Max,0.9047619047619048,0.9047619047619048,0.8947368421052632
\datasetCombined,All,Mean,0.7213681678706343,0.6901640032799317,0.5203945949679097
\datasetCombined,All,Min,0.05263157894736842,0.05263157894736842,0.02127659574468085
\datasetCombined,All,Max,1.0,1.0,1.0
\datasetCombined,Attempted,Mean,0.39143163839465683,0.3783937631002024,0.23223107118421457
\datasetCombined,Attempted,Min,0.05263157894736842,0.05263157894736842,0.02127659574468085
\datasetCombined,Attempted,Max,0.9310344827586207,0.9428571428571428,0.8666666666666667
\datasetController,All,Mean,0.8360593419602071,0.7703741005828497,0.5776665815836113
\datasetController,All,Min,0.047619047619047616,0.047619047619047616,0.02127659574468085
\datasetController,All,Max,1.0,1.0,1.0
\datasetController,Attempted,Mean,0.5008115456989001,0.43774577864989284,0.2604502965908716
\datasetController,Attempted,Min,0.047619047619047616,0.047619047619047616,0.02127659574468085
\datasetController,Attempted,Max,0.9310344827586207,0.9393939393939394,0.9259259259259259
\end{filecontents*}
\csvreader[
		late after line={\\\ifnumequal{\thecsvrow}{7}{\midrule}{}\ifnumequal{\thecsvrow}{13}{\midrule}{}},
		late after last line=\\,]
		{results/ablation-analysis/baseline_pruning_summary.csv}{}
		{\compBaselinePruning}
	\end{tabular}
\end{table}

\para{RQ3: Influence of Different Heuristics.}
We first investigate the effect of the different dataset reduction heuristics in \stepTwo. %
\cref{tab:state-statistics-agg} reports the effect of the dataset reduction heuristics on the size of the dataset created in \stepTwo. 
To obtain comparable values, we divide the dataset size by the number of states in the model, which corresponds to the dataset size for \dtc, and then compute the geometric mean, the minimum and the maximum over all benchmarks and configurations.
We find that the final dataset is always significantly smaller than the number of states in the model.
However, when comparing the different values for $\varepsilon$, we observed an increased geometric mean from $\varepsilon=0$ to $\varepsilon=10^{-6}$.
This is caused by benchmarks that resulted in numerical instability and could not be treated for the former.
We further find that all reduction heuristics can help to decrease the dataset size.

To quantify the contribution of individual parts of our pipeline, we additionally rerun all 38 benchmarks with the reduction heuristics of \stepTwo (\allSafe, \agency, \dominance) disabled, with \earlystop disabled, and with both disabled, and compare against the full pipeline.
However, we always enable \unreach, as this heuristic is straight forward and often greatly reduces the dataset on its own.
\cref{tab:ablation-dtsize} reports the resulting effect on the pre-pruning DT size, as the ratio of the DT size obtained with a heuristic disabled to the DT size of the full pipeline. 
Values above~$1$ indicate that disabling the heuristic increases DT size.

On average, dataset reduction has the largest influence on \datasetPermissive, but disabling it on average more than doubles the DT size.
However, we also observe the fact, that the algorithm constructing the DT is only a heuristic, as the size may increase, even though the dataset size decreases through the reduction.
When considering early stopping, we also observe that tree size increases on average, but not as drastically as for dataset reduction.
The decrease of tree size observed in the row reporting the minima is caused by the simulation weight assignment and its inherent randomness.
Without both earl stopping and dataset reduction, the DT sizes increase on average by at least 4. 
Overall, the data shows that early stopping and dataset reduction are helpful to decrease DT size.

\cref{tab:baseline-pruning} reports the effect of pruning within the full pipeline, as the ratio of the DT size after pruning to the DT size before pruning, aggregated over all benchmarks and configurations (\emph{All}) and restricted to only those where pruning was attempted, which is the case for DTs with less than 50 nodes (\emph{Attempted}).
On average, pruning reduces DT size by $16$--$58\%$ across all instances, and by $50$--$79\%$ when restricting to instances where it was applied, with the effect increasing for larger~$\varepsilon$ in both cases.

\para{RQ4: Scalability.}
\cref{fig:treesize} shows that \dte{} solved all but one benchmark within the time limit of 20 minutes.
As our focus is on obtaining small explainable controllers, runtime is only important so long as the whole pipeline remains feasible.
Nevertheless, we include full results in \cref{tab:runtime-full,tab:runtime-dtnest_full} in \cref{app:exp-times}.
We remark that the runtime of \dtnest is similar, also handling all but one benchmark within the time limit, while \dtcontrol and \dtpaynt often time out.
Comparing the runtime of our methods to the size of the model, we observe no clear correlation, showing again that the structure is more important than size, cf.~\cite{pet2}. %
Overall, our methods are feasible even on large and complex models. %

\subsection{Case Studies}
\label{sec:case-studies}

We illustrate three qualitatively different outcomes that \dte{} can produce, each revealing a structural insight that is invisible to \dtcontrol{}.
Recall that in states not covered by the DT actions are chosen randomly.

\para{Pnueli-Zuck randomized mutual exclusion ($n{=}5$)~\cite{qcomp-benchmarkset,pnueli1984verification}}
models $n$ processes competing for a shared resource.
The non-determinism arises from the scheduler choosing which process to activate and from internal branching within each process.
The property maximises the probability that process~1 eventually reaches state $p_1 = 10$ on its way to the critical section.
The DT states, that, whenever $p_1=1$, it should be updated to $p_1=2$.
This  sends it along a path that reaches $p_1=10$ with probability~$1$.
Consequently, the optimal value is~$1$ in every reachable state, making all scheduling and branching choices simultaneously optimal.
\dtcontrol{} is unaware of this and encodes a decision for every state, yielding \num{171373} nodes.
\dte{} reduces this to a \emph{single-node} DT, showing that following a simple rule satisfies the property. %

\begin{figure}
		\centering
		\begin{tikzpicture}[baseline=(n0.center),
			node/.style={draw,ellipse,inner sep=1pt,minimum height=.6cm,font=\small},
			leaf/.style={draw,rectangle,rounded corners,minimum height=.6cm,minimum width=.6cm,font=\small},
			tedge/.style={draw},
			fedge/.style={draw,dashed},
			auto,yscale=0.9
			]
			\node[node] at (0,0) (n0) {$s_1 \leq 1$};
			\node[node] at (-1,-1) (n1) {$b \leq 1$};
			\node[leaf] at (-1.75,-2) (l1) {\emph{end1}};
			\node[leaf] at (-0.25,-2) (l2) {\emph{cd}};
			\node[leaf] at (1,-1) (l3) {\emph{end2}};
			\path[->]
			(n0) edge[tedge,swap] node[inner sep=.5pt,font=\scriptsize] {Y} (n1)
			(n0) edge[fedge] node[inner sep=.5pt,font=\scriptsize] {N} (l3)
			(n1) edge[tedge,swap] node[inner sep=.5pt,font=\scriptsize] {Y} (l1)
			(n1) edge[fedge] node[inner sep=.5pt,font=\scriptsize] {N} (l2)
			;
		\end{tikzpicture}\hspace{1cm}
		\begin{tikzpicture}[baseline=(n0.center),
			node/.style={draw,ellipse,inner sep=1pt,minimum height=.6cm,font=\small},
			leaf/.style={draw,rectangle,rounded corners,minimum height=.6cm,minimum width=.6cm,font=\small},
			tedge/.style={draw},
			fedge/.style={draw,dashed},
			auto,yscale=0.9
			]
			\node[node] at (0,0) (n0) {$s_1 \leq 1$};
			\node[leaf] at (-0.75,-1) (l1) {\emph{end1}};
			\node[leaf] at (0.75,-1) (l2) {\emph{end2}};
			\path[->]
			(n0) edge[tedge,swap] node[inner sep=.5pt,font=\scriptsize] {Y} (l1)
			(n0) edge[fedge] node[inner sep=.5pt,font=\scriptsize] {N} (l2)
			;
		\end{tikzpicture}\\[2pt]
	\caption{CSMA/CD ($N{=}2$, $K{=}6$) DTs produced by \dte{}. Left:  $\varepsilon = 10^{-6}$, right: $\varepsilon = 10^{-2}$.}
	\label{fig:csma-dts}
\end{figure}
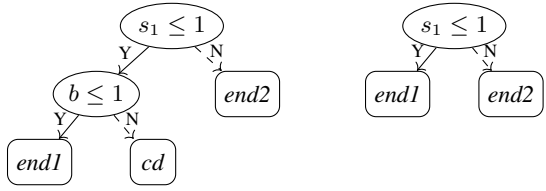

\para{CSMA time minimisation ($N{=}2$, $K{=}6$)~\cite{qcomp-benchmarkset}}
 models $N{=}2$ stations competing for a shared bus with exponential backoff up to $K{=}6$ collision rounds.
Each station checks whether the bus is free, and if so, sends its message. 
Otherwise, it waits for a random amount of time.
The property $\mathsf{time\_min}$ minimises the expected time until stations have delivered their messages.
\dtcontrol{} yields a \num{125}-node DT and \dtnest{} a $25$-node DT, while \dte{} already produces a \emph{$5$-node} DT at $\varepsilon = 10^{-6}$ (see \cref{fig:csma-dts}).
The DT encodes the following rule: if station~1 is in its initial or transmitting phase ($s_1 \leq 1$), let it finish if the bus is collision-free ($b \leq 1$, action \emph{end1}) or trigger collision detection (\emph{cd}) otherwise.
If station~1 is backing off or done, schedule station~2 (\emph{end2}).
At $\varepsilon = 10^{-2}$, the DT collapses to \emph{$3$ nodes}: the collision-handling branch disappears and the rule reduces to choosing \emph{end1} or \emph{end2}.

\para{Israeli-Jalfon self-stabilisation ($n{=}10$) \cite{qcomp-benchmarkset,israeli1990token}}
is a self-stabilising algorithm with $n$ processes.
In the beginning, several processes have tokens. %
The scheduler chooses which process to activate in each step.
However, it may only activate processes with tokens.
After completion, the process randomly moves its token to its left or right neighbour, who deletes the token, if it already has one.
The considered property states that there should finally be exactly one token.
$\dtcontrol{}$ produces a \num{1317}-node DT and \dtnest{} a \num{719}-node DT.
\dte{} discovers a structural property of the protocol: the probability of stabilisation is 1 for any scheduling strategy, so all scheduling choices are simultaneously optimal.
Therefore, \stepTwo{} produces an empty dataset.

\section{Conclusion and Future Work} \label{sec:conclusion}

\dte combines the maturity of \dtcontrol with the ability to trade the optimality of controllers for their compactness and explainability.
Our methods for exploiting model knowledge and allowed imprecision $\varepsilon$ enable our tool to produce controllers significantly smaller than the state of the art.
Future work includes coupling our tool with \texttt{SWITSS}~\cite{SWITSS} or \texttt{PET}~\cite{pet2} to use witnessing subsystems or cores~\cite{KM20-cores} for relevancy heuristics, as well as with \dtpaynt~\cite{dtpaynt} in the sense of \cite{DTNest} for advanced post-processing.
Towards explainable controllers for LTL, we want to extend the API of \stormpy{} and allow integration of semantic LTL-to-automaton translations as provided by e.g.\ \texttt{Owl} \cite{DBLP:conf/atva/KretinskyMS18}.
\medskip

\section*{Data Availability Statement}
Our tool is continually developed at \url{https://gitlab.com/live-lab/software/dtcontrol-epsilon}.
Further, the version of the tool used for the experiments, together with all benchmarks and scripts to reproduce our results, is included in our artefact.

\section*{Acknowledgements}
This research has received funding from the European Union under Grant Agreement No. 101171844, European Research Council (ERC) project Intelligence-Oriented Verification\&Controller Synthesis (InOVationCS) and from the ERC Starting Grant 101077178 (DEUCE).  Views and opinions expressed are however those of the author(s) only and do not necessarily reflect those of the European Union or of the granting authority. Neither the European Union nor the granting authority can be held responsible for them.
This research has also received funding from the MUNI Award in Science and Humanities MUNI/I/1757/2021 of the Grant Agency of Masaryk University.

\FloatBarrier
\newpage
\bibliographystyle{splncs04}
\bibliography{references}

\newpage
\appendix
\crefalias{section}{appendix} %
\crefalias{subsection}{appendix} %
\crefalias{subsubsection}{appendix}

\subsection{Linear Temporal Logic and Decision Trees} \label{app:ltl}

We extend our approach to also represent ($\varepsilon$-optimal) controllers for LTL queries.
The major challenge is that LTL requires controllers with finite memory, such as the states of an $\omega$-automaton for the LTL formula, see, e.g.,~\cite[Chapter 10.6.4]{BK08}.
In essence, instead of simply yielding an action in each state, finite state controllers have a set of memory states $M$, and produce for each state-memory pair $(s, m)$ an action $a$, i.e.\ the action chosen in a state depends on the current memory state.
Additionally, we output the underlying memory structure in form of an automaton.
Using tight coupling with \storm, \dte captures this information by extending the state with an additional variable that represents the current memory state.
Thus, the learned DT is a mapping from MDP states and memory states to actions.
We highlight that other tools like \dtpaynt or \dtnest could theoretically represent such a mapping, too; however, \storm currently does not offer the option to export the mapping in this form, whereas our tool explicitly constructs it.
Thus, \dte is the only tool with LTL support.

Still, due to technical limitations, the interpretability of DTs obtained this way is somewhat questionable.
In particular, the memory states we can obtain through \stormpy{} are opaque numbers without any further context.
For example, the state representing \enquote{waiting for a request} may have the number 1, while \enquote{request received, preparing for send} may have the number 2, and so on.
As such, a predicate of the form \enquote{$\text{automaton state} \leq 2$} has no connection to the LTL formula and thus hardly any interpretable meaning.
In particular, the numbering we obtain is arbitrary and may change from run to run.
To tackle this issue, a connection to \emph{semantic} automaton constructions (see \cite{DBLP:journals/jacm/EsparzaKS20}) implemented by tools such as \texttt{Rabinizer}~\cite{DBLP:conf/cav/KretinskyMSZ18} or \texttt{Owl}~\cite{DBLP:conf/atva/KretinskyMS18} would be required to obtain interpretable trees, which \storm{} does not (yet) support.

We successfully evaluated our tool on several hand-crafted instances, confirming that \dte{} can indeed construct DTs for strategies including memory.
We omit LTL properties from our experimental evaluation because there are no established true LTL benchmarks on MDPs (i.e.\ more than constrained reachability), see~\cite[Section~5]{DBLP:conf/toolympics/AndriushchenkoB23}.

\subsection{Extended Example Decision Tree} \label{app:dt_example_figures}

\cref{fig:dt_example_vanilla_reduced} shows the complete and reduced DT for the example in \cref{sec:overview_example}.

\begin{figure*}[t]
	\centering
	\begin{minipage}[t]{.49\textwidth}
		\centering\vspace{0pt}%
		\includegraphics[width=.9\textwidth,height=.8\textheight,keepaspectratio]{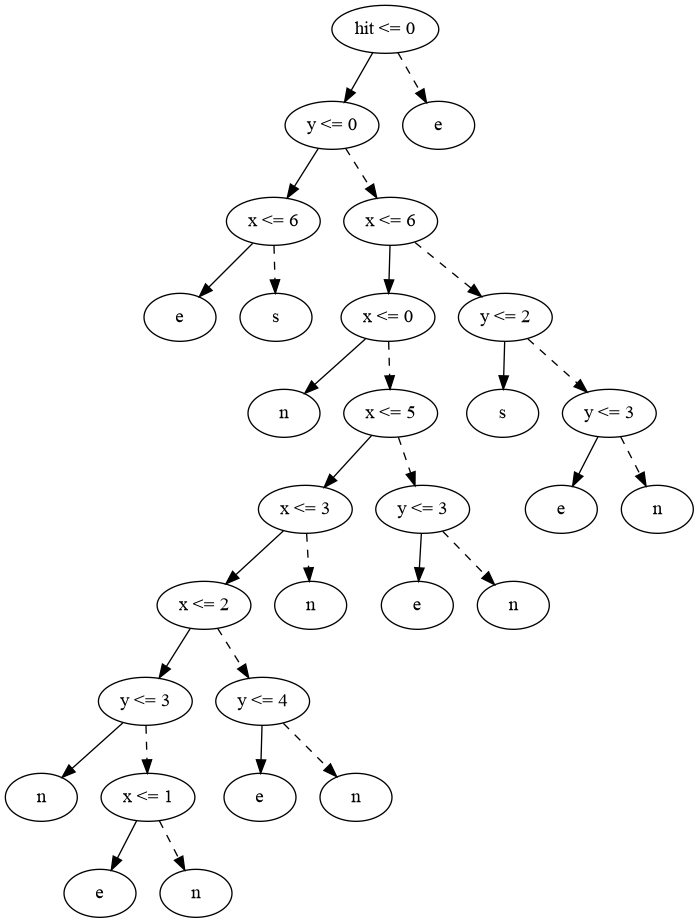}
	\end{minipage}%
	\begin{minipage}[t]{.49\textwidth}
		\centering\vspace{0pt}%
		\includegraphics[width=.9\textwidth,height=.8\textheight,keepaspectratio]{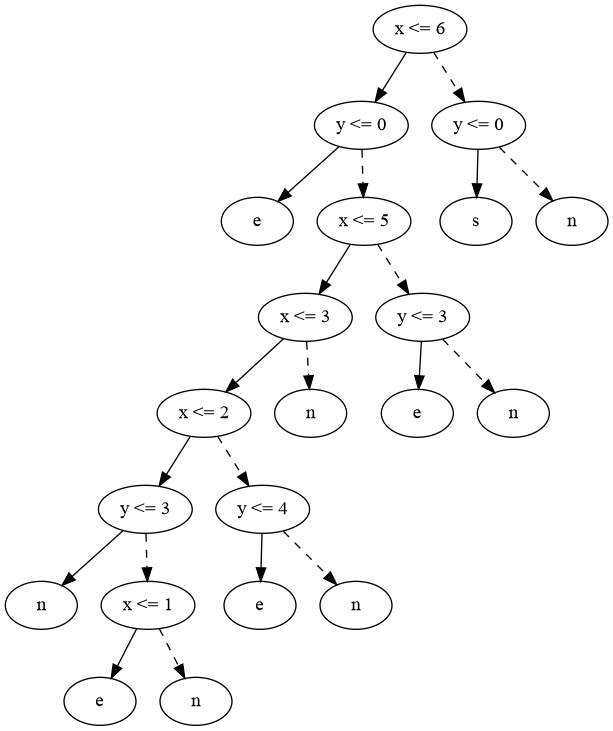}
	\end{minipage}
	\caption{
		Complete DT (left) and reduced DT (right) learnt for the optimal controller of the example in \cref{sec:overview_example}.
	} \label{fig:dt_example_vanilla_reduced}
\end{figure*}

\subsection{Incomparability of the Relevance Heuristics}\label{app:4-2-example-incomparable}

For \unreach($\mathcal{D}$), observe that whether a state is reachable or not under one particular controller is orthogonal to how many actions it has, whether it is dominated, whether its agency value is zero, or what the value of its actions are.
For example, suppose we have two copies of an MDP and, in the initial state, we choose between them. Then one of the two copies will always be unreachable, yet whichever states satisfy one condition also satisfy it in the other copy.
In particular, observe that all other methods do not depend on the controller or initial state.

For \allSafe{}, consider the following example.
Suppose in state $s$ the single action goes to a state $t$ or sink with probability 0.5 each.
In $t$, we have two choices: one leads to the goal and the other to the sink.
Here, $s$ is not dominated (no other state is reached with probability 1 under every controller), and its agency equals $0.5 > 0$.
However, as there is only one available action, all available actions are optimal, and the state can be removed.

When comparing \agency{} and \dominance{}, note that in a dominated state it is indeed relevant how the controller acts once the dominating (a.k.a. essential) state is reached.
In contrast, \agency{} prescribes that \emph{all} subsequent actions do not matter.
For example, consider state $s$ with a single action leading to state $t$, which in turn can choose between moving to the goal or a sink.
The state $s$ is dominated by $t$, but the agency of $s$ equals 1.
For the other way, consider a state with a single action that leads to a goal and sink each with probability 0.5.
This state has no agency but is not dominated.

We also note that \allSafe, \agency, and \dominance{} all imply \eqval{}. %
The former provides stronger guarantees, as these allow us to completely remove a state from the dataset instead of only adding the uniform action. %
However, \eqval{} indeed is a strictly weaker criterion, as exemplified in \cref{fig:permissive}:
The controller in which both states allow all actions is not optimal (since one can loop back and forth without reaching the goal), but the controller in which both play uniformly is optimal.

\subsection{Proof of Theorem \ref{thm:reduction-safe}}\label{app:4-2-reduction-safe}

For simplicity, we assume that, for all goal-based objectives (reachability, safety, and reachability cost), the goal/sink states are absorbing.
(Note that most model checkers, in particular \storm, perform this transformation automatically.) %
We formulate the proofs for reachability/safety, then outline the adaptation to reward-based objectives.
Before we prove the overall correctness, we introduce some useful lemmas.

\begin{lemma} \label{stm:convex_combination}
	Let $\MDP = (\states,\actions,\initstate,\trans)$ be an MDP with reachability or safety objective $\gStates$, $s \in \states$ a state, and $\controller$ an $\varepsilon$-optimal controller.
	Let $A^\varepsilon \subseteq \actions(s)$ be all $\varepsilon$-optimal actions, i.e.\ $|V(s, a) - V(s)| \leq \varepsilon$ for all $a\in A^\varepsilon$.
	Then, randomly choosing among $A^\varepsilon$ is $\varepsilon$-optimal, i.e.\ replacing the choice of $\controller$ in $s$ with randomising over $A^\varepsilon$ in any way yields an $\varepsilon$-optimal controller.
\end{lemma}
\begin{proof}
	For now, we focus on the case of maximising reachability (equivalent to minimising safety).
	Let $\controller'$ be the controller obtained from $\controller$ by randomising over $A^\varepsilon$ in $s$.
	Clearly, if $s$ is never visited, then the probability of reaching the goal is unchanged.
	Moreover, the probability of visiting $s$ for the first time is the same under $\controller$ and $\controller'$.
	Consider the set of paths that go through $s$ and split them into those that then do not reach $s$ again, and those that do.
	If they never visit $s$ again, due to the Markov property, the value they achieve is exactly the one achieved by following $\controller$ afterwards, i.e.\ $\varepsilon$-optimal.
	However, if they visit $s$ again, we can split the paths again.
	Thus, by induction, we can prove for all paths that visit $s$ finitely often that they achieve the same value as under $\controller$.
	Now, suppose that the set of paths that visit $s$ infinitely often is not a zero measure set.
	Then, the probability of going back to $s$ under $\controller'$ must be equal to $1$.
	The same holds true for $\controller$, since until we come back to $s$, $\controller$ and $\controller'$ behave the same.
	Consequently, the value of $\controller'$ is $0$, meaning that $\val(s) \leq \varepsilon$ and the choice in this state does not matter.
	(Note that, for this argument, it is important that $\controller$ randomises over \emph{all} actions in $A^\varepsilon$, or at least has $\controller'(s)(a) > 0$ for all $a$ with $\controller(s)(a) > 0$.)
	
	For maximising safety (and minimising reachability), the proof follows analogously.
	Notably, there we could even prove a slightly stronger statement, namely that $\controller$ could randomise over any subset of $A^\varepsilon$, as looping back to $s$ infinitely often is optimal. 
\end{proof}
To extend to reward based objectives, we simply need to exchange probabilities with rewards and observe that either we loop back to $s$ infinitely often (meaning that we get $0$ or $\infty$ reward), or, if not, we cannot obtain significantly more reward than $\controller$ in between, as otherwise $\controller$ would not be optimal.

Observe that the controller $\controller$ we started with only needs to be $\varepsilon$-optimal, and the controller we get by randomising is $\varepsilon$-optimal, too.
Hence, we can apply \cref{stm:convex_combination} repeatedly to obtain that randomising in several states maintains $\varepsilon$-optimality.

\begin{lemma} \label{stm:agency_dominance_correct}
	Let $\mathcal{D}$ be an optimal dataset for an MDP $(\states,\actions,\initstate,\trans)$, i.e.\ every controller consistent with $\mathcal{D}$ is optimal.
	Further, let $\mathcal{D}'$ be a dataset obtained by applying \dominance{} and \agency{}.
	Then $\lvert \val(\initstate) - \val^{\controller'}(\initstate) \rvert \leq \varepsilon$ for any controller $\controller'$ consistent with $\mathcal{D}'$.
\end{lemma}

\begin{proof}
	Let $\states'$ be all states in which $\mathcal{D}'$ differs from $\mathcal{D}$, i.e.\ where \dominance{} or \agency{} are applied.
	Moreover, let $\states_A$ be all states which have been removed through \agency{} but not \dominance{}.
	
	Further, let $\controller'$ be an arbitrary controller consistent with $\mathcal{D}'$; for it, we want to prove $\varepsilon$-optimality.
	To do so, we will relate it to $\controller$, a controller that (i) is consistent with $\mathcal{D}$ and hence optimal, and (ii) where for all states $s\in\states\setminus\states'$, we have $\controller(s) = \controller'(s)$.
	Such a controller $\pi$ exists, since for the unmodified states $s\in\states\setminus\states'$, we have $\mathcal{D}(s) = \mathcal{D}'(s)$, so every action $\controller'$ chooses is also available to~$\controller$.
	
	Now, we partition the set of paths according to the state of $\states_A$ they reach first:
	For all $s\in\states_A$, let 
	$\paths_s = \{ \rho \in \paths \mid \exists i.~\rho_i = s \land \forall j < i. \rho_j \notin \states_A\}$ be the paths for which $s$ is the first agency-modified state reached on it;
	and let $\paths' = \overline{\lozenge \states_A}$ denote the set of paths which never reach any state of $\states_A$.
	Since these sets $\paths_s$ and $\paths'$ partition the set of paths, we can write the value as:
	\begin{equation}\label{eq:split-paths}
		\begin{aligned}
			\val(\initstate) = \val^{\controller}(\initstate) = \prob_{\initstate}^{\controller}[\lozenge \gStates] \\
			\quad = \prob_{\initstate}^{\controller}[\lozenge \gStates \cap \paths'] + \sum_{s \in \states_A} \prob_{\initstate}^{\controller}[\lozenge \gStates \cap \paths_s].
		\end{aligned}
	\end{equation}
	
	Now, by Bayes' rule and linearity of reachability, we can rewrite:
	\begin{equation}\label{eq:product}
		\prob_{\initstate}^{\controller}[\lozenge \gStates \cap \paths_s] = \prob_{\initstate}^\controller[\paths_s] \cdot \prob_{s}^{\controller}[\lozenge \gStates] 
	\end{equation}
	
	Next, observe that $\prob_{\initstate}^\controller[\paths_s] = \prob_{\initstate}^{\controller'}[\paths_s]$ for all $s \in \states_A$:
	On the way from $\initstate$ to $s$, we either have states where $\controller$ and $\controller'$ agree, or dominated states which a.s.\ lead to their \enquote{exit} $s'$.
	As the states in $\states_A$ are not dominated by construction, the differences between $\controller$ and $\controller'$ on dominated states have no influence on the probability of reaching $s$.
	(Formally, we can apply the same reasoning, splitting the set of paths by the sequence of dominating states they visit on their way from $\initstate$ to $s$ and observing that the individual segments have the same probability.)
	By the analogous argument and the fact that no states from $\states_A$ are on the path, we also obtain $\prob_{\initstate}^{\controller}[\lozenge \gStates \cap \paths'] = \prob_{\initstate}^{\controller'}[\lozenge \gStates \cap \paths']$
	Together:
	\begin{equation}\label{eq:samePathsProb}
		\begin{aligned}
			\prob_{\initstate}^\controller[\paths_s] &= \prob_{\initstate}^{\controller'}[\paths_s], \\
			\prob_{\initstate}^{\controller}[\lozenge \gStates \cap \paths'] &= \prob_{\initstate}^{\controller'}[\lozenge \gStates \cap \paths']
		\end{aligned}
	\end{equation}
	
	Moreover, we can bound the difference between $\prob_{s}^{\controller}[\lozenge \gStates]$ and $\prob_{s}^{\controller'}[\lozenge \gStates]$ using the knowledge that $s\in\states_A$:
	When considering maximising reachability, we have
	$\prob_{s}^{\controller}[\lozenge \gStates] = \prob_{s}^{\max}[\lozenge \gStates]$ and $\prob_{s}^{\controller'}[\lozenge \gStates] \geq \prob_s^{\min}[\lozenge \gStates]$;
	dually, for minimizing reachability, we have $\prob_{s}^{\controller}[\lozenge \gStates] = \prob_{s}^{\min}[\lozenge \gStates]$ and $\prob_{s}^{\controller'}[\lozenge \gStates] \leq \prob_s^{\max}[\lozenge \gStates]$.
	Combining these with the fact that $s$ satisfies \agency{}, we obtain:
	\begin{equation}\label{eq:difference}
		\lvert \prob_{s}^{\controller}[\lozenge \gStates] - \prob_{s}^{\controller'}[\lozenge \gStates]\rvert \leq \varepsilon.
	\end{equation}

	Overall, we have 
	\begin{align*}
		\val(\initstate) & = 
		\prob_{\initstate}^{\controller}[\lozenge \gStates \cap \paths'] + {\sum}_{s \in \states_A} \prob_{\initstate}^{\controller}[\lozenge \gStates \cap \paths_s]
		\tag{By \cref{eq:split-paths}}\\
		&=
		\prob_{\initstate}^{\controller}[\lozenge \gStates \cap \paths'] + {\sum}_{s \in \states_A} \prob_{\initstate}^\controller[\paths_s] \cdot \prob_{s}^{\controller}[\lozenge \gStates]
		\tag{By \cref{eq:product}}\\
		&=
		\prob_{\initstate}^{\controller'}[\lozenge \gStates \cap \paths'] + {\sum}_{s \in \states_A} \prob_{\initstate}^{\controller'}[\paths_s] \cdot \prob_{s}^{\controller}[\lozenge \gStates]
		\tag{By \cref{eq:samePathsProb}}
	\end{align*}
	By analogously unfolding $\val^{\controller'}(\initstate)$ using \cref{eq:split-paths,eq:product} and applying \cref{eq:difference}, we obtain our goal:
	\begin{equation*}
		\lvert \val(\initstate) - \val^{\controller'}(\initstate) \rvert 
		\leq \varepsilon
	\end{equation*}

	To adapt the proof to reward-based objectives, we only need to swap out probabilities with expected rewards, e.g.\ the overall expected reward can be split into the expected reward obtained until reaching a modified state $s \in \states_A$ and the reward obtained afterwards.
	The key observation is that $\paths_s$ and $\paths'$ still partition the set of paths, hence by bounding the difference between $\controller$ and $\controller'$ in each $s \in \states_A$, we obtain an overall bound. %
\end{proof}

\begin{lemma}\label{lem:goalDist}
	\datasetPermissive{} yields an optimal dataset, i.e.\ every controller consistent with the dataset is optimal.
\end{lemma}
\begin{proof}
	For safety objectives (minimizing reachability), choosing any action with optimal value in every state yields an optimal controller; this follows from, e.g., \cite[Lemma~8]{EKKW22} by inserting the value function for the upper bound $\mathsf{U}$ (note that that lemma even proves the claim for stochastic games, which generalize MDPs).
	Thus, it is correct that \permissive\ picks all actions with optimal value.
	
	For reachability objectives, we additionally have to ensure progress, as exemplified in \cref{ex:progress}.
	For this, we first provide more detail on how the dataset is computed, i.e.\ on the \permissive\ heuristic.
	Essentially, we restrict to the MDP where only locally optimal actions are allowed and then perform a backwards breadth-first search from the goal states, thus computing the minimum number of steps required to reach a goal state with positive probability.
	This is exactly what the distance operator described in \cite[Section 3]{DBLP:conf/tacas/ChatterjeeQSWWZ25} does.
	Let $r$ be the resulting distance function.
	Then, for each state $s\in\states$, we add all actions $a$ to the dataset that (i) are optimal, i.e.\ $\val(s)=\val(s,a)$, and that (ii) reduce the distance to the goal, i.e.\ there exists an $s'\in\states$ such that $\trans(s,a,s')>0$ and $r(s')<r(s)$.
	For states with positive value, this set is non-empty, since they have a path to the target, and the distance function was computed based on only the optimal actions.	
	
	Let $\controller$ be a controller consistent with the dataset computed by \permissive.
	This means that, for every state, there is an optimal action that reduces the distance to the goal with positive probability.
	Then, we apply \cite[Proposition~5]{DBLP:conf/tacas/ChatterjeeQSWWZ25}, using the value $\val$ as the lower bound $x$, the distance function $r$, and the controller $\controller$.
	Condition 1) of the proposition is satisfied because the controller picks actions reducing the distance; Condition 2) is satisfied as the controller picks optimal actions; Condition 3) is trivially satisfied since $x=\val$.
	Thus, the proposition yields:
	$\prob_s^\controller[\lozenge\gStates]\geq \val(s)$.
	Since the value of a controller cannot be greater than the optimal value by definition, we conclude that for all states $s\in \states$, we have $\val^\controller(s) = \prob_s^\controller[\lozenge\gStates](s) = \val(s)$.
	As we picked an arbitrary controller consistent with the dataset computed by \permissive, we know that all controllers consistent with the dataset computed by \permissive\ are optimal.
	
\end{proof}

\reduceOK*
\begin{proof}
	We split the proof into two parts, primarily referring to the lemmas above and combining the insights.
	First, we show that the initial dataset is optimal for all three possible choices.
	Second, we show that the reductions yield a dataset that is $\varepsilon$-optimal.
	
	For the initial dataset, \datasetController{} is optimal by definition, as it uses the actions of an optimal controller.
	\datasetPermissive{} is optimal by \cref{lem:goalDist}.
	Finally, for \datasetCombined{}, optimality follows from the fact that we effectively obtain a subset of \datasetPermissive{}, which is hence consistent with \datasetPermissive{} and thus optimal by \cref{lem:goalDist}.
	
	Finally, for all initial datasets, we also apply \unreach\ and \eqval.
	For \unreach, observe that removing unreachable states does not change the value of any consistent controller, as the induced probability measures are equal.
	To maintain this optimality argument when further modifying the dataset, we need to ensure that whenever we remove a state, we do not allow reaching what has been removed by \unreach{} and thereby make its choices matter.
	This restriction is part of how we apply \eqval, making sure that we do not reach states that were removed by \unreach.
	The optimality of the uniform action added by \eqval\ follows from \cref{stm:convex_combination}.
	\medskip
	
	Now for the reduction methods, our previous lemmas show that applying them to a safe dataset again yields a safe dataset, as follows.
	Correctness of \agency{} and \dominance{} follows from \cref{stm:agency_dominance_correct}.
	Note that after applying these, the dataset is now only $\varepsilon$-optimal, since \agency{} utilises the allowed imprecision.
	For \allSafe, observe that if $\decisiontree(s) \in \actions(s) = \mathcal{D}(s)$, the action is already safe by assumption. If instead $\decisiontree(s)$ yields uniform choice (by virtue of $s$ not being defined by $\decisiontree$), correctness of this choice follows from \cref{stm:convex_combination}.
	In terms of \enquote{escaping} the reachable set, observe that as soon as we reach a state satisfying \agency{}, no subsequent actions matter.
	For \dominance{}, if $s$ is reachable, then $s'$ is always reached, no matter what we choose in $s$.
	Thus, $s'$ must be reachable as well.
	Similarly, for \allSafe, if $s$ is reachable and all of its actions are in the dataset, then none of its successors could have been removed by \unreach{}.
\end{proof}

\subsection{Additional Experimental Results}\label{app:exp-add}

In this section, we provide additional data on our experimental evaluation.
First, \cref{app:exp-competitors} details the tool versions used for our competitors.
Then, \cref{app:exp-highlights} summarises the main insights gained from all tables and figures in this appendix.
Finally, we provide the detailed results in the form of tables and plots:
\cref{app:subsec:state_statistics} concerns the dataset construction, \cref{app:subsec_tree_sizes} the DT size, and \cref{app:exp-times} the runtime.

Throughout this appendix, we employ scatter plots for visualisation.
Scatter plots compare two methods X and Y, and each point $(x, y)$ indicates that for one concrete instance, the considered metric for X and Y equals $x$ and $y$, respectively.
The plots contain additional lines for $\geq 512$, where we map all points that have a value of at least 512, \textit{TO} which indicates points that produced a timeout, and \textit{err} showing that an error occurred.
In addition, plots comparing size contain dotted lines for double and ten times the size of the other tool.
Lastly, in all plots, our tool performs better on points below the diagonal.

\subsubsection{Details on Tool Configurations}\label{app:exp-competitors}

For the competitors in our experimental evaluation, we used the following configurations:
\begin{itemize}
	\item \dtcontrol{}: We use the default configuration of \dtcontrol, using commit 6e0eb73c2d67a77611ad015c4fea9f7917c83da1 of our extension (\url{https://gitlab.com/live-lab/software/dtcontrol-epsilon}).
	\item \dtnest: 
	We tried using the artefact at DOI \url{10.5281/zenodo.15642001}. However, this is hard-coded to specific models and pre-computed datasets.
	After personal communication with the authors, we used the version of the tool available at \url{https://github.com/TheGreatfpmK/synthesis/tree/dtnest}, commit 9dbec52.
	\item \dtpaynt: We first tried the default configuration available here (\url{https://github.com/randriu/synthesis}). Surprised by the fact that the resulting DTs are not minimal, we contacted the authors, who recommended using commit 9dbec52 here (\url{https://github.com/TheGreatfpmK/synthesis/tree/dtnest}).
	This improved performance, but did not remove the problem that some DTs are not minimal.
	\item CAV15: Due to technical limitations, our reimplementation slightly differs from \cite{CAV15}:
	\stormpy does not provide information on whether a simulation fulfilled the property, and hence we cannot calculate importance as in~\cite[Sec.~4]{CAV15}.
	We use their similar performing \enquote{I$\forall\prob$} variant~\cite[Tbl.~2]{CAV15}.
	As \cite{CAV15} does not provide details for their pruning method, we allow the approach to use ours for fairness.
\end{itemize}

Furthermore, we explain our reasoning behind the selection of methods in the portfolio:
\begin{itemize}
	\item \datasetController+\wsimulation: Simulation benefits from fewer choices being available, leading to a smaller state space that needs to be traversed.
	\item \datasetPermissive+\wmessup: As this dataset contains a large portion of the state space, we expect that the DT only chooses wrongly in a few of the states.
	\item \datasetCombined+\wagency: \wagency{} estimates the error if a state is mapped incorrectly. This helps when many actions are available, but struggles for large datasets, as it does not capture the probability of reaching this state in the first place. %
\end{itemize}

\subsubsection{Summary of Additional Results}\label{app:exp-highlights}
The following sections provide deeper insights and more data about our experiments.
We start by comparing the different datasets in \cref{tab:agg-dataset-permissive,tab:agg-dataset-combined,tab:agg-dataset-controller} and find that, depending on the model, the impact of the heuristics varies widely.
Applying the combination of heuristics generally reduces the dataset more than any heuristic on its own.
Further, we see that \datasetPermissive{} is, as expected, mostly larger than the other two options.

We compare the DT sizes of the different tools in \cref{tab:tree-size} and \cref{fig:app_1e-6comp,fig:app_dtnest_size_comp}.
We see that \dtpaynt is often unable to provide results.
Especially for $\varepsilon=0$, these were often caused by excessive memory usage. 
We also conducted another run where we allowed them 128GB of memory instead of the usual 16GB, but in these cases, \dtpaynt unfortunately timed out.
\dtnest, as discussed, was not able to benefit from larger values for $\varepsilon$.
Furthermore, when comparing our portfolio with an overall runtime of 20 minutes with \dtnest, \cref{fig:app_dtnest_size_comp} shows that the portfolio generally performs better than \dtnest with few exceptions.
Often, \dte generates trees at most half the size of \dtnest.

Lastly, we analyse the runtime of our methods. 
\cref{fig:model-time} shows that there is no correlation between model size and the runtime of our methods.
\cref{tab:runtime-full} contains the tree building time (without model building and pruning) for completeness.
In \cref{fig:app_dtnest_time_comp} we compare the runtime of the portfolio approach with \dtnest. 
Surprisingly, despite the portfolio executing three different configurations, the runtime is comparable.
For completeness, we also include \cref{tab:runtime-dtnest_full}, which contains the runtime of \dtnest, \dtpaynt, and the portfolio approach, measured outside Docker.

\subsubsection{Sizes for Dataset Construction} \label{app:subsec:state_statistics}

We compare the number of states that can be reduced by the different relevance heuristics across the different dataset configurations.
Results are shown in \cref{tab:agg-dataset-permissive} for \datasetPermissive, \cref{tab:agg-dataset-combined} for \datasetCombined, and \cref{tab:agg-dataset-controller} and \datasetController.
Mostly, we found that all heuristics are sometimes able to reduce large portions of the dataset size and have little influence, showing that none outperforms the other. 
Applying all the heuristics together generally outperforms a single one. 
Furthermore, when focusing on \datasetPermissive, we see that it produces larger datasets than the other methods. 
As expected, \unreach{} does not provide as much of a benefit in this case as it does for \datasetCombined{} and \datasetController.
Since the DTs constructed from \datasetPermissive{} are comparable to those for other methods, we can observe the impact of determinisation.

Note that for several entries with $\varepsilon=0$, the table contains a \enquote{dash}.
This is because for these, the controller we obtained through a model checking query to \storm was not optimal, as confirmed through a second model checking query; it also differed by more than the float precision of $10^{-14}$ from the claimed value of the first query.
Even though we set \storm to sound solving, we were not able to avoid such cases. 
In these cases, \dte aborts.

Lastly, when analysing \agency, we see that $\varepsilon$ only has a minor influence on the number of states that can be removed.

\begin{table*}[t]
	\caption{
		Comparison of relevance heuristics for the \datasetPermissive{} dataset.
	}\label{tab:agg-dataset-permissive}%
	\centering%
	\setlength{\tabcolsep}{5pt}%
	\sisetup{table-format = 2.2\%, table-alignment-mode = format, table-auto-round}%
	\begin{adjustbox}{angle=270}%
		\resizebox*{.9\textheight}{!}{%
			\footnotesize
			\begin{tabular}{l|rrrrrrrrrr}
				Name & Model Size & Final $\varepsilon=0$ & Final $\varepsilon=10^{-6}$ & Final $\varepsilon=10^{-2}$ & \agency $\varepsilon=0$ & \agency $\varepsilon=10^{-6}$ & \agency$\varepsilon = 10^{-2}$ & \allSafe & \dominance & \unreach  \\
				\midrule
				\begin{filecontents*}{results/ablation-results/state_statistics.csv}
benchmark,combined-agency-reduction-0.000000,combined-agency-reduction-0.000001,combined-agency-reduction-0.010000,combined-all-actions-safe-states-0.000000,combined-all-actions-safe-states-0.000001,combined-all-actions-safe-states-0.010000,combined-dominated-states-reduction-0.000000,combined-dominated-states-reduction-0.000001,combined-dominated-states-reduction-0.010000,combined-equal-action-states-0.000000,combined-equal-action-states-0.000001,combined-equal-action-states-0.010000,combined-equal-action-states-time-0.000000,combined-equal-action-states-time-0.000001,combined-equal-action-states-time-0.010000,combined-essential-states-0.000000,combined-essential-states-0.000001,combined-essential-states-0.010000,combined-essential-states-time-0.000000,combined-essential-states-time-0.000001,combined-essential-states-time-0.010000,combined-final-dataset-size-0.000000,combined-final-dataset-size-0.000001,combined-final-dataset-size-0.010000,combined-initial-dataset-time-0.000000,combined-initial-dataset-time-0.000001,combined-initial-dataset-time-0.010000,combined-reachability-reduction-0.000000,combined-reachability-reduction-0.000001,combined-reachability-reduction-0.010000,combined-single-action-states-0.000000,combined-single-action-states-0.000001,combined-single-action-states-0.010000,controller-agency-reduction-0.000000,controller-agency-reduction-0.000001,controller-agency-reduction-0.010000,controller-all-actions-safe-states-0.000000,controller-all-actions-safe-states-0.000001,controller-all-actions-safe-states-0.010000,controller-dominated-states-reduction-0.000000,controller-dominated-states-reduction-0.000001,controller-dominated-states-reduction-0.010000,controller-equal-action-states-0.000000,controller-equal-action-states-0.000001,controller-equal-action-states-0.010000,controller-equal-action-states-time-0.000000,controller-equal-action-states-time-0.000001,controller-equal-action-states-time-0.010000,controller-essential-states-0.000000,controller-essential-states-0.000001,controller-essential-states-0.010000,controller-essential-states-time-0.000000,controller-essential-states-time-0.000001,controller-essential-states-time-0.010000,controller-final-dataset-size-0.000000,controller-final-dataset-size-0.000001,controller-final-dataset-size-0.010000,controller-initial-dataset-time-0.000000,controller-initial-dataset-time-0.000001,controller-initial-dataset-time-0.010000,controller-reachability-reduction-0.000000,controller-reachability-reduction-0.000001,controller-reachability-reduction-0.010000,controller-single-action-states-0.000000,controller-single-action-states-0.000001,controller-single-action-states-0.010000,dataset-creation-time-0.000000,dataset-creation-time-0.000001,dataset-creation-time-0.010000,initial-scheduler-difference-0.000000,initial-scheduler-difference-0.000001,initial-scheduler-difference-0.010000,model-size-0.000000,model-size-0.000001,model-size-0.010000,outcome-0.000000,outcome-0.000001,outcome-0.010000,outcome-details-0.000000,permissive-agency-reduction-0.000000,permissive-agency-reduction-0.000001,permissive-agency-reduction-0.010000,permissive-all-actions-safe-states-0.000000,permissive-all-actions-safe-states-0.000001,permissive-all-actions-safe-states-0.010000,permissive-dominated-states-reduction-0.000000,permissive-dominated-states-reduction-0.000001,permissive-dominated-states-reduction-0.010000,permissive-equal-action-states-0.000000,permissive-equal-action-states-0.000001,permissive-equal-action-states-0.010000,permissive-equal-action-states-time-0.000000,permissive-equal-action-states-time-0.000001,permissive-equal-action-states-time-0.010000,permissive-essential-states-0.000000,permissive-essential-states-0.000001,permissive-essential-states-0.010000,permissive-essential-states-time-0.000000,permissive-essential-states-time-0.000001,permissive-essential-states-time-0.010000,permissive-final-dataset-size-0.000000,permissive-final-dataset-size-0.000001,permissive-final-dataset-size-0.010000,permissive-initial-dataset-time-0.000000,permissive-initial-dataset-time-0.000001,permissive-initial-dataset-time-0.010000,permissive-reachability-reduction-0.000000,permissive-reachability-reduction-0.000001,permissive-reachability-reduction-0.010000,permissive-single-action-states-0.000000,permissive-single-action-states-0.000001,permissive-single-action-states-0.010000
consensus4-4-c2,-,23430,23436,-,30706,30706,-,4442,4442,-,-,-,-,-,-,-,38694,38694,-,1,1,-,335,335,-,1,1,-,40199,40199,-,4480,4480,-,23430,23436,-,4480,4480,-,4442,4442,-,-,-,-,-,-,-,38694,38694,-,1,1,-,418,418,-,0,0,-,40199,40199,-,4480,4480,-,6,6,-,0,0,43136,43136,43136,error,success,success,precision,-,23430,23436,-,30706,30706,-,4442,4442,-,-,-,-,-,-,-,38694,38694,-,1,1,-,2832,2832,-,2,1,-,16771,16771,-,4480,4480
consensus4-4-disagree,-,12026,12026,-,24824,24824,-,4442,4442,-,3632,3632,-,1,1,-,38694,38694,-,1,1,-,5746,5746,-,2,2,-,31291,31291,-,4480,4480,-,12026,12026,-,4480,4480,-,4442,4442,-,3642,3642,-,1,1,-,38694,38694,-,1,1,-,6966,6966,-,0,0,-,31022,31022,-,4480,4480,-,11,12,-,0,0,43136,43136,43136,error,success,success,precision,-,12026,12026,-,24824,24824,-,4442,4442,-,14613,14613,-,1,1,-,38694,38694,-,1,1,-,14296,14296,-,3,3,-,10813,10813,-,4480,4480
consensus4-4-steps-max,-,5810,5810,-,14620,14620,-,3600,3600,-,-,-,-,-,-,-,39536,39536,-,1,1,-,1169,1169,-,1,1,-,41809,41809,-,4480,4480,-,5810,5810,-,4480,4480,-,3600,3600,-,-,-,-,-,-,-,39536,39536,-,1,1,-,1309,1309,-,0,0,-,41809,41809,-,4480,4480,-,6,6,-,0,0,43136,43136,43136,error,success,success,precision,-,5810,5810,-,14620,14620,-,3600,3600,-,-,-,-,-,-,-,39536,39536,-,1,1,-,3924,3924,-,1,1,-,37450,37450,-,4480,4480
consensus4-4-steps-min,-,5810,5810,-,21236,21236,-,3600,3600,-,88,88,-,1,1,-,39536,39536,-,1,1,-,1692,1692,-,2,2,-,40883,40883,-,4480,4480,-,5810,5810,-,4480,4480,-,3600,3600,-,88,88,-,1,1,-,39536,39536,-,1,1,-,2243,2243,-,0,0,-,40883,40883,-,4480,4480,-,10,10,-,0,0,43136,43136,43136,error,success,success,precision,-,5810,5810,-,21236,21236,-,3600,3600,-,1995,1995,-,1,1,-,39536,39536,-,1,1,-,4976,4976,-,2,2,-,36153,36153,-,4480,4480
csma2-6-all-before-max,66718,66718,66718,66681,66681,66681,66110,66110,66110,0,0,0,0,0,0,608,608,608,1,1,1,0,0,0,1,1,1,2248,2248,2248,66648,66648,66648,66718,66718,66718,66648,66648,66648,66110,66110,66110,0,0,0,0,0,0,608,608,608,1,1,1,0,0,0,1,1,1,485,485,485,66648,66648,66648,8,8,8,0,0,0,66718,66718,66718,success,success,success,-,66718,66718,66718,66681,66681,66681,66110,66110,66110,0,0,0,0,0,0,608,608,608,1,1,1,0,0,0,1,1,1,2115,2115,2115,66648,66648,66648
csma2-6-time-max,49109,49109,49109,66681,66681,66681,66063,66063,66063,-,-,-,-,-,-,655,655,655,1,1,1,37,37,37,1,1,1,218,218,218,66648,66648,66648,49109,49109,49109,66648,66648,66648,66063,66063,66063,-,-,-,-,-,-,655,655,655,1,1,1,43,43,43,1,1,1,218,218,218,66648,66648,66648,6,6,6,0,0,0,66718,66718,66718,success,success,success,-,49109,49109,49109,66681,66681,66681,66063,66063,66063,-,-,-,-,-,-,655,655,655,1,1,1,37,37,37,1,1,1,65,65,65,66648,66648,66648
csma2-6-time-min,49109,49109,49109,66681,66681,66681,66063,66063,66063,0,0,0,0,0,0,655,655,655,1,1,1,37,37,37,1,1,1,2248,2248,2248,66648,66648,66648,49109,49109,49109,66648,66648,66648,66063,66063,66063,0,0,0,0,0,0,655,655,655,1,1,1,43,43,43,1,1,1,2248,2248,2248,66648,66648,66648,8,8,8,0,0,0,66718,66718,66718,success,success,success,-,49109,49109,49109,66681,66681,66681,66063,66063,66063,0,0,0,0,0,0,655,655,655,1,1,1,37,37,37,1,1,1,2115,2115,2115,66648,66648,66648
eajs3-150-exputil,31736,31736,32194,123967,123967,123967,63371,63371,63371,-,-,-,-,-,-,79784,79784,79784,2,2,2,2119,2119,2095,1,1,1,129961,129961,129961,119406,119406,119406,31736,31736,32194,119406,119406,119406,63371,63371,63371,-,-,-,-,-,-,79784,79784,79784,2,2,2,2201,2201,2177,0,0,0,129961,129961,129961,119406,119406,119406,13,13,13,0,0,0,143155,143155,143155,success,success,success,-,31736,31736,32194,123967,123967,123967,63371,63371,63371,-,-,-,-,-,-,79784,79784,79784,2,2,2,4936,4936,4888,1,1,1,107653,107653,107653,119406,119406,119406
firewire-dl-3-400-deadline,44093,44093,44093,67435,67435,67435,62789,62789,62789,-,-,-,-,-,-,6894,6894,6894,1,1,1,126,126,126,0,0,0,65239,65239,65239,63361,63361,63361,44093,44093,44093,63361,63361,63361,62789,62789,62789,-,-,-,-,-,-,6894,6894,6894,1,1,1,150,150,150,0,0,0,65239,65239,65239,63361,63361,63361,5,5,5,0,0,0,69683,69683,69683,success,success,success,-,44093,44093,44093,67435,67435,67435,62789,62789,62789,-,-,-,-,-,-,6894,6894,6894,1,1,1,2020,2020,2020,1,1,1,6580,6580,6580,63361,63361,63361
firewire-dl-3-600-deadline,102720,102720,102720,162246,162246,162246,150988,150988,150988,-,-,-,-,-,-,17423,17423,17423,2,2,2,331,331,331,1,1,1,159688,159688,159688,152288,152288,152288,102720,102720,102720,152288,152288,152288,150988,150988,150988,-,-,-,-,-,-,17423,17423,17423,2,2,2,391,391,391,0,0,0,159688,159688,159688,152288,152288,152288,12,12,12,0,0,0,168411,168411,168411,success,success,success,-,102720,102720,102720,162246,162246,162246,150988,150988,150988,-,-,-,-,-,-,17423,17423,17423,2,2,2,3189,3189,3189,1,1,1,80249,80249,80249,152288,152288,152288
firewire-dl-36-200-deadline,45214,45214,45214,61417,61417,61417,43374,43374,43374,-,-,-,-,-,-,24682,24682,24682,1,1,1,176,176,176,1,1,1,66805,66805,66805,43614,43614,43614,45214,45214,45214,43614,43614,43614,43374,43374,43374,-,-,-,-,-,-,24682,24682,24682,1,1,1,250,250,250,0,0,0,66805,66805,66805,43614,43614,43614,6,6,6,0,0,0,68056,68056,68056,success,success,success,-,45214,45214,45214,61417,61417,61417,43374,43374,43374,-,-,-,-,-,-,24682,24682,24682,1,1,1,1966,1966,1966,1,1,1,48941,48941,48941,43614,43614,43614
firewire-dl-36-400-deadline,238854,238854,238854,451779,451779,451779,290496,290496,290496,-,-,-,-,-,-,240469,240469,240469,8,9,9,1387,1387,1387,6,6,6,521146,521146,521146,293136,293136,293136,238854,238854,238854,293136,293136,293136,290496,290496,290496,-,-,-,-,-,-,240469,240469,240469,8,8,8,2008,2008,2008,1,1,1,521146,521146,521146,293136,293136,293136,51,53,53,0,0,0,530965,530965,530965,success,success,success,-,238854,238854,238854,451779,451779,451779,290496,290496,290496,-,-,-,-,-,-,240469,240469,240469,8,9,9,48712,48712,48712,9,10,9,165516,165516,165516,293136,293136,293136
firewire-dl-36-600-deadline,45214,45214,45214,61417,61417,61417,43374,43374,43374,-,-,-,-,-,-,24682,24682,24682,1,1,1,176,176,176,1,1,1,66805,66805,66805,43614,43614,43614,45214,45214,45214,43614,43614,43614,43374,43374,43374,-,-,-,-,-,-,24682,24682,24682,1,1,1,250,250,250,0,0,0,66805,66805,66805,43614,43614,43614,6,6,6,0,0,0,68056,68056,68056,success,success,success,-,45214,45214,45214,61417,61417,61417,43374,43374,43374,-,-,-,-,-,-,24682,24682,24682,1,1,1,1966,1966,1966,1,1,1,48941,48941,48941,43614,43614,43614
firewire-false-36-400-timemax,-,4,4,-,164057,136233,-,27353,27353,-,-,-,-,-,-,-,184915,184915,-,4,4,-,216,216,-,5,5,-,211393,211393,-,26161,26161,-,4,4,-,26161,26161,-,27353,27353,-,-,-,-,-,-,-,184915,184915,-,5,5,-,512,512,-,0,0,-,211393,211393,-,26161,26161,-,29,28,-,0,0,212268,212268,212268,error,success,success,precision,-,4,4,-,164057,136233,-,27353,27353,-,-,-,-,-,-,-,184915,184915,-,5,4,-,431,431,-,5,5,-,209746,209746,-,26161,26161
firewire-false-36-400-timemin,4,4,4,199995,206651,199995,27353,27353,27353,0,0,0,4,4,4,184915,184915,184915,5,5,5,9,9,9,9,9,9,211691,211691,211691,26161,26161,26161,4,4,4,26161,26161,26161,27353,27353,27353,0,0,0,4,4,4,184915,184915,184915,5,5,5,232,232,232,0,0,0,211691,211691,211691,26161,26161,26161,50,51,50,0,0,0,212268,212268,212268,success,success,success,-,4,4,4,199995,206651,199995,27353,27353,27353,161,161,161,4,4,4,184915,184915,184915,5,5,5,12,12,12,9,9,9,210488,210488,210488,26161,26161,26161
ij10,1023,1023,1023,103,103,103,10,10,10,54,54,54,0,0,0,1013,1013,1013,0,0,0,0,0,0,0,0,0,568,568,568,10,10,10,1023,1023,1023,10,10,10,10,10,10,137,137,137,0,0,0,1013,1013,1013,0,0,0,0,0,0,0,0,0,452,452,452,10,10,10,1,1,1,0,0,0,1023,1023,1023,success,success,success,-,1023,1023,1023,103,103,103,10,10,10,883,883,883,0,0,0,1013,1013,1013,0,0,0,0,0,0,0,0,0,20,20,20,10,10,10
pacman5,235,235,235,235,235,235,195,195,195,-,-,-,-,-,-,40,40,40,0,0,0,0,0,0,0,0,0,180,180,180,199,199,199,235,235,235,199,199,199,195,195,195,-,-,-,-,-,-,40,40,40,0,0,0,0,0,0,0,0,0,180,180,180,199,199,199,0,0,0,0,0,0,235,235,235,success,success,success,-,235,235,235,235,235,235,195,195,195,-,-,-,-,-,-,40,40,40,0,0,0,0,0,0,0,0,0,0,0,0,199,199,199
pnueli-zuck-5-live,4096,4096,4096,6094,6094,6094,4096,4096,4096,0,0,0,13,13,13,303427,303427,303427,12,12,12,4,4,4,35,35,34,307518,307518,307518,4096,4096,4096,4096,4096,4096,4096,4096,4096,4096,4096,4096,0,0,0,13,13,13,303427,303427,303427,12,12,12,2395,2395,2395,1,1,1,305127,305127,305127,4096,4096,4096,150,152,149,0,0,0,307523,307523,307523,success,success,success,-,4096,4096,4096,6094,6094,6094,4096,4096,4096,0,0,0,13,13,13,303427,303427,303427,12,12,12,4,4,4,34,35,34,307518,307518,307518,4096,4096,4096
pnueli3-live,64,64,64,202,202,202,64,64,64,0,0,0,0,0,0,1885,1885,1885,0,0,0,4,4,4,0,0,0,1944,1944,1944,64,64,64,64,64,64,64,64,64,64,64,64,0,0,0,0,0,0,1885,1885,1885,0,0,0,71,71,71,0,0,0,1877,1877,1877,64,64,64,1,1,1,0,0,0,1949,1949,1949,success,success,success,-,64,64,64,202,202,202,64,64,64,0,0,0,0,0,0,1885,1885,1885,0,0,0,4,4,4,0,0,0,1944,1944,1944,64,64,64
rabin3-live,1088,1088,1088,794,794,794,384,384,384,0,0,0,0,0,0,704,704,704,0,0,0,0,0,0,0,0,0,1077,1077,1077,741,741,741,1088,1088,1088,741,741,741,384,384,384,0,0,0,0,0,0,704,704,704,0,0,0,0,0,0,0,0,0,561,561,561,741,741,741,0,0,0,0,0,0,1088,1088,1088,success,success,success,-,1088,1088,1088,794,794,794,384,384,384,1,1,1,0,0,0,704,704,704,0,0,0,0,0,0,0,0,0,1057,1057,1057,741,741,741
wlan3-0-cost-max,7,7,7,89857,89869,89857,83972,83972,83972,-,-,-,-,-,-,12330,12330,12330,2,2,2,4897,4890,4897,1,1,1,30239,30264,30239,68957,68957,68957,7,7,7,68957,68957,68957,83972,83972,83972,-,-,-,-,-,-,12330,12330,12330,2,2,2,5795,5807,5795,1,1,1,30239,30264,30239,68957,68957,68957,11,11,11,0,0,0,96302,96302,96302,success,success,success,-,7,7,7,89857,89869,89857,83972,83972,83972,-,-,-,-,-,-,12330,12330,12330,2,2,2,5545,5541,5545,2,2,2,9155,9051,9155,68957,68957,68957
wlan3-0-cost-min,7,7,7,90407,90761,90407,83972,83972,83972,0,0,0,1,1,1,12330,12330,12330,2,2,2,11,11,11,2,2,2,96238,96238,96238,68957,68957,68957,7,7,7,68957,68957,68957,83972,83972,83972,0,0,0,1,1,1,12330,12330,12330,2,2,2,12,12,12,0,0,0,96236,96236,96236,68957,68957,68957,14,14,14,0,0,0,96302,96302,96302,success,success,success,-,7,7,7,90407,90761,90407,83972,83972,83972,0,0,0,1,1,1,12330,12330,12330,2,2,2,21,21,21,2,2,2,96178,96178,96178,68957,68957,68957
wlan4-0-cost-max,9,9,9,331807,331961,331807,311003,311003,311003,-,-,-,-,-,-,33997,33997,33997,6,6,6,10000,10000,10000,5,5,5,102211,102211,102211,249897,249897,249897,9,9,9,249897,249897,249897,311003,311003,311003,-,-,-,-,-,-,33997,33997,33997,6,6,6,12082,12082,12082,2,3,3,102211,102211,102211,249897,249897,249897,39,40,40,0,0,0,345000,345000,345000,success,success,success,-,9,9,9,331807,331961,331807,311003,311003,311003,-,-,-,-,-,-,33997,33997,33997,6,6,6,10791,10687,10791,5,6,6,25717,25717,25717,249897,249897,249897
wlan4-0-cost-min,9,9,9,333169,333853,333169,311003,311003,311003,0,0,0,2,2,2,33997,33997,33997,6,6,6,11,11,11,8,8,8,344936,344936,344936,249897,249897,249897,9,9,9,249897,249897,249897,311003,311003,311003,0,0,0,2,2,2,33997,33997,33997,6,6,6,12,12,12,1,1,1,344934,344934,344934,249897,249897,249897,49,50,50,0,0,0,345000,345000,345000,success,success,success,-,9,9,9,333169,333853,333169,311003,311003,311003,0,0,0,2,2,2,33997,33997,33997,6,6,6,21,21,21,8,8,8,344876,344876,344876,249897,249897,249897
wlandl0-80-deadline,117876,117876,117876,186506,186506,186506,140065,140065,140065,-,-,-,-,-,-,49638,49638,49638,3,3,3,478,478,478,2,2,2,138867,138867,138867,126234,126234,126234,117876,117876,117876,126234,126234,126234,140065,140065,140065,-,-,-,-,-,-,49638,49638,49638,3,3,3,3517,3517,3517,1,1,1,138867,138867,138867,126234,126234,126234,19,19,19,0,0,0,189703,189703,189703,success,success,success,-,117876,117876,117876,186506,186506,186506,140065,140065,140065,-,-,-,-,-,-,49638,49638,49638,3,3,3,727,727,727,3,3,3,62571,62571,62571,126234,126234,126234
wlandl1-80-deadline,313182,313182,313182,446782,446782,446782,354173,354173,354173,-,-,-,-,-,-,96454,96454,96454,7,7,7,478,478,478,5,5,5,293641,293641,293641,309376,309376,309376,313182,313182,313182,309376,309376,309376,354173,354173,354173,-,-,-,-,-,-,96454,96454,96454,7,7,7,4260,4260,4260,2,2,2,293641,293641,293641,309376,309376,309376,45,46,45,0,0,0,450627,450627,450627,success,success,success,-,313182,313182,313182,446782,446782,446782,354173,354173,354173,-,-,-,-,-,-,96454,96454,96454,7,7,7,727,727,727,7,7,7,142272,142272,142272,309376,309376,309376
zeroconf-1000-2-false-correctmax,-,29186,29189,-,56356,56351,-,22875,22875,-,1,1,-,1,1,-,65983,65983,-,1,2,-,328,328,-,3,3,-,87540,87540,-,24338,24338,-,29186,29189,-,24338,24338,-,22875,22875,-,1,1,-,1,1,-,65983,65983,-,2,1,-,464,464,-,0,0,-,87519,87519,-,24338,24338,-,17,18,-,0,0,88858,88858,88858,error,success,success,precision,-,29186,29189,-,56356,56351,-,22875,22875,-,31,31,-,1,1,-,65983,65983,-,2,2,-,732,732,-,3,3,-,85992,85992,-,24338,24338
zeroconf-1000-2-false-correctmin,-,29186,29189,-,49921,48515,-,22875,22875,-,-,-,-,-,-,-,65983,65983,-,2,1,-,3,4,-,2,2,-,88397,88411,-,24338,24338,-,29186,29189,-,24338,24338,-,22875,22875,-,-,-,-,-,-,-,65983,65983,-,2,1,-,32,27,-,0,0,-,88397,88411,-,24338,24338,-,11,11,-,0,0,88858,88858,88858,error,success,success,precision,-,29186,29189,-,49921,48515,-,22875,22875,-,-,-,-,-,-,-,65983,65983,-,2,1,-,308,331,-,2,2,-,85158,86370,-,24338,24338
zeroconf-1000-4-false-correctmax,-,93317,93323,-,188525,188509,-,82336,82336,-,73,73,-,4,4,-,224249,224249,-,6,5,-,2108,2108,-,11,11,-,300571,300571,-,82170,82170,-,93317,93323,-,82170,82170,-,82336,82336,-,73,73,-,4,4,-,224249,224249,-,5,5,-,2844,2844,-,1,1,-,300456,300456,-,82170,82170,-,62,61,-,0,0,306585,306585,306585,error,success,success,precision,-,93317,93323,-,188525,188509,-,82336,82336,-,1093,1093,-,4,4,-,224249,224249,-,5,5,-,3959,3959,-,11,11,-,292348,292348,-,82170,82170
zeroconf-1000-4-false-correctmin,-,93317,93323,-,199768,182670,-,82336,82336,-,-,-,-,-,-,-,224249,224249,-,6,6,-,3,1,-,6,6,-,305840,305876,-,82170,82170,-,93317,93323,-,82170,82170,-,82336,82336,-,-,-,-,-,-,-,224249,224249,-,5,5,-,50,47,-,1,1,-,305840,305876,-,82170,82170,-,39,38,-,0,0,306585,306585,306585,error,success,success,precision,-,93317,93323,-,199768,182670,-,82336,82336,-,-,-,-,-,-,-,224249,224249,-,6,5,-,3614,1562,-,6,6,-,255291,286602,-,82170,82170
zeroconf-20-2-false-correctmax,29186,29186,29192,56362,56365,56362,22875,22875,22875,1,1,1,1,1,1,65983,65983,65983,2,1,2,355,355,355,3,3,3,87479,87479,87479,24338,24338,24338,29186,29186,29192,24338,24338,24338,22875,22875,22875,1,1,1,1,1,1,65983,65983,65983,1,2,1,496,496,496,0,0,0,87454,87454,87454,24338,24338,24338,17,18,17,0,0,0,88858,88858,88858,success,success,success,-,29186,29186,29192,56362,56365,56362,22875,22875,22875,31,31,31,1,1,1,65983,65983,65983,2,2,2,732,732,732,3,3,3,85992,85992,85992,24338,24338,24338
zeroconf-20-2-false-correctmin,29186,29186,29192,50507,53349,50507,22875,22875,22875,-,-,-,-,-,-,65983,65983,65983,1,2,1,1,3,1,2,2,2,88413,88399,88413,24338,24338,24338,29186,29186,29192,24338,24338,24338,22875,22875,22875,-,-,-,-,-,-,65983,65983,65983,1,2,1,25,28,25,0,0,0,88413,88399,88413,24338,24338,24338,11,11,11,0,0,0,88858,88858,88858,success,success,success,-,29186,29186,29192,50507,53349,50507,22875,22875,22875,-,-,-,-,-,-,65983,65983,65983,1,2,1,354,707,354,2,2,2,85251,79983,85251,24338,24338,24338
zeroconf-20-4-false-correctmax,93317,93317,93329,188752,188753,188752,82336,82336,82336,36,36,36,4,4,4,224249,224249,224249,5,6,5,1437,1437,1437,11,11,11,302164,302164,302164,82170,82170,82170,93317,93317,93329,82170,82170,82170,82336,82336,82336,36,36,36,4,4,4,224249,224249,224249,5,5,5,1910,1910,1910,1,1,1,302062,302062,302062,82170,82170,82170,60,63,61,0,0,0,306585,306585,306585,success,success,success,-,93317,93317,93329,188752,188753,188752,82336,82336,82336,513,513,513,4,4,4,224249,224249,224249,5,6,5,2717,2717,2717,11,11,11,296553,296553,296553,82170,82170,82170
zeroconf-20-4-false-correctmin,-,93317,93329,-,242165,217779,-,82336,82336,-,-,-,-,-,-,-,224249,224249,-,6,5,-,2,1,-,6,5,-,305842,305877,-,82170,82170,-,93317,93329,-,82170,82170,-,82336,82336,-,-,-,-,-,-,-,224249,224249,-,5,5,-,47,46,-,1,1,-,305842,305877,-,82170,82170,-,41,39,-,0,0,306585,306585,306585,error,success,success,precision,-,93317,93329,-,242165,217779,-,82336,82336,-,-,-,-,-,-,-,224249,224249,-,6,5,-,34433,39296,-,8,8,-,82695,120081,-,82170,82170
zeroconfdl-1000-1-false-20-deadlinemax,27405,27405,27600,34925,34925,34925,18040,18040,18040,17,17,17,1,1,1,25907,25907,25907,1,1,1,359,359,358,1,2,1,39740,39740,39740,17986,17986,17986,27405,27405,27600,17986,17986,17986,18040,18040,18040,22,22,22,1,1,1,25907,25907,25907,1,1,1,614,614,609,0,0,0,37889,37889,37889,17986,17986,17986,8,8,8,0,0,0,43947,43947,43947,success,success,success,-,27405,27405,27600,34925,34925,34925,18040,18040,18040,641,641,641,1,1,1,25907,25907,25907,1,1,1,749,749,748,1,2,1,31648,31648,31648,17986,17986,17986
zeroconfdl-1000-1-false-20-deadlinemin,27405,27405,27600,38207,38207,38207,18040,18040,18040,-,-,-,-,-,-,25907,25907,25907,1,1,1,554,554,554,1,1,1,39038,39038,39038,17986,17986,17986,27405,27405,27600,17986,17986,17986,18040,18040,18040,-,-,-,-,-,-,25907,25907,25907,1,1,1,614,614,614,0,0,0,39038,39038,39038,17986,17986,17986,5,5,5,0,0,0,43947,43947,43947,success,success,success,-,27405,27405,27600,38207,38207,38207,18040,18040,18040,-,-,-,-,-,-,25907,25907,25907,1,1,1,838,838,838,1,1,1,35341,35341,35341,17986,17986,17986
zeroconfdl-1000-1-false-50-deadlinemax,155738,155738,156341,250073,250073,250073,101663,101663,101663,18,18,18,5,5,5,279018,279018,279018,7,7,7,1890,1890,1889,13,14,13,371991,371991,371991,116170,116170,116170,155738,155738,156341,116170,116170,116170,101663,101663,101663,22,22,22,5,5,5,279018,279018,279018,7,7,7,3134,3134,3129,1,1,1,369543,369543,369543,116170,116170,116170,75,78,75,0,0,0,380681,380681,380681,success,success,success,-,155738,155738,156341,250073,250073,250073,101663,101663,101663,701,701,701,5,5,5,279018,279018,279018,7,7,7,4229,4229,4228,13,14,13,357797,357797,357797,116170,116170,116170
zeroconfdl-1000-1-false-50-deadlinemin,155738,155738,156341,326528,326528,326528,101663,101663,101663,-,-,-,-,-,-,279018,279018,279018,7,7,7,991,991,991,7,7,7,369439,369439,369439,116170,116170,116170,155738,155738,156341,116170,116170,116170,101663,101663,101663,-,-,-,-,-,-,279018,279018,279018,7,7,7,1731,1731,1731,1,1,1,369439,369439,369439,116170,116170,116170,48,49,48,0,0,0,380681,380681,380681,success,success,success,-,155738,155738,156341,326528,326528,326528,101663,101663,101663,-,-,-,-,-,-,279018,279018,279018,7,7,7,19102,19102,19102,8,9,8,188897,188897,188897,116170,116170,116170
\end{filecontents*}
\csvreader[
				late after line=\\,
				late after last line=\\,]
				{results/ablation-results/state_statistics.csv}{}
				{\compStateStatisticsPermissive}
		\end{tabular}}
	\end{adjustbox}
\end{table*}

\begin{table*}[t]
	\caption{
		Comparison of relevance heuristics for the \datasetCombined{} dataset.
	}\label{tab:agg-dataset-combined}%
	\centering%
	\setlength{\tabcolsep}{5pt}%
	\sisetup{table-format = 2.2\%, table-alignment-mode = format, table-auto-round}%
	\begin{adjustbox}{angle=270}%
		\resizebox*{.9\textheight}{!}{%
			\footnotesize
			\begin{tabular}{l|rrrrrrrrrr}
				Name & Model Size & Final $\varepsilon=0$ & Final $\varepsilon=10^{-6}$ & Final $\varepsilon=10^{-2}$ & \agency $\varepsilon=0$ & \agency $\varepsilon=10^{-6}$ & \agency$\varepsilon = 10^{-2}$ & \allSafe & \dominance & \unreach  \\
				\midrule
				\csvreader[
				late after line=\\,
				late after last line=\\,]
				{results/ablation-results/state_statistics.csv}{}
				{\compStateStatisticsCombined}
		\end{tabular}}
	\end{adjustbox}
\end{table*}

\begin{table*}[t]
	\caption{
		Comparison of relevance heuristics for the \datasetController{} dataset.
	}\label{tab:agg-dataset-controller}%
	\centering%
	\setlength{\tabcolsep}{5pt}%
	\sisetup{table-format = 2.2\%, table-alignment-mode = format, table-auto-round}%
	\begin{adjustbox}{angle=270}%
		\resizebox*{.9\textheight}{!}{%
			\footnotesize
			\begin{tabular}{l|rrrrrrrrrr}
				Name & Model Size & Final $\varepsilon=0$ & Final $\varepsilon=10^{-6}$ & Final $\varepsilon=10^{-2}$ & \agency $\varepsilon=0$ & \agency $\varepsilon=10^{-6}$ & \agency$\varepsilon = 10^{-2}$ & \allSafe & \dominance & \unreach  \\
				\midrule
				\csvreader[
				late after line=\\,
				late after last line=\\,]
				{results/ablation-results/state_statistics.csv}{}
				{\compStateStatisticsController}
		\end{tabular}}
	\end{adjustbox}
\end{table*}

\subsubsection{Results on DT Sizes}\label{app:subsec_tree_sizes}
\cref{tab:tree-size} shows the sizes of the computed DTs.
We highlight the columns for \dtpaynt.
While the tool aims to provide the smallest possible DTs, we found provably smaller ones, suggesting either an internal timeout during some steps of the computation or a bug in the code.

Note that for several entries with $\varepsilon=0$, \dte does not produce a result, indicated by a \enquote{dash}.
This is because for these, the controller we obtained through a model checking query to \storm was not optimal, as confirmed through a second model checking query; it also differed by more than the float precision of $10^{-14}$ from the claimed value of the first query.
Even though we set \storm to sound solving, we were not able to avoid such cases. 
In these cases, \dte aborts.
Even though we set \storm to sound solving, we were not able to avoid such cases. 

The comparison to CAV15 in \cref{fig:app_1e-6comp} shows that the heuristics of the portfolio are, at least for $\varepsilon=10^{-6}$, often incomparable to CAV15.
However, as has been shown before in \cref{fig:treesize}, their combination outperforms CAV15.
Furthermore, CAV15 is not guaranteed to provide results as it is aggressive.

We compare the size of DTs produced by our portfolio and \dtnest in \cref{fig:app_dtnest_size_comp}. 
For this plot, we ran the portfolio with an overall timeout of 10 minutes in addition to 5 minutes for model solving and 5 minutes for pruning.
We start with \datasetCombined+\wagency, continue with \datasetPermissive+\wmessup, and, if time permits, finish with \datasetController+\wsimulation.
\dtnest had a total time budget of 20 minutes.
For $\varepsilon=0$, we can see the computation errors for \dte mentioned above. 
Generally, especially for larger values of $\varepsilon$, \dte outperforms \dtnest.

\begin{table*}[tp]
	\caption{
		Concrete sizes of the DTs obtained through the portfolio methods for individual benchmarks (Step~3, \cref{sec:step3}), also considering different precision requirements. A \enquote{-} denotes a failure to finish.
	}\label{tab:tree-size}\centering%
	\begin{adjustbox}{angle=270}%
		\resizebox*{.9\textheight}{!}{%
			\setlength{\tabcolsep}{5pt}%
			\begin{tabular}{>{\scriptsize}lrrrrrrrrr||rrrrrrrrrrrrr}
				Name & \multicolumn{3}{c}{\wsimulation, \datasetController} & \multicolumn{3}{c}{\wagency, \datasetCombined} & \multicolumn{3}{c}{\wmessup, \datasetPermissive} & \multicolumn{3}{c}{Our Best}& \multicolumn{3}{c}{CAV15} & \dtcontrol & \multicolumn{3}{c}{\dtnest} & \multicolumn{3}{c}{\dtpaynt}\\
				\bfseries & {$0$} & {$10^{-6}$} & {$10^{-2}$} & {$0$} & {$10^{-6}$} & {$10^{-2}$} & {$0$} & {$10^{-6}$} & {$10^{-2}$}  & {$0$} & {$10^{-6}$} & {$10^{-2}$} & {$0$} & {$10^{-6}$} & {$10^{-2}$} & & {$0$} & {$10^{-6}$} & {$10^{-2}$} & {$0$} & {$10^{-6}$} & {$10^{-2}$} \\
				\midrule
				\begin{filecontents*}{results/ablation-results/final_tree_nodes_with_dtn_dtp.csv}
benchmark,cav15-0.000000,cav15-0.000001,cav15-0.010000,ours-combined-agency-simple-0.000000,ours-combined-agency-simple-0.000001,ours-combined-agency-simple-0.010000,ours-combined-cost-of-error-simple-0.000000,ours-combined-cost-of-error-simple-0.000001,ours-combined-cost-of-error-simple-0.010000,ours-combined-simulation-simple-0.000000,ours-combined-simulation-simple-0.000001,ours-combined-simulation-simple-0.010000,ours-combined-uniform-simple-0.000000,ours-combined-uniform-simple-0.000001,ours-combined-uniform-simple-0.010000,ours-controller-agency-simple-0.000000,ours-controller-agency-simple-0.000001,ours-controller-agency-simple-0.010000,ours-controller-cost-of-error-simple-0.000000,ours-controller-cost-of-error-simple-0.000001,ours-controller-cost-of-error-simple-0.010000,ours-controller-simulation-simple-0.000000,ours-controller-simulation-simple-0.000001,ours-controller-simulation-simple-0.010000,ours-controller-uniform-simple-0.000000,ours-controller-uniform-simple-0.000001,ours-controller-uniform-simple-0.010000,ours-permissive-agency-simple-0.000000,ours-permissive-agency-simple-0.000001,ours-permissive-agency-simple-0.010000,ours-permissive-cost-of-error-simple-0.000000,ours-permissive-cost-of-error-simple-0.000001,ours-permissive-cost-of-error-simple-0.010000,ours-permissive-simulation-simple-0.000000,ours-permissive-simulation-simple-0.000001,ours-permissive-simulation-simple-0.010000,ours-permissive-uniform-simple-0.000000,ours-permissive-uniform-simple-0.000001,ours-permissive-uniform-simple-0.010000,portfolio-0.000000,portfolio-0.000001,portfolio-0.010000,vanilla-0.000000,virtual-best-0.000000,virtual-best-0.000001,virtual-best-0.010000,dtnest-11-0,dtnest-7-0,dtnest-9-0,dtnest-11-1e-2,dtnest-7-1e-2,dtnest-9-1e-2,dtnest-11-1e-6,dtnest-7-1e-6,dtnest-9-1e-6,dtpaynt-128-0,dtpaynt-16-0,dtpaynt-128-1e-2,dtpaynt-16-1e-2,dtpaynt-128-1e-6,dtpaynt-16-1e-6
consensus4-4-c2,err,107,107,err,31,23,err,33,25,err,55,25,err,31,23,err,31,23,err,31,25,err,59,25,err,31,23,err,31,25,err,31,25,err,31,25,err,31,25,-,31,23,err,-,31,23,-,-,-,-,-,-,-,-,-,-,-,-,-,-,-
consensus4-4-disagree,err,err,1479,err,1561,895,err,1997,315,err,1769,267,err,1585,1163,err,1545,839,err,2059,2059,err,1883,197,err,1651,999,err,741,639,err,829,267,err,983,469,err,751,547,-,829,197,err,-,741,197,-,-,-,-,-,-,-,-,-,-,-,-,-,-,-
consensus4-4-steps-max,err,341,341,err,231,231,err,265,265,err,277,235,err,233,233,err,241,241,err,277,277,err,271,147,err,239,239,err,339,339,err,241,241,err,381,469,err,247,331,-,231,147,err,-,231,147,-,-,-,-,-,-,-,-,-,-,-,-,-,-,-
consensus4-4-steps-min,err,1021,1021,err,767,737,err,887,887,err,895,499,err,767,737,err,795,711,err,897,897,err,929,575,err,801,717,err,2067,2067,err,1785,435,err,2055,1931,err,2067,2067,-,767,435,err,-,767,435,-,-,-,-,-,-,-,-,-,-,-,-,-,-,-
csma2-6-all-before-max,105,105,105,1,1,1,1,1,1,1,1,1,1,1,1,1,1,1,1,1,1,1,1,1,1,1,1,1,1,1,1,1,1,1,1,1,1,1,1,1,1,1,133,1,1,1,29,29,29,29,29,29,29,29,29,-,-,1,1,1,1
csma2-6-time-max,107,107,107,5,5,3,5,5,3,11,11,9,5,5,3,5,5,3,5,5,3,11,11,9,5,5,3,5,5,3,5,5,3,5,5,3,5,5,3,5,5,3,135,5,5,3,25,25,25,5,5,5,5,5,5,-,-,5,5,5,5
csma2-6-time-min,83,83,83,5,5,3,5,5,3,11,11,9,5,5,3,5,5,3,5,5,3,11,11,9,5,5,3,5,5,3,5,5,3,5,5,3,5,5,3,5,5,3,125,5,5,3,25,25,25,7,7,7,7,7,7,-,-,5,5,11,11
eajs3-150-exputil,59,59,59,13,13,13,13,13,13,13,13,11,13,13,13,13,13,13,13,13,13,13,13,11,13,13,13,25,25,25,21,21,25,15,19,15,25,25,23,13,13,11,467,13,13,11,75,75,75,75,75,75,75,75,75,-,-,-,-,-,-
firewire-dl-3-400-deadline,15,15,15,9,9,9,7,7,7,9,9,9,11,11,11,9,9,9,7,7,7,9,9,9,11,11,11,9,9,9,9,9,9,13,13,11,9,9,9,9,9,9,57,7,7,7,37,37,37,9,9,9,9,9,9,-,-,41,41,41,41
firewire-dl-3-600-deadline,15,15,15,13,13,13,13,13,11,13,13,11,15,15,13,13,13,13,13,13,11,13,13,9,15,15,11,19,19,7,19,19,7,19,19,9,19,19,7,13,13,7,57,13,13,7,35,35,35,13,15,15,13,15,15,-,-,27,27,27,27
firewire-dl-36-200-deadline,13,13,13,13,13,13,11,11,11,13,13,11,13,13,13,13,13,13,11,11,11,13,11,13,13,13,13,17,17,17,17,17,17,17,17,17,17,17,17,13,11,13,57,11,11,11,23,23,23,13,13,13,13,13,13,-,-,-,-,-,-
firewire-dl-36-400-deadline,21,21,23,13,13,13,13,13,13,15,15,11,13,13,13,13,13,13,13,13,13,15,15,11,13,13,13,19,19,19,19,19,17,19,19,17,19,19,17,13,13,11,57,13,13,11,-,-,-,15,15,15,15,15,15,-,-,63,-,63,-
firewire-dl-36-600-deadline,13,13,13,13,13,13,11,11,11,11,13,13,13,13,13,13,13,13,11,11,11,13,13,13,13,13,13,17,17,17,17,17,17,17,17,17,17,17,17,13,13,13,57,11,11,11,23,23,23,13,13,13,13,13,13,-,-,-,-,-,-
firewire-false-36-400-timemax,err,51,51,err,19,19,err,9,9,err,21,21,err,19,19,err,19,19,err,9,9,err,21,21,err,19,19,err,15,13,err,9,7,err,19,19,err,17,17,-,9,7,err,-,9,7,-,17,17,17,17,17,21,21,21,-,-,25,-,25,-
firewire-false-36-400-timemin,31,31,25,21,21,9,17,17,13,21,21,19,27,27,23,21,21,9,17,17,11,21,21,19,27,27,23,13,13,9,13,13,7,11,11,9,13,13,5,13,13,7,12913,11,11,5,-,31,-,-,-,-,-,-,-,-,-,29,-,29,-
ij10,753,753,753,1,1,1,1,1,1,1,1,1,1,1,1,1,1,1,1,1,1,1,1,1,1,1,1,1,1,1,1,1,1,1,1,1,1,1,1,1,1,1,1317,1,1,1,719,719,719,719,719,719,719,719,719,-,-,1,1,1,1
pacman5,1,1,1,1,1,1,1,1,1,1,1,1,1,1,1,1,1,1,1,1,1,1,1,1,1,1,1,1,1,1,1,1,1,1,1,1,1,1,1,1,1,1,15,1,1,1,3,3,3,3,3,3,3,3,3,-,-,1,1,1,1
pnueli-zuck-5-live,2673,2673,2673,1,1,1,1,1,1,1,1,1,1,1,1,2033,2033,1865,2037,2037,1869,1319,1319,1149,2033,2033,1865,1,1,1,1,1,1,1,1,1,1,1,1,1,1,1,171373,1,1,1,-,-,-,-,-,-,-,-,-,-,-,1,1,1,1
pnueli3-live,99,99,99,1,1,1,1,1,1,1,1,1,1,1,1,91,91,91,93,93,93,89,89,89,91,91,91,1,1,1,1,1,1,1,1,1,1,1,1,1,1,1,1729,1,1,1,97,97,97,97,97,97,97,97,97,-,-,1,1,1,1
rabin3-live,119,119,119,1,1,1,1,1,1,1,1,1,1,1,1,1,1,1,1,1,1,1,1,1,1,1,1,1,1,1,1,1,1,1,1,1,1,1,1,1,1,1,113,1,1,1,33,33,33,33,33,33,33,33,33,-,-,1,1,1,1
wlan3-0-cost-max,err,221,221,279,279,227,275,275,275,231,231,231,211,211,211,281,281,233,277,277,277,233,233,233,213,213,213,119,119,119,157,157,157,129,143,129,135,135,135,157,157,157,775,119,119,119,325,325,325,-,-,-,-,101,111,-,-,-,-,-,-
wlan3-0-cost-min,41,41,35,13,13,9,15,15,11,13,13,11,13,13,11,17,17,9,17,17,13,15,15,13,15,15,13,17,17,11,19,19,11,13,13,11,13,13,9,13,13,9,633,13,13,9,13,13,13,9,9,9,13,13,13,-,-,13,13,13,13
wlan4-0-cost-max,err,149,149,397,397,381,399,399,399,611,611,611,253,253,253,399,399,383,401,401,401,613,613,613,255,255,255,109,103,103,245,245,245,107,103,121,103,103,103,245,245,245,1263,103,103,103,-,-,-,-,-,-,-,-,-,-,-,-,-,-,-
wlan4-0-cost-min,41,41,35,13,13,9,15,15,11,13,13,11,13,13,11,17,17,9,17,17,13,15,15,13,15,15,13,17,17,11,19,19,11,17,17,11,13,13,9,13,13,9,1053,13,13,9,13,13,13,9,9,9,13,13,13,-,-,15,15,15,15
wlandl0-80-deadline,TO,TO,TO,539,539,539,523,523,523,531,519,533,533,533,533,501,501,501,523,523,523,515,547,541,541,541,533,61,61,57,73,73,23,63,63,67,59,59,27,73,73,23,3233,59,59,23,617,617,617,-,-,-,-,139,-,-,-,-,-,-,-
wlandl1-80-deadline,TO,TO,TO,575,575,557,511,511,511,577,591,95,563,563,553,559,559,547,511,511,511,551,519,519,563,563,557,61,61,51,73,73,23,69,67,29,61,61,35,73,73,23,4035,61,61,23,-,659,659,-,-,-,-,-,229,-,-,-,-,-,-
zeroconf-1000-2-false-correctmax,err,err,77,err,77,1,err,73,1,err,81,1,err,17,1,err,71,1,err,71,1,err,67,1,err,15,1,err,19,1,err,17,1,err,17,1,err,15,1,-,17,1,err,-,15,1,111,111,111,-,-,-,-,-,-,-,-,-,-,-,-
zeroconf-1000-2-false-correctmin,err,23,1,err,3,1,err,1,1,err,5,1,err,5,1,err,3,1,err,1,1,err,5,1,err,5,1,err,9,1,err,1,1,err,7,1,err,7,1,-,1,1,err,-,1,1,9,9,9,-,-,-,-,-,-,-,-,-,-,-,-
zeroconf-1000-4-false-correctmax,err,65,65,err,139,117,err,67,53,err,161,1,err,161,135,err,175,145,err,69,57,err,379,1,err,169,169,err,9,1,err,9,1,err,87,1,err,59,55,-,9,1,err,-,9,1,-,509,509,-,-,-,-,-,-,-,-,-,-,-,-
zeroconf-1000-4-false-correctmin,err,err,1,err,3,1,err,1,1,err,5,1,err,5,1,err,3,1,err,1,1,err,5,1,err,5,1,err,25,25,err,61,61,err,61,61,err,61,61,-,3,1,err,-,1,1,9,9,9,-,-,-,-,-,-,-,-,-,-,-,-
zeroconf-20-2-false-correctmax,err,err,1,93,11,1,105,9,1,65,9,1,65,13,1,91,11,1,109,9,1,73,61,61,73,61,61,19,9,1,17,9,1,17,11,1,17,11,1,17,9,1,1373,17,9,1,-,127,127,-,-,-,-,-,55,-,-,-,-,-,-
zeroconf-20-2-false-correctmin,err,err,1,3,3,1,1,1,1,5,3,1,5,3,1,3,3,1,1,1,1,5,3,1,5,3,1,11,11,1,1,1,1,11,11,1,11,11,1,1,1,1,7737,1,1,1,9,9,9,-,-,-,-,-,-,-,-,-,-,-,-
zeroconf-20-4-false-correctmax,err,3,1,241,71,71,253,1,1,91,91,91,91,91,91,237,69,69,255,1,1,219,1,1,135,135,135,111,1,1,109,1,1,31,19,19,31,19,19,109,1,1,4553,31,1,1,303,303,303,-,-,-,-,-,-,-,-,-,-,-,-
zeroconf-20-4-false-correctmin,err,1,1,err,1,1,err,1,1,err,1,1,err,1,1,err,1,1,err,1,1,err,1,1,err,1,1,err,2143,2143,err,933,933,err,1961,1961,err,1913,1913,-,1,1,err,-,1,1,9,9,9,-,-,-,-,-,-,-,-,-,-,-,-
zeroconfdl-1000-1-false-20-deadlinemax,err,err,113,101,101,65,103,103,1,109,109,1,57,57,55,107,107,63,125,125,1,115,115,1,69,69,73,21,19,1,17,17,1,51,51,1,21,21,1,17,17,1,1553,17,17,1,-,303,303,-,-,-,-,-,-,-,-,-,-,-,-
zeroconfdl-1000-1-false-20-deadlinemin,19,9,1,13,3,1,17,5,1,15,5,1,13,3,1,15,3,1,17,5,1,15,7,1,13,5,1,13,9,1,15,15,1,19,5,1,13,5,1,13,3,1,823,13,3,1,35,35,35,-,-,-,-,-,-,-,-,-,-,-,-
zeroconfdl-1000-1-false-50-deadlinemax,err,err,107,171,171,105,187,187,1,175,175,1,115,115,109,205,205,117,205,205,1,195,195,1,127,127,117,51,51,1,37,37,1,57,57,1,23,23,1,37,37,1,5045,23,23,1,-,389,389,-,-,-,-,-,-,-,-,-,-,-,-
zeroconfdl-1000-1-false-50-deadlinemin,13,15,9,1,1,1,1,1,1,5,5,1,3,3,1,3,3,1,1,1,1,19,19,1,5,5,1,1,1,1,1,1,1,65,65,65,65,65,65,1,1,1,2215,1,1,1,35,35,35,-,-,-,-,-,-,-,-,-,-,-,-
\end{filecontents*}
\csvreader[
				late after line=\\,
				late after last line=\\,]
				{results/ablation-results/final_tree_nodes_with_dtn_dtp.csv}{}
				{\compTabletreesize}
		\end{tabular}}
	\end{adjustbox}
\end{table*}

\begin{figure}[tp]
	\centering
	\setlength{\scatterplotsize}{0.42\textwidth}
	\scatterplottrees{results/ablation-results/final_tree_nodes_scatter_with_dtnest.csv}{dtnest-7-0}{\dtnest $\varepsilon=0$ Size}{virtual-best-0.000000}{Virtual Best $\varepsilon=0$}{false}
	\scatterplottrees{results/ablation-results/final_tree_nodes_scatter_with_dtnest.csv}{dtnest-7-0}{\dtnest $\varepsilon=0$ Size}{virtual-best-0.000001}{Virtual Best $\varepsilon=10^{-6}$}{false}
	\scatterplottrees{results/ablation-results/final_tree_nodes_scatter_with_dtnest.csv}{dtnest-7-0}{\dtnest $\varepsilon=0$ Size}{virtual-best-0.010000}{Virtual Best $\varepsilon=10^{-2}$}{false}
	
	\caption{DT sizes comparing \dtnest with the best of \dte.}
	\label{fig:app_virtualbest-dtnest}
\end{figure}

\begin{figure}[tp]
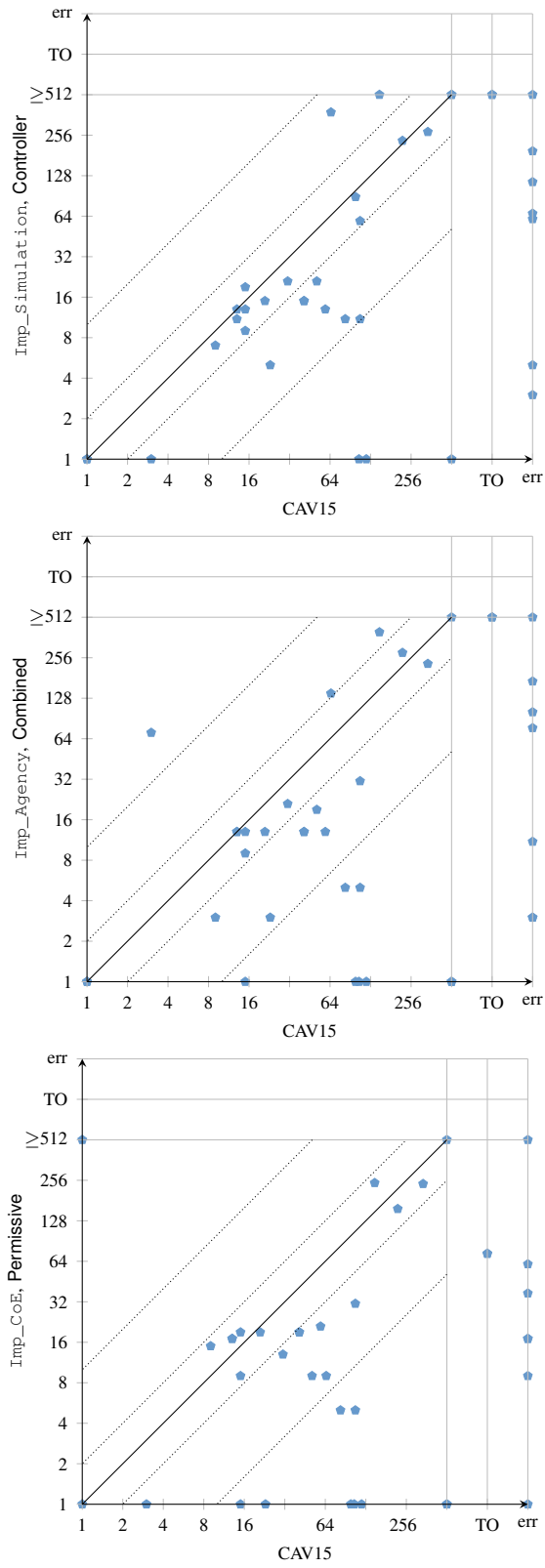

	\centering
	\setlength{\scatterplotsize}{0.42\textwidth}
	\scatterplottrees{results/ablation-results/final_tree_nodes_scatter.csv}{cav15-0.000001}{CAV15}{ours-controller-simulation-simple-0.000001}{\wsimulation, \datasetController}{false}
	\scatterplottrees{results/ablation-results/final_tree_nodes_scatter.csv}{cav15-0.000001}{CAV15}{ours-combined-agency-simple-0.000001}{\wagency, \datasetCombined}{false}
	\scatterplottrees{results/ablation-results/final_tree_nodes_scatter.csv}{cav15-0.000001}{CAV15}{ours-permissive-cost-of-error-simple-0.000001}{\wmessup, \datasetPermissive}{false}
	
	\caption{DT size after Step~3, comparing our portfolio heuristics to CAV15 for $\varepsilon=10^{-6}$.}
	\label{fig:app_1e-6comp}
\end{figure}

\begin{figure}[tp]
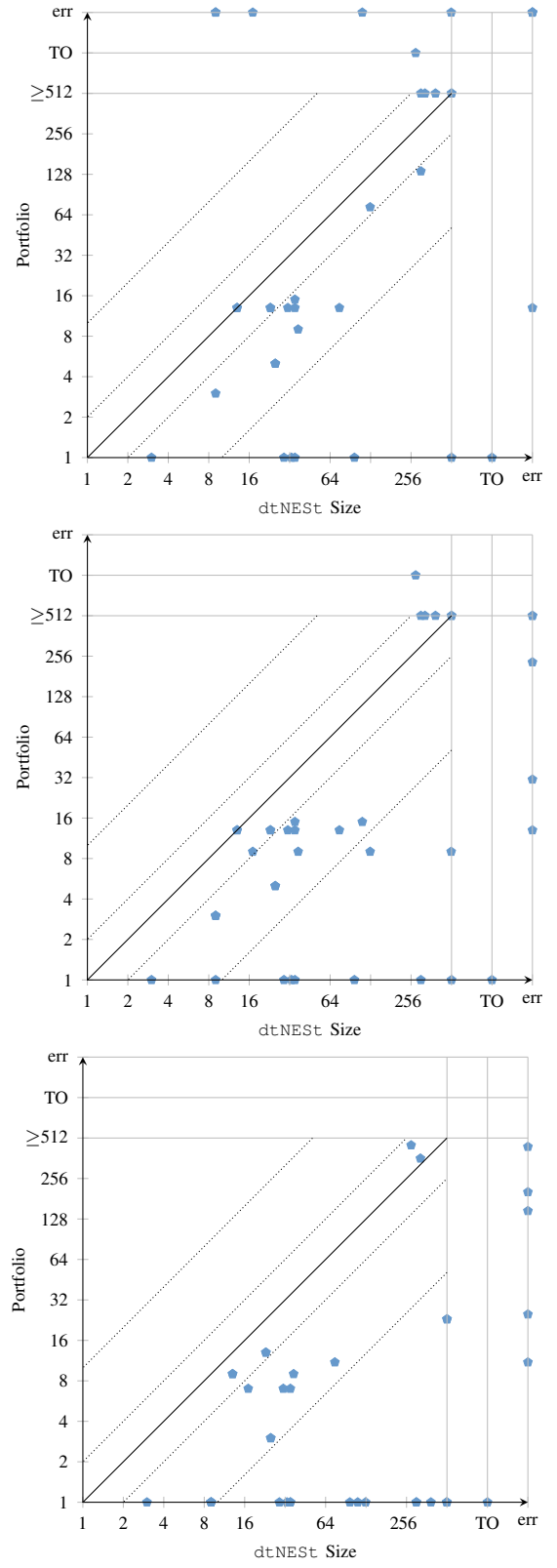

	\centering
	\setlength{\scatterplotsize}{0.42\textwidth}
	\scatterplottrees{results/logs-time-portfolio/time_portfolio_with_dtnest.csv}{dtnest-7-size-0}{\dtnest Size}{dteps-size-0}{Portfolio}{false}
	\scatterplottrees{results/logs-time-portfolio/time_portfolio_with_dtnest.csv}{dtnest-7-size-0}{\dtnest Size}{dteps-size-1e-6}{Portfolio}{false}
	\scatterplottrees{results/logs-time-portfolio/time_portfolio_with_dtnest.csv}{dtnest-7-size-0}{\dtnest Size}{dteps-size-1e-2}{Portfolio}{false}
	
	\caption{DT size comparison between \dtnest and the \dte portfolio. While \dtnest always uses $\varepsilon=0$, \dte is shown with $\varepsilon=0$ (top left), $\varepsilon=10^{-6}$ (top right) and $\varepsilon=10^{-2}$ (bottom).}
	\label{fig:app_dtnest_size_comp}
\end{figure}

\subsubsection{Results on Runtime} \label{app:exp-times} 
This section compares runtime across the different configurations implemented in \dte, as well as with \dtnest.

\subsubsection{Comparison of Different Configurations}
For the following experiments, we measured the time within \dtc and excluded the time required for model solving, as it is the same for all methods.
We also exclude the time for pruning, as pruning often tends to exhaust the given budget.
\cref{tab:runtime-full} provides an overview of the runtime observed.
A "-" either reports a timeout, precision error in model checking as described before, or cases where the size of the dataset is 0, meaning that no tree building is required.
We find that in many cases, the time required for tree building is relatively low. For example, with \datasetCombined+\agency{} and $\varepsilon=10^{-6}$, we have only one actual timeout for \texttt{wlan4-0-cost-max} and solve 28 benchmarks in less than one minute.

\cref{fig:model-time} compares the size of the model to the required runtime.
We find no correlation, showing that model structure has a bigger influence than model size.

\begin{figure}[tp]
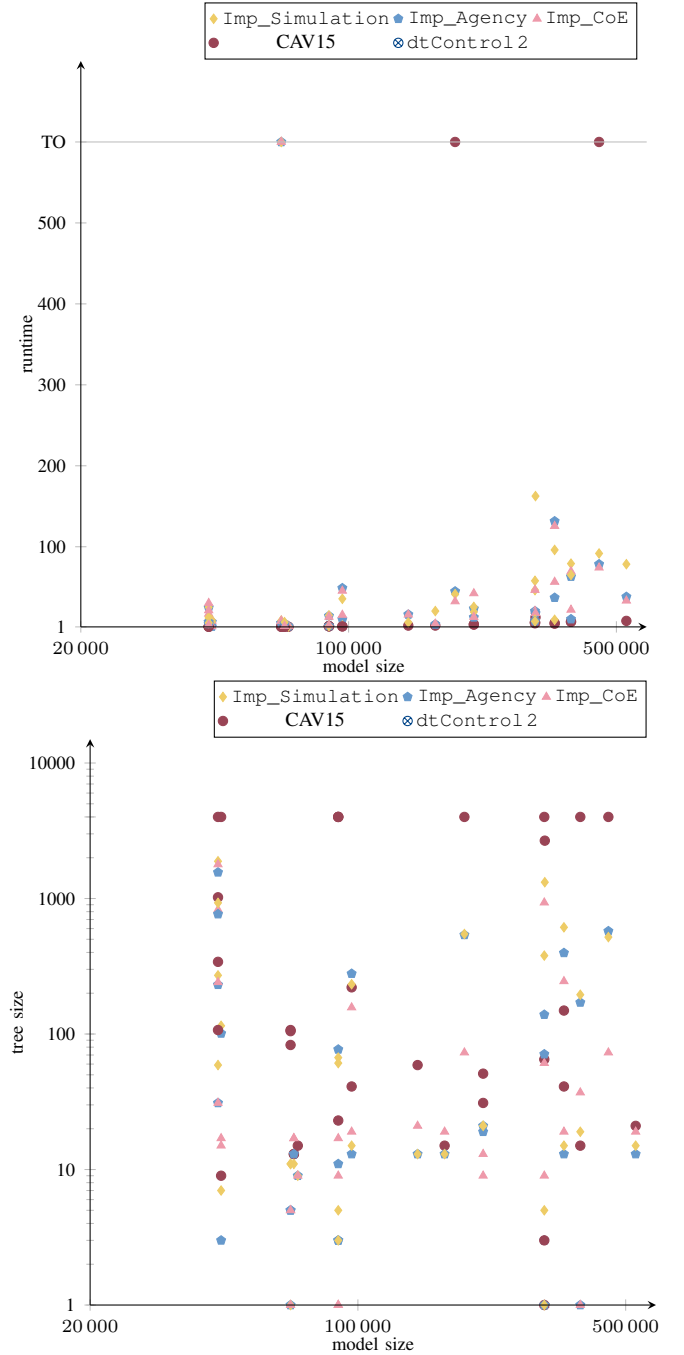

	\centering%
	\renewcommand{\scatterplotsize}{.5\textwidth}%
	\scatterplottimeToModel{results/ablation-results/tree_fitting_time_to_model_size_0.000001.csv}{modelSize}{model size}{tree-fitting-time}{runtime}{true}%
	\scatterplotSizeToModel{results/ablation-results/final_tree_nodes_to_model_size_0.000001.csv}{modelSize}{model size}{final-tree-nodes}{tree size}{true}%
	\caption{Comparison of runtime and DT size in relationship to the underlying model size for $\varepsilon=10^{-6}$.}
	\label{fig:model-time}
\end{figure}

\begin{table*}[tp]
	\caption{
		Detailed runtime of our considered methods (excluding pruning) per concrete instance (in seconds).
	}\label{tab:runtime-full}%
	\centering%
	\begin{adjustbox}{angle=270}%
		\resizebox*{.9\textheight}{!}{%
			\setlength{\tabcolsep}{5pt}
			\begin{tabular}{>{\scriptsize}lrrrrrrrrr||rrrrrrr}
				Name & \multicolumn{3}{c}{\wsimulation, \datasetController} & \multicolumn{3}{c}{\wagency, \datasetCombined} & \multicolumn{3}{c}{\wmessup, \datasetPermissive} & \multicolumn{3}{c}{Our Best}& \multicolumn{3}{c}{CAV15} & \dtcontrol \\
				\bfseries & {$0$} & {$10^{-6}$} & {$10^{-2}$} & {$0$} & {$10^{-6}$} & {$10^{-2}$} & {$0$} & {$10^{-6}$} & {$10^{-2}$}  & {$0$} & {$10^{-6}$} & {$10^{-2}$} & {$0$} & {$10^{-6}$} & {$10^{-2}$} &\\
				\midrule
				\begin{filecontents*}{results/ablation-results/tree_fitting_time.csv}
benchmark,cav15-0.000000,cav15-0.000001,cav15-0.010000,ours-combined-agency-simple-0.000000,ours-combined-agency-simple-0.000001,ours-combined-agency-simple-0.010000,ours-combined-cost-of-error-simple-0.000000,ours-combined-cost-of-error-simple-0.000001,ours-combined-cost-of-error-simple-0.010000,ours-combined-simulation-simple-0.000000,ours-combined-simulation-simple-0.000001,ours-combined-simulation-simple-0.010000,ours-combined-uniform-simple-0.000000,ours-combined-uniform-simple-0.000001,ours-combined-uniform-simple-0.010000,ours-controller-agency-simple-0.000000,ours-controller-agency-simple-0.000001,ours-controller-agency-simple-0.010000,ours-controller-cost-of-error-simple-0.000000,ours-controller-cost-of-error-simple-0.000001,ours-controller-cost-of-error-simple-0.010000,ours-controller-simulation-simple-0.000000,ours-controller-simulation-simple-0.000001,ours-controller-simulation-simple-0.010000,ours-controller-uniform-simple-0.000000,ours-controller-uniform-simple-0.000001,ours-controller-uniform-simple-0.010000,ours-permissive-agency-simple-0.000000,ours-permissive-agency-simple-0.000001,ours-permissive-agency-simple-0.010000,ours-permissive-cost-of-error-simple-0.000000,ours-permissive-cost-of-error-simple-0.000001,ours-permissive-cost-of-error-simple-0.010000,ours-permissive-simulation-simple-0.000000,ours-permissive-simulation-simple-0.000001,ours-permissive-simulation-simple-0.010000,ours-permissive-uniform-simple-0.000000,ours-permissive-uniform-simple-0.000001,ours-permissive-uniform-simple-0.010000,portfolio-0.000000,portfolio-0.000001,portfolio-0.010000,vanilla-0.000000,virtual-best-0.000000,virtual-best-0.000001,virtual-best-0.010000
consensus4-4-c2,err,0.8,0.8,err,4.7,4.7,err,1.2,1.2,err,6.2,5.3,err,3.5,3.6,err,4.5,4.6,err,1.2,1.2,err,5.9,5.0,err,3.4,3.5,err,3.3,3.3,err,3.3,3.3,err,8.2,2.8,err,3.2,3.2,-,3.3,3.3,err,-,1.2,1.2
consensus4-4-disagree,err,err,1.6,err,24.8,13.5,err,27.7,11.8,err,26.0,12.6,err,24.6,16.2,err,22.5,12.8,err,26.8,26.6,err,25.3,10.2,err,25.0,15.2,err,27.0,20.3,err,30.4,19.3,err,29.8,24.8,err,25.8,20.2,-,24.8,10.2,err,-,22.5,10.2
consensus4-4-steps-max,err,1.1,1.1,err,7.1,6.9,err,11.5,11.6,err,10.7,9.5,err,7.0,7.0,err,7.7,7.6,err,12.3,12.3,err,10.8,3.3,err,7.2,7.0,err,7.2,7.0,err,17.2,17.2,err,7.3,7.3,err,7.0,7.1,-,7.1,3.3,err,-,7.0,3.3
consensus4-4-steps-min,err,1.2,1.2,err,6.9,5.5,err,14.4,14.5,err,13.3,5.2,err,6.6,5.3,err,8.1,5.2,err,10.1,10.4,err,13.0,9.9,err,8.1,5.1,err,13.1,13.3,err,21.4,9.6,err,16.2,13.2,err,13.1,13.2,-,6.9,5.5,err,-,6.6,5.1
csma2-6-all-before-max,1.3,1.3,1.3,-,-,-,-,-,-,-,-,-,-,-,-,-,-,-,-,-,-,-,-,-,-,-,-,-,-,-,-,-,-,-,-,-,-,-,-,-,-,-,2.1,-,-,-
csma2-6-time-max,1.3,1.3,1.3,5.7,5.8,1.4,8.4,8.5,1.4,7.7,7.8,1.3,1.4,1.5,1.5,5.7,5.8,1.4,8.4,8.5,1.4,7.8,7.8,1.3,1.4,1.5,1.5,5.6,5.6,1.4,8.4,8.5,1.4,3.9,4.8,4.7,1.4,1.4,1.4,5.7,5.8,1.3,2.2,1.4,1.4,1.3
csma2-6-time-min,1.2,1.2,1.2,5.7,5.7,1.4,8.3,8.4,1.4,8.6,8.7,1.5,1.4,1.4,1.4,5.7,5.7,1.4,8.3,8.4,1.4,8.6,8.8,1.5,1.4,1.4,1.4,5.5,5.6,1.4,8.2,8.4,1.4,6.9,7.0,6.9,1.4,1.4,1.4,5.7,5.7,1.4,2.1,1.4,1.4,1.4
eajs3-150-exputil,2.5,2.6,2.5,15.9,16.3,15.9,6.5,6.6,6.4,6.5,9.9,3.4,19.1,19.5,19.1,15.9,16.3,15.9,6.5,6.6,6.4,6.5,6.6,3.4,19.1,19.5,19.0,6.6,6.7,6.6,15.9,16.1,15.8,3.7,3.8,3.7,15.8,16.1,16.0,6.5,6.6,3.4,4.9,3.7,3.8,3.4
firewire-dl-3-400-deadline,1.2,1.2,1.2,1.7,1.7,1.8,1.9,1.9,1.9,1.5,1.5,1.5,1.7,1.7,1.8,1.7,1.7,1.8,1.9,1.9,1.9,4.2,1.5,1.5,5.6,5.6,5.7,2.3,2.3,2.3,2.3,2.3,2.3,11.4,9.5,9.9,2.3,2.3,2.3,1.7,1.5,1.5,2.7,1.5,1.5,1.5
firewire-dl-3-600-deadline,2.6,2.7,2.6,4.0,4.1,4.0,21.6,21.9,3.8,20.0,12.2,13.0,4.0,4.1,4.1,12.4,12.6,12.5,19.0,19.1,3.3,20.5,20.6,3.7,12.3,12.3,12.4,4.9,5.0,4.9,4.8,4.9,4.8,4.9,23.9,5.0,4.9,4.9,4.9,4.0,4.1,3.7,7.7,4.0,4.1,3.3
firewire-dl-36-200-deadline,1.1,1.1,1.1,4.8,4.8,4.8,7.8,7.8,7.8,1.5,1.4,6.4,4.3,4.3,4.3,4.8,4.8,4.8,7.8,7.8,7.8,1.4,6.8,1.4,4.3,4.3,4.3,4.9,4.8,4.8,2.0,2.0,2.0,9.1,8.4,9.5,4.9,4.8,4.8,1.4,2.0,1.4,2.2,1.4,1.4,1.4
firewire-dl-36-400-deadline,8.4,8.4,8.5,37.1,38.2,37.5,68.7,71.1,69.6,73.3,78.4,67.0,11.5,11.6,11.6,21.9,22.2,22.0,68.8,71.1,69.7,77.0,78.5,13.3,11.4,11.6,11.4,32.7,33.3,33.0,32.7,33.4,32.9,33.5,89.3,33.8,32.2,32.9,32.5,32.7,33.4,13.3,29.6,11.4,11.6,11.4
firewire-dl-36-600-deadline,1.1,1.1,1.1,4.9,4.8,4.8,7.9,7.9,7.9,6.5,1.5,1.5,4.4,4.4,4.4,4.9,4.8,4.8,7.9,7.9,7.9,1.5,1.5,1.4,4.4,4.3,4.4,4.9,4.9,4.9,2.0,2.0,2.0,8.5,9.6,9.2,4.9,4.9,4.9,1.5,1.5,1.4,2.2,1.5,1.5,1.4
firewire-false-36-400-timemax,err,3.9,4.0,err,12.8,13.1,err,19.6,20.0,err,19.5,19.9,err,15.2,15.7,err,12.9,13.0,err,19.6,20.0,err,19.5,20.2,err,15.2,15.5,err,9.9,10.1,err,13.2,13.5,err,18.7,19.0,err,18.9,19.4,-,12.8,13.1,err,-,9.9,10.1
firewire-false-36-400-timemin,4.1,4.1,4.1,22.4,23.6,23.7,22.4,25.0,24.9,25.5,25.6,25.8,19.4,19.7,19.6,22.4,23.5,23.6,22.3,25.1,25.1,25.6,25.7,25.9,19.5,19.6,19.6,14.8,15.0,7.5,42.1,42.4,42.5,7.0,7.0,7.0,7.0,7.0,7.0,22.4,23.6,23.7,18.0,7.0,7.0,7.0
ij10,0.1,0.1,0.1,-,-,-,-,-,-,-,-,-,-,-,-,-,-,-,-,-,-,-,-,-,-,-,-,-,-,-,-,-,-,-,-,-,-,-,-,-,-,-,0.2,-,-,-
pacman5,0.0,0.0,0.0,-,-,-,-,-,-,-,-,-,-,-,-,-,-,-,-,-,-,-,-,-,-,-,-,-,-,-,-,-,-,-,-,-,-,-,-,-,-,-,0.0,-,-,-
pnueli-zuck-5-live,12.8,12.8,12.8,17.8,17.8,17.7,17.7,17.9,17.8,19.3,19.7,19.4,17.7,17.7,17.8,142.8,145.1,79.4,143.4,144.9,79.2,160.7,162.6,116.3,142.8,144.9,79.3,17.9,17.9,17.9,17.9,18.0,17.9,19.2,19.3,19.4,17.8,17.8,17.9,17.8,17.8,17.7,40.2,17.7,17.7,17.7
pnueli3-live,0.1,0.1,0.1,0.1,0.1,0.1,0.1,0.1,0.1,0.1,0.1,0.1,0.1,0.1,0.1,0.5,0.5,0.5,0.5,0.5,0.5,0.5,0.5,0.5,0.5,0.5,0.5,0.1,0.1,0.1,0.1,0.1,0.1,0.1,0.1,0.1,0.1,0.1,0.1,0.1,0.1,0.1,0.3,0.1,0.1,0.1
rabin3-live,0.0,0.0,0.0,-,-,-,-,-,-,-,-,-,-,-,-,-,-,-,-,-,-,-,-,-,-,-,-,-,-,-,-,-,-,-,-,-,-,-,-,-,-,-,0.1,-,-,-
wlan3-0-cost-max,err,1.6,1.6,49.2,48.9,42.2,66.8,66.8,66.2,34.9,35.3,35.2,50.9,50.9,50.7,49.8,49.5,43.7,67.1,67.1,66.5,35.9,35.7,35.6,50.8,51.0,50.5,50.5,50.1,50.3,45.4,45.6,45.8,45.1,43.7,42.7,49.1,48.9,48.8,35.9,35.7,35.6,2.7,34.9,35.3,35.2
wlan3-0-cost-min,1.7,1.7,1.7,10.3,10.4,5.2,13.9,14.1,13.9,2.2,2.2,2.2,2.2,2.2,2.2,5.2,5.3,5.2,17.1,17.0,17.0,2.7,2.7,2.7,2.7,2.7,2.7,9.4,9.4,9.3,15.7,15.8,15.7,2.6,2.7,2.6,2.7,2.7,2.6,2.7,2.7,2.7,2.5,2.2,2.2,2.2
wlan4-0-cost-max,err,5.0,5.0,130.5,131.5,119.4,198.1,200.3,208.9,94.1,95.0,95.7,133.7,134.7,139.9,128.8,130.3,118.1,198.4,200.9,208.7,94.5,96.1,98.6,132.6,133.9,139.0,130.0,119.0,123.1,124.4,125.5,130.4,118.5,115.5,116.9,124.9,125.7,130.3,94.5,96.1,98.6,9.7,94.1,95.0,95.7
wlan4-0-cost-min,5.8,5.9,5.8,36.9,37.1,18.4,49.7,50.2,49.6,7.9,8.0,7.9,7.9,7.9,7.9,18.6,18.8,18.5,61.3,61.5,60.8,9.7,9.8,9.6,9.6,9.7,9.5,33.2,33.4,32.9,55.8,56.5,55.6,9.4,9.5,9.3,9.4,9.5,9.3,9.7,9.8,9.6,9.1,7.9,7.9,7.9
wlandl0-80-deadline,TO,TO,TO,44.8,44.9,44.6,60.7,60.6,60.5,42.5,38.5,40.9,39.2,39.1,39.1,38.6,38.9,38.6,61.2,61.1,60.7,42.2,41.5,40.8,40.1,40.1,35.3,24.2,24.5,19.0,32.4,32.5,5.4,37.1,36.3,36.3,24.3,24.4,4.8,32.4,32.5,5.4,7.6,24.2,24.4,4.8
wlandl1-80-deadline,TO,TO,TO,78.7,78.4,67.0,119.7,118.7,118.6,118.4,92.5,54.0,75.9,75.6,65.5,81.7,81.4,70.0,120.6,119.7,119.3,92.3,91.7,96.7,79.1,78.6,67.8,55.3,55.3,30.7,74.4,74.2,11.0,99.4,139.8,66.5,56.1,56.1,10.0,74.4,74.2,11.0,15.9,55.3,55.3,10.0
zeroconf-1000-2-false-correctmax,err,err,1.8,err,14.6,2.5,err,13.9,2.4,err,15.2,2.7,err,2.5,2.6,err,13.4,2.4,err,13.8,2.4,err,15.2,2.7,err,2.7,2.7,err,13.5,2.5,err,13.4,2.5,err,12.3,2.9,err,4.7,2.5,-,13.4,2.5,err,-,2.5,2.4
zeroconf-1000-2-false-correctmin,err,1.9,1.9,err,2.1,2.1,err,2.4,2.4,err,2.1,2.2,err,2.1,2.1,err,2.1,2.1,err,2.4,2.5,err,2.2,2.2,err,2.1,2.1,err,3.9,2.1,err,2.4,2.5,err,7.4,2.5,err,7.4,2.0,-,2.1,2.1,err,-,2.1,2.0
zeroconf-1000-4-false-correctmax,err,5.4,5.4,err,20.5,8.3,err,24.4,7.5,err,14.6,9.6,err,14.7,8.5,err,20.7,8.7,err,24.6,7.6,err,57.9,10.2,err,9.8,9.8,err,14.4,8.2,err,20.6,8.1,err,45.6,8.1,err,21.0,8.8,-,20.5,8.1,err,-,9.8,7.5
zeroconf-1000-4-false-correctmin,err,err,6.5,err,7.1,7.1,err,8.5,8.5,err,46.4,7.8,err,14.9,7.8,err,7.1,7.1,err,8.5,8.5,err,46.4,7.7,err,14.9,7.8,err,8.0,8.1,err,10.0,10.0,err,8.0,7.9,err,7.9,7.9,-,7.1,7.1,err,-,7.1,7.1
zeroconf-20-2-false-correctmax,err,err,1.8,12.5,2.4,2.4,14.4,2.4,2.5,10.2,2.6,2.8,10.1,2.8,2.8,12.7,2.4,2.4,14.1,2.3,2.3,8.5,2.5,2.5,8.5,2.5,2.5,13.4,2.5,2.5,13.3,2.5,2.5,12.2,2.5,2.5,9.0,2.5,2.5,8.5,2.4,2.4,3.6,8.5,2.3,2.3
zeroconf-20-2-false-correctmin,err,err,1.9,2.1,2.1,2.1,2.4,2.5,2.4,2.1,2.1,2.1,2.0,2.1,2.1,2.1,2.1,2.1,2.4,2.5,2.5,2.1,2.1,2.1,2.1,2.1,2.1,2.6,2.6,2.6,2.9,3.0,3.0,2.6,2.6,2.7,2.6,2.6,2.6,2.1,2.1,2.1,4.9,2.0,2.1,2.1
zeroconf-20-4-false-correctmax,err,6.3,6.2,43.8,8.7,8.5,51.4,8.7,8.4,7.8,7.9,7.7,7.8,7.8,7.8,42.2,8.3,8.1,50.3,8.8,8.6,55.9,9.8,9.8,9.1,9.1,9.0,39.2,7.9,7.8,39.1,7.4,7.3,16.9,9.4,9.2,17.0,9.4,9.2,39.1,7.4,7.3,13.1,7.8,7.4,7.3
zeroconf-20-4-false-correctmin,err,6.8,6.9,err,8.0,8.1,err,8.5,8.7,err,7.8,7.9,err,7.8,7.9,err,8.0,8.0,err,8.5,8.7,err,7.9,8.0,err,7.9,8.0,err,73.2,73.9,err,46.6,47.2,err,73.5,73.8,err,73.3,73.6,-,7.9,8.0,err,-,7.8,7.9
zeroconfdl-1000-1-false-20-deadlinemax,err,err,0.8,7.2,7.1,1.5,8.1,8.1,1.4,8.1,8.2,1.5,1.5,1.5,1.5,7.7,7.8,1.8,8.3,8.2,1.5,9.5,9.5,1.7,2.4,2.4,1.8,7.4,7.5,1.7,8.3,8.3,1.6,7.8,7.8,1.5,5.3,5.3,1.8,7.2,7.1,1.5,2.1,1.5,1.5,1.4
zeroconfdl-1000-1-false-20-deadlinemin,0.8,0.8,0.8,1.4,1.4,1.5,1.4,1.4,1.4,7.4,7.5,1.3,2.3,1.3,1.3,1.6,1.6,1.7,1.5,1.5,1.5,7.4,7.5,1.3,2.9,1.6,1.6,1.7,1.7,1.7,1.7,1.7,1.7,7.6,7.7,1.3,1.5,1.5,1.5,1.4,1.4,1.3,1.7,1.4,1.3,1.3
zeroconfdl-1000-1-false-50-deadlinemax,err,err,7.3,63.3,63.1,11.7,68.5,68.0,11.6,77.7,77.2,13.5,21.1,21.1,11.3,68.3,67.6,13.1,72.4,71.6,12.4,79.4,79.1,13.7,11.7,11.5,13.0,66.1,66.0,13.0,70.1,69.5,12.4,66.9,67.1,12.0,13.8,13.7,13.7,63.3,63.1,11.7,18.8,11.7,11.5,11.3
zeroconfdl-1000-1-false-50-deadlinemin,7.3,7.4,7.4,10.3,10.6,10.5,9.4,9.6,9.6,57.9,59.2,9.6,10.0,10.2,10.3,11.1,11.3,11.3,10.1,10.2,10.3,63.5,65.2,10.7,11.8,12.2,12.0,19.3,19.7,19.4,21.6,22.1,21.7,21.3,21.7,21.6,21.3,21.7,21.5,10.3,10.6,10.5,15.2,9.4,9.6,9.6
\end{filecontents*}
\csvreader[
				late after line=\\,
				late after last line=\\,]
				{results/ablation-results/tree_fitting_time.csv}{}
				{\compTabletime}
		\end{tabular}}
	\end{adjustbox}
\end{table*}

\subsubsection{Runtime Comparison to \dtnest} \label{app:subsec_dtn_time}
In this section, we report on the runtime of the \dte portfolio and \dtnest. 
For both, we measured the time outside of the Docker container, meaning the time includes startup and shutdown.
Experiments were repeated 5 times.
\dte is again run as a portfolio, meaning there is a timeout of 5 minutes for model solving, 10 minutes for running all portfolio configurations (first \datasetCombined+\wagency, then \datasetPermissive+\wmessup, last \datasetController+\wsimulation) until the time budget is exhausted, and a 5-minute timeout for pruning.
\cref{fig:app_dtnest_time_comp} shows results.
We always choose $\varepsilon=0$ for \dtnest, as we have seen before that \dtnest struggles when this value is replaced and solves fewer benchmarks.
Surprisingly, even though \dte ran consecutively in three configurations, the runtime is often comparable, especially for larger values of $\varepsilon$.

For completeness, we include the runtime of our portfolio, \dtnest{} and \dtpaynt in \cref{tab:runtime-dtnest_full}.
We ran \dtpaynt with 16GB and 128GB of memory, as we encountered many cases in which the process was killed due to memory usage.
Unfortunately, these cases then mostly turned into timeouts.

\begin{figure}[tp]
	\centering
	\setlength{\scatterplotsize}{0.42\textwidth}
	\scatterplottime{results/logs-time-portfolio/time_portfolio_with_dtnest.csv}{dtnest-7-runtime-0}{\dtnest}{dteps-runtime-0}{Portfolio}{false}
	\scatterplottime{results/logs-time-portfolio/time_portfolio_with_dtnest.csv}{dtnest-7-runtime-0}{\dtnest}{dteps-runtime-1e-6}{Portfolio}{false}
	\scatterplottime{results/logs-time-portfolio/time_portfolio_with_dtnest.csv}{dtnest-7-runtime-0}{\dtnest}{dteps-runtime-1e-2}{Portfolio}{false}
	
	\caption{Runtime comparison between \dtnest and the \dte portfolio. While \dtnest always uses $\varepsilon=0$, \dte is shown with $\varepsilon=0$ (top left), $\varepsilon=10^{-6}$ (top right) and $\varepsilon=10^{-2}$ (bottom).}
	\label{fig:app_dtnest_time_comp}
\end{figure}

\begin{table*}[tp]
	\caption{
		Detailed runtime of our portfolio, \dtnest, and \dtpaynt with 128GB and 16GB of RAM (in seconds).
	}\label{tab:runtime-dtnest_full}%
	\centering%
	\begin{adjustbox}{angle=270}%
		\resizebox*{.62\textheight}{!}{%
			\setlength{\tabcolsep}{5pt}
			\begin{tabular}{>{\scriptsize}lrrrrrrrrrrrr}
				Name & \multicolumn{3}{c}{\dte} & \multicolumn{3}{c}{\dtnest} & \multicolumn{3}{c}{\dtpaynt, 16GB} & \multicolumn{3}{c}{\dtpaynt, 128GB} \\
				\bfseries & {$0$} & {$10^{-6}$} & {$10^{-2}$} & {$0$} & {$10^{-6}$} & {$10^{-2}$} & {$0$} & {$10^{-6}$} & {$10^{-2}$}  & {$0$} & {$10^{-6}$} & {$10^{-2}$} \\
				\midrule
				\begin{filecontents*}{results/logs-time-portfolio/time_portfolio_table_with_dtnest_with_dtpyant.csv}
benchmark,dteps-size-0,dteps-runtime-0,dteps-comment-0,dteps-size-1e-2,dteps-runtime-1e-2,dteps-comment-1e-2,dteps-size-1e-6,dteps-runtime-1e-6,dteps-comment-1e-6,dtnest-11-size-0,dtnest-11-runtime-0,dtnest-11-comment-0,dtnest-7-size-0,dtnest-7-runtime-0,dtnest-7-comment-0,dtnest-9-size-0,dtnest-9-runtime-0,dtnest-9-comment-0,dtnest-11-size-1e-2,dtnest-11-runtime-1e-2,dtnest-11-comment-1e-2,dtnest-7-size-1e-2,dtnest-7-runtime-1e-2,dtnest-7-comment-1e-2,dtnest-9-size-1e-2,dtnest-9-runtime-1e-2,dtnest-9-comment-1e-2,dtnest-11-size-1e-6,dtnest-11-runtime-1e-6,dtnest-11-comment-1e-6,dtnest-7-size-1e-6,dtnest-7-runtime-1e-6,dtnest-7-comment-1e-6,dtnest-9-size-1e-6,dtnest-9-runtime-1e-6,dtnest-9-comment-1e-6,dtpaynt-128-size-0,dtpaynt-128-runtime-0,dtpaynt-128-comment-0,dtpaynt-16-size-0,dtpaynt-16-runtime-0,dtpaynt-16-comment-0,dtpaynt-128-size-1e-2,dtpaynt-128-runtime-1e-2,dtpaynt-128-comment-1e-2,dtpaynt-16-size-1e-2,dtpaynt-16-runtime-1e-2,dtpaynt-16-comment-1e-2,dtpaynt-128-size-1e-6,dtpaynt-128-runtime-1e-6,dtpaynt-128-comment-1e-6,dtpaynt-16-size-1e-6,dtpaynt-16-runtime-1e-6,dtpaynt-16-comment-1e-6
consensus4-4-c2,err,err,Success,25,38,Success,31,43,Success,err,err,OptimisticValueIterationHelper.cpp:230,err,err,OptimisticValueIterationHelper.cpp:230,err,err,OptimisticValueIterationHelper.cpp:230,err,err,OptimisticValueIterationHelper.cpp:230,err,err,OptimisticValueIterationHelper.cpp:230,err,err,OptimisticValueIterationHelper.cpp:230,err,err,OptimisticValueIterationHelper.cpp:230,err,err,OptimisticValueIterationHelper.cpp:230,err,err,OptimisticValueIterationHelper.cpp:230,err,err,OptimisticValueIterationHelper.cpp:230,err,err,OptimisticValueIterationHelper.cpp:230,err,err,OptimisticValueIterationHelper.cpp:230,err,err,OptimisticValueIterationHelper.cpp:230,err,err,OptimisticValueIterationHelper.cpp:230,err,err,OptimisticValueIterationHelper.cpp:230
consensus4-4-disagree,err,err,Success,203,66,Success,512,512,Success,err,err,OptimisticValueIterationHelper.cpp:230,err,err,OptimisticValueIterationHelper.cpp:230,err,err,OptimisticValueIterationHelper.cpp:230,err,err,OptimisticValueIterationHelper.cpp:230,err,err,OptimisticValueIterationHelper.cpp:230,err,err,OptimisticValueIterationHelper.cpp:230,err,err,OptimisticValueIterationHelper.cpp:230,err,err,OptimisticValueIterationHelper.cpp:230,err,err,OptimisticValueIterationHelper.cpp:230,err,err,OptimisticValueIterationHelper.cpp:230,err,err,OptimisticValueIterationHelper.cpp:230,err,err,OptimisticValueIterationHelper.cpp:230,err,err,OptimisticValueIterationHelper.cpp:230,err,err,OptimisticValueIterationHelper.cpp:230,err,err,OptimisticValueIterationHelper.cpp:230
consensus4-4-steps-max,err,err,Success,147,94,Success,231,125,Success,err,err,OptimisticValueIterationHelper.cpp:230,err,err,OptimisticValueIterationHelper.cpp:230,err,err,OptimisticValueIterationHelper.cpp:230,err,err,OptimisticValueIterationHelper.cpp:230,err,err,OptimisticValueIterationHelper.cpp:230,err,err,OptimisticValueIterationHelper.cpp:230,err,err,OptimisticValueIterationHelper.cpp:230,err,err,OptimisticValueIterationHelper.cpp:230,err,err,OptimisticValueIterationHelper.cpp:230,err,err,OptimisticValueIterationHelper.cpp:230,err,err,OptimisticValueIterationHelper.cpp:230,err,err,OptimisticValueIterationHelper.cpp:230,err,err,OptimisticValueIterationHelper.cpp:230,err,err,OptimisticValueIterationHelper.cpp:230,err,err,OptimisticValueIterationHelper.cpp:230
consensus4-4-steps-min,err,err,Success,441,43,Success,512,109,Success,err,err,OptimisticValueIterationHelper.cpp:230,err,err,OptimisticValueIterationHelper.cpp:230,err,err,OptimisticValueIterationHelper.cpp:230,err,err,OptimisticValueIterationHelper.cpp:230,err,err,OptimisticValueIterationHelper.cpp:230,err,err,OptimisticValueIterationHelper.cpp:230,err,err,OptimisticValueIterationHelper.cpp:230,err,err,OptimisticValueIterationHelper.cpp:230,err,err,OptimisticValueIterationHelper.cpp:230,err,err,OptimisticValueIterationHelper.cpp:230,err,err,OptimisticValueIterationHelper.cpp:230,err,err,OptimisticValueIterationHelper.cpp:230,err,err,OptimisticValueIterationHelper.cpp:230,err,err,OptimisticValueIterationHelper.cpp:230,err,err,OptimisticValueIterationHelper.cpp:230
csma2-6-all-before-max,1,11,Success,1,10,Success,1,11,Success,29,10,Success,29,10,Success,29,10,Success,29,10,Success,29,10,Success,29,10,Success,29,10,Success,29,10,Success,29,10,Success,TO,TO,Timeout,err,err,Stop,1,1,Success,1,2,Success,1,2,Success,1,2,Success
csma2-6-time-max,5,58,Success,3,20,Success,5,57,Success,25,15,Success,25,15,Success,25,15,Success,5,24,Success,5,24,Success,5,23,Success,5,24,Success,5,24,Success,5,24,Success,TO,TO,Timeout,err,err,Stop,5,69,Success,5,69,Success,5,94,Success,5,94,Success
csma2-6-time-min,5,60,Success,3,21,Success,5,59,Success,25,18,Success,25,18,Success,25,18,Success,7,16,Success,7,16,Success,7,16,Success,7,16,Success,7,16,Success,7,16,Success,TO,TO,Timeout,err,err,Stop,5,49,Success,5,49,Success,11,95,Success,11,95,Success
eajs3-150-exputil,13,69,Success,11,63,Success,13,72,Success,75,283,Success,75,50,Success,75,81,Success,75,280,Success,75,50,Success,75,80,Success,75,285,Success,75,51,Success,75,81,Success,TO,TO,Timeout,err,err,Stop,err,err,AssertionError,err,err,AssertionError,err,err,AssertionError,err,err,AssertionError
firewire-dl-3-400-deadline,9,25,Success,9,21,Success,9,25,Success,37,47,Success,37,6,Success,37,48,Success,9,93,Success,9,78,Success,9,93,Success,9,93,Success,9,78,Success,9,93,Success,err,err,Stop,err,err,Stop,41,131,Success,41,131,Success,41,131,Success,41,132,Success
firewire-dl-3-600-deadline,13,214,Success,7,65,Success,13,215,Success,35,132,Success,35,29,Success,35,46,Success,13,57,Success,15,104,Success,15,96,Success,13,58,Success,15,104,Success,15,96,Success,err,err,Stop,err,err,Stop,27,134,Success,27,135,Success,27,133,Success,27,134,Success
firewire-dl-36-200-deadline,13,29,Success,13,29,Success,13,29,Success,23,7,Success,23,7,Success,23,7,Success,13,35,Success,13,35,Success,13,35,Success,13,35,Success,13,36,Success,13,35,Success,err,err,Stop,err,err,Stop,err,err,Stop,err,err,Stop,err,err,Stop,err,err,Stop
firewire-dl-36-400-deadline,13,241,Success,11,240,Success,13,172,Success,err,err,Stop,err,err,ValueError,err,err,ValueError,15,111,Success,15,106,Success,15,101,Success,15,110,Success,15,107,Success,15,101,Success,err,err,Stop,err,err,Stop,63,1033,Success,err,err,Stop,63,1035,Success,err,err,Stop
firewire-dl-36-600-deadline,13,30,Success,13,29,Success,13,29,Success,23,6,Success,23,7,Success,23,6,Success,13,36,Success,13,36,Success,13,35,Success,13,35,Success,13,36,Success,13,36,Success,err,err,Stop,err,err,Stop,err,err,Stop,err,err,Stop,err,err,Stop,err,err,Stop
firewire-false-36-400-timemax,err,err,Success,7,148,Success,9,142,Success,err,err,Stop,17,249,Success,17,246,Success,17,370,Success,17,256,Success,17,255,Success,21,285,Success,21,174,Success,21,172,Success,err,err,Stop,err,err,Stop,25,262,Success,err,err,Stop,25,260,Success,err,err,Stop
firewire-false-36-400-timemin,13,340,Success,7,409,Success,13,337,Success,err,err,Stop,31,86,Success,err,err,Stop,err,err,Stop,err,err,Stop,err,err,Stop,err,err,Stop,err,err,Stop,err,err,Stop,err,err,Stop,err,err,Stop,29,261,Success,err,err,Stop,29,261,Success,err,err,Stop
ij10,1,6,Success,1,6,Success,1,7,Success,719,42,Success,719,17,Success,719,42,Success,719,41,Success,719,18,Success,719,41,Success,719,43,Success,719,17,Success,719,41,Success,TO,TO,Timeout,err,err,Stop,1,0,Success,1,1,Success,1,1,Success,1,0,Success
pacman5,1,7,Success,1,6,Success,1,7,Success,3,3,Success,3,3,Success,3,3,Success,3,3,Success,3,3,Success,3,3,Success,3,3,Success,3,3,Success,3,3,Success,TO,TO,Timeout,err,err,Stop,1,1,Success,1,0,Success,1,0,Success,1,1,Success
pnueli-zuck-5-live,1,408,Success,1,361,Success,1,410,Success,err,err,Stop,TO,TO,Timeout,err,err,Stop,err,err,Stop,TO,TO,Timeout,err,err,Stop,err,err,Stop,TO,TO,Timeout,err,err,Stop,err,err,Stop,err,err,Stop,1,68,Success,1,69,Success,1,67,Success,1,67,Success
pnueli3-live,1,14,Success,1,14,Success,1,14,Success,97,61,Success,97,11,Success,97,18,Success,97,63,Success,97,11,Success,97,18,Success,97,62,Success,97,11,Success,97,18,Success,TO,TO,Timeout,err,err,Stop,1,1,Success,1,1,Success,1,0,Success,1,1,Success
rabin3-live,1,7,Success,1,7,Success,1,7,Success,33,8,Success,33,8,Success,33,8,Success,33,8,Success,33,8,Success,33,8,Success,33,8,Success,33,7,Success,33,8,Success,TO,TO,Timeout,err,err,Stop,1,0,Success,1,1,Success,1,1,Success,1,0,Success
wlan3-0-cost-max,512,512,Success,361,512,Success,512,512,Success,325,944,Success,325,310,Success,325,503,Success,95,1221,Success,TO,TO,Timeout,95,1214,Success,111,1235,Success,99,1220,Success,TO,TO,Timeout,TO,TO,Timeout,err,err,Stop,TO,TO,Timeout,TO,TO,Timeout,TO,TO,Timeout,TO,TO,Timeout
wlan3-0-cost-min,13,54,Success,9,48,Success,13,54,Success,13,14,Success,13,14,Success,13,15,Success,9,55,Success,9,54,Success,9,54,Success,13,48,Success,13,47,Success,13,47,Success,TO,TO,Timeout,err,err,Stop,13,81,Success,13,81,Success,13,72,Success,13,72,Success
wlan4-0-cost-max,TO,TO,Success,453,512,Success,TO,TO,Success,TO,TO,Timeout,277,1233,Success,TO,TO,Timeout,TO,TO,Timeout,TO,TO,Timeout,TO,TO,Timeout,TO,TO,Timeout,TO,TO,Timeout,TO,TO,Timeout,TO,TO,Timeout,err,err,Stop,TO,TO,Timeout,err,err,Stop,TO,TO,Timeout,err,err,Stop
wlan4-0-cost-min,13,176,Success,9,157,Success,13,175,Success,13,39,Success,13,38,Success,13,38,Success,9,106,Success,9,105,Success,9,105,Success,13,83,Success,13,83,Success,13,84,Success,TO,TO,Timeout,err,err,Stop,15,146,Success,15,146,Success,15,133,Success,15,131,Success
wlandl0-80-deadline,512,507,Success,23,487,Success,512,497,Success,err,err,Stop,617,239,Success,617,368,Success,TO,TO,Timeout,TO,TO,Timeout,err,err,ValueError,TO,TO,Timeout,141,1225,Success,err,err,ValueError,err,err,Stop,err,err,Stop,TO,TO,Timeout,err,err,Stop,TO,TO,Timeout,err,err,Stop
wlandl1-80-deadline,512,512,Success,23,512,Success,512,512,Success,err,err,Stop,659,747,Success,659,1022,Success,TO,TO,Timeout,TO,TO,Timeout,TO,TO,Timeout,TO,TO,Timeout,205,1229,Success,229,1222,Success,TO,TO,Timeout,err,err,Stop,TO,TO,Timeout,err,err,Stop,TO,TO,Timeout,err,err,Stop
zeroconf-1000-2-false-correctmax,err,err,Success,1,43,Success,15,239,Success,111,548,Success,111,138,Success,111,410,Success,err,err,IllegalArgumentException,err,err,IllegalArgumentException,err,err,IllegalArgumentException,err,err,IllegalArgumentException,err,err,IllegalArgumentException,err,err,IllegalArgumentException,TO,TO,Timeout,err,err,Stop,err,err,AssertionError,err,err,AssertionError,err,err,AssertionError,err,err,AssertionError
zeroconf-1000-2-false-correctmin,err,err,Success,1,33,Success,3,44,Success,9,19,Success,9,19,Success,9,19,Success,err,err,IllegalArgumentException,err,err,IllegalArgumentException,err,err,IllegalArgumentException,err,err,IllegalArgumentException,err,err,IllegalArgumentException,err,err,IllegalArgumentException,TO,TO,Timeout,err,err,Stop,err,err,AssertionError,err,err,AssertionError,err,err,AssertionError,err,err,AssertionError
zeroconf-1000-4-false-correctmax,err,err,Success,1,112,Success,9,290,Success,err,err,Stop,509,337,Success,509,624,Success,err,err,IllegalArgumentException,err,err,IllegalArgumentException,err,err,IllegalArgumentException,err,err,Stop,err,err,IllegalArgumentException,err,err,Stop,TO,TO,Timeout,err,err,Stop,err,err,AssertionError,err,err,Stop,err,err,AssertionError,err,err,Stop
zeroconf-1000-4-false-correctmin,err,err,Success,1,87,Success,3,127,Success,9,58,Success,9,59,Success,9,59,Success,err,err,IllegalArgumentException,err,err,IllegalArgumentException,err,err,IllegalArgumentException,err,err,IllegalArgumentException,err,err,IllegalArgumentException,err,err,IllegalArgumentException,TO,TO,Timeout,err,err,Stop,err,err,AssertionError,err,err,AssertionError,err,err,AssertionError,err,err,AssertionError
zeroconf-20-2-false-correctmax,73,91,Success,1,43,Success,9,37,Success,err,err,Stop,127,72,Success,127,200,Success,err,err,IllegalArgumentException,err,err,IllegalArgumentException,err,err,IllegalArgumentException,err,err,IllegalArgumentException,err,err,IllegalArgumentException,55,231,Success,TO,TO,Timeout,err,err,Stop,err,err,AssertionError,err,err,AssertionError,err,err,AssertionError,err,err,AssertionError
zeroconf-20-2-false-correctmin,3,28,Success,1,29,Success,3,29,Success,9,17,Success,9,18,Success,9,18,Success,err,err,IllegalArgumentException,err,err,IllegalArgumentException,err,err,IllegalArgumentException,err,err,IllegalArgumentException,err,err,IllegalArgumentException,err,err,IllegalArgumentException,TO,TO,Timeout,err,err,Stop,err,err,AssertionError,err,err,AssertionError,err,err,AssertionError,err,err,AssertionError
zeroconf-20-4-false-correctmax,135,244,Success,1,134,Success,1,133,Success,303,534,Success,303,133,Success,303,255,Success,err,err,IllegalArgumentException,err,err,IllegalArgumentException,err,err,IllegalArgumentException,err,err,IllegalArgumentException,err,err,IllegalArgumentException,err,err,IllegalArgumentException,TO,TO,Timeout,err,err,Stop,err,err,AssertionError,err,err,Stop,err,err,AssertionError,err,err,Stop
zeroconf-20-4-false-correctmin,err,err,Success,1,118,Success,1,117,Success,9,55,Success,9,55,Success,9,56,Success,err,err,IllegalArgumentException,err,err,IllegalArgumentException,err,err,IllegalArgumentException,err,err,IllegalArgumentException,err,err,IllegalArgumentException,err,err,IllegalArgumentException,TO,TO,Timeout,err,err,Stop,err,err,AssertionError,err,err,AssertionError,err,err,AssertionError,err,err,AssertionError
zeroconfdl-1000-1-false-20-deadlinemax,512,60,Success,1,27,Success,512,60,Success,err,err,ValueError,303,37,Success,303,82,Success,err,err,ValueError,err,err,IllegalArgumentException,err,err,IllegalArgumentException,err,err,ValueError,err,err,IllegalArgumentException,err,err,IllegalArgumentException,TO,TO,Timeout,err,err,Stop,err,err,AssertionError,err,err,AssertionError,err,err,AssertionError,err,err,AssertionError
zeroconfdl-1000-1-false-20-deadlinemin,15,39,Success,1,18,Success,15,38,Success,35,62,Success,35,27,Success,35,32,Success,err,err,IllegalArgumentException,err,err,IllegalArgumentException,err,err,IllegalArgumentException,err,err,IllegalArgumentException,err,err,IllegalArgumentException,err,err,IllegalArgumentException,TO,TO,Timeout,err,err,Stop,TO,TO,Timeout,err,err,Stop,TO,TO,Timeout,err,err,Stop
zeroconfdl-1000-1-false-50-deadlinemax,512,512,Success,1,177,Success,512,512,Success,err,err,Stop,389,201,Success,389,280,Success,err,err,IllegalArgumentException,err,err,IllegalArgumentException,err,err,IllegalArgumentException,err,err,IllegalArgumentException,err,err,IllegalArgumentException,err,err,IllegalArgumentException,TO,TO,Timeout,err,err,Stop,TO,TO,Timeout,err,err,Stop,TO,TO,Timeout,err,err,Stop
zeroconfdl-1000-1-false-50-deadlinemin,1,417,Success,1,120,Success,1,417,Success,35,285,Success,35,192,Success,35,205,Success,err,err,IllegalArgumentException,err,err,IllegalArgumentException,err,err,IllegalArgumentException,err,err,IllegalArgumentException,err,err,IllegalArgumentException,err,err,IllegalArgumentException,TO,TO,Timeout,err,err,Stop,TO,TO,Timeout,err,err,Stop,TO,TO,Timeout,err,err,Stop
\end{filecontents*}
\csvreader[
				late after line=\\,
				late after last line=\\,]
				{results/logs-time-portfolio/time_portfolio_table_with_dtnest_with_dtpyant.csv}{}
				{\compTablePayntNestTime}
		\end{tabular}}
	\end{adjustbox}
\end{table*}

\end{document}